\documentclass{article}
\usepackage{arxiv}
\usepackage{graphicx} 
\usepackage{amsfonts}
\usepackage{amssymb}
\usepackage{amsmath}
\usepackage{xcolor}
\usepackage{amsmath,amssymb,amsthm}
\usepackage{float}
\usepackage{multirow}
\usepackage{subcaption}
\usepackage{mathrsfs} 
\newtheorem{theorem}{Theorem}
\newtheorem{lemma}{Lemma}

\newtheorem{assumption}{Assumption}
\newtheorem{proposition}{Proposition}
\usepackage{booktabs}
\usepackage{algorithm}
\usepackage{algpseudocode}
\usepackage{hyperref}

\newcommand{\bx}{\mathbf{x}}
\newcommand{\by}{\mathbf{y}}
\newcommand{\bz}{\mathbf{z}}

\newcommand{\bs}{\mathbf{s}}
\newcommand{\bk}{\mathbf{k}}
\newcommand{\bj}{\mathbf{j}}
\newcommand{\bu}{\mathbf{u}}

\newcommand{\mM}{\mathcal{M}}
\newcommand{\mN}{\mathcal{N}}

\newcommand{\mL}{\mathcal{L}}

\newcommand{\E}{\mathbb{E}}
\newcommand{\R}{\mathbb{R}}

\newcommand{\eps}{\varepsilon}
\newcommand{\rd}{\mathrm d}

\title{Transformer-based In-Context Score Operator Learning for Diffusion Models}

\author{
Wuzhe Xu\\
South Dakota State University\\
\texttt{wuzhe.xu@sdstate.edu}
\And
Zhaiming Shen\\
Georgia Institute of Technology\\
\texttt{zshen49@gatech.edu}
\And
Yulong Lu\\
University of Minnesota\\
\texttt{yulonglu@umn.edu}
\And
\And
Wenjing Liao\\
Georgia Institute of Technology\\
\texttt{wliao@gatech.edu}
\And
Rongjie Lai\\
Purdue University\\
\texttt{lairj@purdue.edu}
}

\begin{document}

\maketitle

\begin{abstract}
Diffusion models typically learn the generative dynamics of a single target distribution or a family of distributions indexed by conditioning variables fixed during training. This limits their ability to adapt to a previously unseen distribution from only a small set of examples. We introduce an in-context solution operator learning framework for diffusion models, in which each target distribution is treated as a task drawn from a meta-distribution and is specified at inference time by an empirical context set. Our method learns an in-context score operator that maps the context samples, a noisy query point, and the diffusion time to the task-dependent score, enabling reverse-diffusion sampling on unseen tasks without retraining or parameter updates. A key observation is that the task-dependent score admits a Gaussian-kernel conditional-mean representation. Its empirical Nadaraya--Watson approximation can be represented exactly by a model in our constructed transformer class, providing a principled foundation for the proposed query-dependent attention architecture. The operator is trained through amortized denoising score matching across task distributions. We establish generalization-error bounds that quantify the effects of the context size, the numbers of training tasks and within-task samples, model complexity, and intrinsic and ambient dimensions. Numerically, we evaluate the method on a range of scientific and visual datasets, including Navier-Stokes, ERA5 climate data \cite{hersbach2020era5}, QuickDraw sketches \cite{quickdraw}, and Food-101 \cite{bossard14}. The results show that the proposed method performs well on relatively simple task families, while its performance degrades as task complexity increases, thereby illustrating both the promise and the limitations of in-context learning for diffusion models.

\end{abstract}

\section{Introduction}

Diffusion models have emerged as a powerful class of generative models, achieving strong empirical performance across a broad range of domains. They have been widely adopted in computer vision \cite{song2019scoregen,ho2020ddpm,song2020score,kim2022diffusionclip,bansal2023universalguidance,po2023visualcomputing}, audio generation \cite{kong2020diffwave,mittal2021symbolicmusic,huang2022prodiff,jeong2021difftts}, and reinforcement learning and control \cite{pearce2023imitatinghuman,reuss2023goalconditioned,zhu2023rlsurvey}. More recently, their applications have expanded to scientific and engineering problems, including data purification \cite{nie2022diffpure,wu2022guidedpurification,xiao2023densepure,xu2025improving}, climate-data downscaling \cite{srivastava2023precipitation_video_diffusion,shidqi2023regional_precipitation_diffusion}, and partial differential equation (PDE) problems \cite{huang2024diffusionpde,hong2024pdeinverse,xu2025diffusion,lu2024generative}. Their success has been further enhanced by conditional generation techniques such as classifier guidance \cite{dhariwal2021guided}, classifier-free guidance \cite{ho2022cfg}, and conditioned latent diffusion \cite{rombach2022ldm}, which provide substantial control over the generated samples.

Despite this progress, standard diffusion models typically learn the reverse generative dynamics associated with the distribution represented by a particular training dataset \cite{song2019scoregen,ho2020ddpm,song2020score}. Conditional diffusion models extend this formulation to families of distributions indexed by predefined conditioning variables, such as class labels, classifier outputs, text embeddings, or other task-specific signals \cite{dhariwal2021guided,ho2022cfg,rombach2022ldm}. Although these models can generate from multiple conditional distributions, the form and semantics of the conditioning mechanism are prescribed during training. Consequently, they are not generally designed to infer a previously unseen target distribution solely from a small, unordered set of contextual samples. Adapting to such a distribution may therefore require retraining, fine-tuning, or the construction of an additional conditioning representation. Recent few-shot and score-operator approaches have begun to relax this restriction \cite{giannone2022few,liao2024sno}, but learning a general in-context score operator directly from empirical context sets remains underexplored.

Many applications, however, require generative models to operate over an entire family of target distributions. In scientific computing, for example, the relevant solution or data distribution may vary with the governing equation, physical parameters, forcing terms, initial or boundary conditions, or available observations. This family-level structure is central to operator-learning methods such as DeepONet and the Fourier neural operator, which learn mappings between problem specifications and their corresponding solutions \cite{lu2021deeponet,li2021fno, kovachki2023neural}. Training an independent diffusion model for every possible configuration is computationally prohibitive and fails to exploit the structure shared across related problems. This motivates a solution-operator perspective for diffusion models: rather than learning the score function associated with a single prescribed target distribution, the objective is to learn an operator that maps information specifying a target distribution to its corresponding time-dependent score function \cite{liao2024sno}. Such an operator should generalize across a family of distributions and, ideally, adapt to previously unseen members of that family without retraining or parameter updates.

In this work, we introduce an in-context solution operator learning framework for diffusion models. Each target distribution, referred to as a task, is drawn from a meta-distribution over probability measures. More precisely, we learn an in-context score operator that maps the empirical context distribution, together with a noisy query point and the diffusion time, to the task-dependent score. This formulation casts task adaptation as an operator-evaluation problem: the context set identifies the target distribution, while the shared operator captures common structure across the task family. It thus extends diffusion modeling from learning a score function for a fixed target distribution to learning a reusable score operator over a family of distributions. The solution operator is trained through amortized denoising score matching across task distributions, allowing knowledge to be shared across the entire task family. Once trained, it can perform reverse diffusion for a previously unseen task directly from its prompt samples, without retraining, fine-tuning, or task-specific network parameter updates. This perspective is also closely related to recent work on in-context operator learning over probability measures \cite{cole2026context}.

Our method operates directly on the empirical context set and constructs score estimates through query-dependent interactions among the context samples, the noisy query point, and the diffusion time. A key observation is that the task-dependent score can be expressed through a Gaussian-kernel conditional-mean representation. Replacing the population conditional mean with its empirical counterpart yields a Nadaraya--Watson estimator based on the context samples. We show that this empirical estimator can be represented exactly by a model in our constructed transformer class. This connection provides a principled motivation for the proposed transformer-based in-context score operator, rather than treating the transformer architecture as an ad hoc design choice.

Theoretically, We further establish generalization-error bounds for the learned score operator in Section~\ref{sec:thm}. This bounds quantify the dependence on the context size $n$, the number of training tasks $K$, the number of episodes per task $J$, the number of diffusion samples per episode $N_T$, and the intrinsic and ambient dimensions $d$ and $D$. The finite-context error decays as $n^{-2/(d+6)}$, with the intrinsic dimension governing the exponent and the ambient dimension entering only through polynomial and logarithmic factors. This distinction indicates that in-context score estimation is governed primarily by the intrinsic geometry of the task distributions, exhibiting a curse of intrinsic rather than ambient dimensionality.  The statistical error contains terms of order $K^{-1/2}$, $(KJ)^{-1/2}$, and $(KJN_T)^{-1/2}$, separating the effects of the numbers of tasks, episodes, and diffusion samples. {In particular, the task-level persistent $K^{-1/2}$ term is unaffected by $J$ and $N_T$. Thus, within the present upper-bound analysis, increasing within-task sampling alone cannot reduce this contribution, highlighting the number and diversity of training tasks as fundamental to generalization.}.

We summarize our main contributions as follows:
\begin{itemize}
    \item We formulate in-context learning for diffusion models at the level of task distributions. Each task is modeled as a probability distribution drawn from a meta-distribution, and a previously unseen task is specified at inference time by a small context set represented as an empirical measure.

    \item We show that the task-dependent score can be expressed in terms of a Gaussian-kernel conditional mean. We further prove that its empirical Nadaraya--Watson approximation, and hence the corresponding empirical score estimator, can be represented exactly by a model in our constructed transformer class. This connection provides a principled foundation for the proposed transformer-based in-context score operator. Trained through amortized denoising score matching across task distributions, the operator enables reverse-diffusion sampling on unseen tasks without retraining, fine-tuning, or task-specific parameter updates.

    \item We provide a theoretical analysis of the proposed neural operator and establish a generalization bound under suitable assumptions. The bound separates the contributions of finite context, finite meta-training data, and model complexity, while making explicit the dependence of the in-distribution generalization error on the context size, the numbers of training tasks and samples, and the intrinsic and ambient dimensions of the problem.

    \item We evaluate the proposed method on diverse scientific and visual datasets, including Navier--Stokes solution fields, ERA5 climate data, QuickDraw sketches, and Food-101 images. The results demonstrate strong performance in low-intrinsic-dimension regimes and reveal a gradual degradation as task complexity increases. For low-intrinsic-dimension problems such as Navier--Stokes, the method also exhibits modest generalization to out-of-distribution tasks.
\end{itemize}

The code is available at \url{https://github.com/woodssss/In-context-learning-for-Diffusion-Models}. The repository also provides a Lean 4 formalization of Theorem~\ref{thm:generalization_error}.

\subsection{Related work}
Our framework is closely related to the broader paradigm of in-context learning (ICL), in which a model adapts to a new task through contextual demonstrations without updating its parameters. ICL was initially studied extensively in large language models \cite{brown2020gpt3,min2022rethinkingicl,zhao2021calibrate,min2022metaicl}, and related ideas have since been developed for scientific computing and operator learning. ICON \cite{yang2023context} studies supervised in-context operator learning for differential equations, using prompted input-output examples to adapt a single network to previously unseen operators. PROSE-PDE \cite{sun2025towards} takes a complementary multimodal approach, combining numerical solution data with symbolic equation representations for multi-operator learning. Self-supervised in-context operator learning in probability-measure spaces has also been investigated for optimal transport, mean-field games, and Wasserstein gradient flows \cite{huang2025unsupervised,feng2026learn,cole2026context}. This measure-valued task perspective is especially relevant to our formulation, in which each task is itself a target probability distribution. 

In parallel, several works have incorporated contextual demonstrations into diffusion-based image generation. Prompt Diffusion \cite{wang2023promptdiffusion}, iPromptDiff \cite{chen2023ipromptdiff}, and Context Diffusion \cite{najdenkoska2023contextdiffusion} demonstrate that diffusion models can use visual examples as context for image generation and editing. These methods, however, primarily address visually prompted image-to-image generation in paired or query-conditioned settings. They do not formulate diffusion modeling as in-context learning of a solution operator over a family of target probability distributions, which is the focus of the present work.
Other two related directions are Score Neural Operator (SNO) \cite{liao2024sno} and few-shot diffusion models \cite{giannone2022few}, both of which can be viewed as forms of context-conditioned generation. In SNO, the prompt set is first summarized by a distribution-level representation, specifically a kernel mean embedding, and the score function is then predicted conditioned on this embedding. In few-shot diffusion models, the support set is encoded into a task-dependent feature, for example by aggregating patch-level information through an encoder, and diffusion sampling is performed conditioned on this feature. In contrast, our method operates directly on the empirical context set and constructs score estimates through query-dependent interactions among the context samples, the noisy query point, and the diffusion time, rather than conditioning solely on a compressed global representation.

A growing theoretical literature seeks to explain the mechanisms and generalization properties underlying ICL. Existing studies interpret ICL as implicit Bayesian inference \cite{xie2021bayesian}, characterize the function classes that transformers can learn from demonstrations \cite{garg2022icl}, and connect transformer forward passes to gradient-based meta-learning \cite{vonoswald2023gd}. Other works show that transformers can implement statistical learning algorithms in context \cite{bai2023transformers} and study their generalization and stability from an algorithm-learning perspective \cite{li2023transformers}. The diversity and number of pretraining tasks have also been identified as central factors governing the emergence of ICL and its generalization to unseen tasks \cite{raventos2023pretraining}. In scientific computing, \cite{cole2024context} develops an operator-learning analysis of ICL for linear systems and elliptic PDEs. \cite{cole2026context} studies ICL directly on probability-measure spaces and establishes generalization-error bounds for in-context operator learning of optimal transport maps. Particularly relevant to our analysis is \cite{shen2026understanding}, which connects attention mechanisms to kernel regression on structured manifolds and establishes generalization guarantees governed by intrinsic rather than ambient dimension. Our theory complements these results by analyzing an in-context score operator and separating the effects of finite context, the number of meta-training tasks, within-task sampling, model complexity, and intrinsic dimension.

\paragraph{Organization. }The remainder of this paper is organized as follows. Section~\ref{sec:setup} introduces the problem setup and notation. Section~\ref{sec:method} presents the proposed method, including the training objective and algorithm. Section~\ref{sec:thm} provides a theoretical analysis of the generalization error. Section~\ref{sec:numerical} reports empirical results across a range of datasets. Finally, Section~\ref{sec:conclusion} concludes with a discussion of limitations and directions for future work.

\section{Problem Setup}\label{sec:setup}
In this section, we first review the background of diffusion models \cite{song2020score, yang2023diffusion} and then formalize the in-context score operator learning problem.

\subsection{Background on Diffusion Models}
The score-based diffusion models \cite{song2020score, yang2023diffusion} aim to approximate a target data distribution $\rho(\bx)$ given a dataset $\{\bx_i \}_{i=1}^N$, and it operates in two stages: a forward diffusion process and a reverse generative process. 

\paragraph{Forward Process.} In the forward process, data is progressively corrupted with noise through a stochastic differential equation (SDE), transforming the original data into a standard Gaussian distribution. Specifically, for a task $\rho \in \mathcal P(\mathbb R^D)$, consider a special case of the VP SDE:
\begin{equation}
\mathrm{d} \bx = -\tfrac12\beta(t)\bx\,\mathrm dt+\sqrt{\beta(t)}\,\mathrm d\mathbf{w}_t,
\end{equation}
The perturbed solution of this SDE at time $t$ with initial condition $\bx \sim \rho$ is:
\begin{equation}
\label{eqn:addnoise}
        \bx(t) : = \alpha(t)\bx+\sigma(t)\varepsilon,
    \qquad \varepsilon\sim\mathcal N(0,I).
\end{equation}
where $\beta(t)$ is a user-specific monotonically increasing function for $t\in[\delta,1]$,
\[
\alpha(t)=\exp\left(-\frac{1}{2}\int_{0}^{t}\beta(s)\,ds\right),
\qquad
\sigma^2(t)=1-\alpha^2(t).
\]
And we denote $\rho_{t}$ by the law of $\bx(t)$, $p_t$ by the density function of $\rho_t$. 

\paragraph{Reverse Process.} The reverse process is described by the corresponding reverse time SDE that progressively transforms the standard Gaussian distribution back into the original data distribution.
\begin{equation}\label{eqn:reverse}
    \mathrm{d}\mathbf{x} = \left[ -\frac{1}{2} \beta(t) \bx -   \beta(t) \nabla_{\mathbf{x}} \log p_t(\bx)\right] \mathrm{d}t + \sqrt{\beta(t)}  \mathrm{d}\bar{\mathbf{w}},
\end{equation}
where $\mathrm{d}\mathbf{\bar{w}}$ represents a reverse-time Wiener process and $\log p_t(\mathbf{x})$ is the score function of the marginal distribution of the forward process at time $t$. Moreover, Song et al. \cite{song2020score} proved the existence of the probability flow (PF) ODE, which shares the same marginal distributions as the reverse-time SDE \eqref{eqn:reverse}:
\begin{equation}\label{eqn:pfode_ori}
    \frac{\rd \bx}{\rd t} = -\frac{1}{2} \beta(t) \bx - \frac{1}{2} \beta(t) \nabla_{\mathbf{x}} \log p_t(\bx).
\end{equation}
In this paper, we let the noise scheduling function be $\beta(t) = \beta_0 + (\beta_1 - \beta_0) t$, with $\beta_0=0.1$ and $\beta_1=20$. 

\paragraph{Score matching.} Before introducing our in-context learning objective, we briefly recall the standard denoising score matching formulation. The goal is to learn the time-dependent score field $\nabla_{\bx}\log p_t(\bx)$. A standard denoising score matching objective trains a score network $s_\theta$ by minimizing
\begin{equation}
 L_\rho(s_\theta) = 
\mathbb E_{\bx \sim \rho,\;t\sim \mathrm{U}[\delta,1],\ \varepsilon\sim\mathcal N(0,I)}
\left[
\left\|
s_\theta(\bx(t),t)+\frac{\varepsilon}{\sigma(t)}
\right\|_2^2
\right]
\end{equation}
Here we restrict $t$ to $[\delta,1]$ for some fixed $\delta>0$ to stay away from the zero-noise endpoint $t=0$, where the data distribution may be singular and the score can become ill-behaved as $\sigma(t)\to 0$. The terminal endpoint $t=1$, however, does not pose such a difficulty. Although its law is not exactly standard Gaussian, it is well approximated by a standard Gaussian when the accumulated noise level $\int_0^1\beta(s)\,\mathrm ds$ is sufficiently large, since $\alpha(1)\approx0$, $\sigma(1)\approx1$ and $\bx(1)=\alpha(1)\bx+\sigma(1)\varepsilon$.

\paragraph{Generation.} For the generation phase, any black-box ODE solver may be used to integrate \eqref{eqn:pfode_ori} initiated at a point sampled from the standard Gaussian, making probability flow (PF) ODE sampling computationally efficient. Alternatively, one may generate samples by integrating the reverse-time SDE \eqref{eqn:reverse} with a general SDE solver. In this work, we discretize the reverse-time SDE using the Euler--Maruyama scheme.

\subsection{In-Context Score Function Learning}\label{sec:setup_icl}
Standard denoising score matching treats a fixed target distribution and learns a single time-dependent score function for it. In many applications, however, generation is required not for a single fixed distribution but for a family of related task distributions, each corresponding to a different task, condition, or physical parameter. Learning an independent score model for each distribution would fail to exploit shared structure across the family and would provide no mechanism for adapting to previously unseen tasks. We therefore amortize score learning across the task family by learning a score operator that maps each task distribution to its corresponding score function. More specifically, we assume that tasks are sampled from a meta-distribution over task distributions.
\[
\mu \in \mathcal{P}(\mathcal{P} (\mathbb{R}^D)).
\] 
Thus, each task $\rho \sim \mu$ is itself a probability distribution on $\mathbb R^D$. Given a diffusion time interval $t \in [\delta,1]$ with $\delta>0$. for each task $\rho$, we define the perturbed sample
\[
\bx(t)=\alpha(t)\bx+\sigma(t)\varepsilon,
\qquad
\bx\sim \rho,\quad \varepsilon\sim\mathcal N(0,I),
\]
and let $\rho_t$ denote the law of $\bx(t)$. Denote by $p_{\rho,t}$ the density of $\rho_t$. The corresponding task-dependent score function is then defined as
\[
S_\rho(\bx,t):=\nabla_{\bx}\log p_{\rho,t}(\bx),
\qquad
(\bx,t)\in\mathbb R^D\times[\delta,1].
\]

In this way, each task distribution $\rho$ gives rise to a score function $S_\rho$, which naturally defines the following operator
\[
T:
\mathcal P(\mathbb R^D)
\to
\Bigl\{
S:\mathbb R^D\times[\delta,1]\to\mathbb R^D
\Bigr\},
\qquad
\rho\mapsto T(\rho)=S_\rho.
\]
At the continuous level, we seek to learn an optimal score operator by amortized denoising score matching over the meta-distribution of tasks. This leads to the population loss
\begin{equation}\label{eqn:DSM_pop}
    \mL(T)
    := 
    \mathbb E_{\rho\sim\mu\ } [L_\rho(T(\rho))] = \mathbb E_{\rho\sim\mu\ }
    \mathbb E_{\substack{ \bx\sim\rho,\\ t\sim \mathrm{U}[\delta,1],\ \varepsilon\sim\mathcal N(0,I)}}
    \left[
        \left\|
            T(\rho)\big(\bx(t),t\big)
            + \frac{\varepsilon}{\sigma(t)}
        \right\|_2^2
    \right],
\end{equation}
and we denote $T^*$ a population-risk minimizer.

This formulation yields a distribution-level in-context score-learning problem in which the task distribution itself acts as the context. In practice, however, the learner does not observe the meta distribution $\rho$ directly; it only receives a finite collection of samples drawn from $\rho$. The corresponding finite-context problem is therefore to infer, from these examples, the score function associated with the underlying task distribution and to evaluate it at new noisy inputs and diffusion times. A learned approximation to $T^*$ can consequently adapt to previously unseen task distributions through their observed context samples, without requiring task-specific parameter updates. In the next section, we will provide details on how to compute this model.

\section{Methodology}\label{sec:method}
In Section~\ref{sec:setup_icl}, we formulated the in-context score learning problem at the continuous level, where the optimal map is defined on the space of task distributions and the learning objective is given by the population loss. To approximate this map in practice, we parameterize it by a neural operator 
\(
T_\theta\in\mathcal T,
\)
where $\mathcal T$ is the hypothesis space parameterized by a deep neural network and will be specified in Section~\ref{sec:kernel_transformer}. Since the underlying task distribution is not directly available to the model, we must further replace it by a finite empirical representation. This leads to a sequence of loss functions at four successive levels of discretization: the semi-discretized population loss, the task-discretized loss, the task- and episode-discretized loss, and the fully discretized loss.

\paragraph{Semi-discretized Population Loss.}
In practice, however, the underlying task distribution is not directly observed, and the model cannot operate on it explicitly. Instead, both during training and at test time, each task is specified only through finitely many samples. We therefore now turn to a semi-discretized formulation, in which each task is represented by a finite context set and the population-level objective is replaced by a practical training loss that can be optimized from sampled tasks. Concretely, for a task $\rho\sim\mu$, we observe a context set (prompt) of $n$ i.i.d.\ samples $\bx^{(1)},\dots,\bx^{(n)}\sim\rho$ and represent it by the empirical measure
\(
\hat{\rho} := \frac{1}{n}\sum_{i=1}^n \delta_{\bx^{(i)}}
\) together with a query point $(\bx(t),t)$ that serves as the input of the transformer network. Namely, we have a transformer network:
\begin{equation}
    T_\theta: (\mathbb{R}^D)^n\times\mathbb{R}^D\times [\delta ,1] \rightarrow \mathbb{R}^D,\qquad (\bx^{(1)},\dots,\bx^{(n)};\bx(t), t)\mapsto T_\theta(\hat\rho)(\bx(t), t)
\end{equation}
where $\bx(t)$ is provided by equation \eqref{eqn:addnoise}. We write the corresponding score matching loss for the task $\rho$ as
\begin{equation}
   L_\rho^n (T_\theta) =  \mathbb E_{\substack{ \hat \rho \sim \rho^{\otimes n},\ \bx\sim\rho,\\ t\sim \mathrm{U}[\delta,1],\ \varepsilon\sim\mathcal N(0,I)}}
    \left[
        \left\|
            T_\theta(\hat \rho)\big(\bx(t),t\big)
            + \frac{\varepsilon}{\sigma(t)}
        \right\|_2^2
    \right],
\end{equation}

This leads to the following semi-discretized population operator learning loss:
\begin{equation}\label{eqn:DSM_semi_pop}
     \mL_n(T_\theta)
    :=  \mathbb E_{\rho\sim\mu} [L_\rho^n (T_\theta)] = \mathbb E_{\rho\sim\mu}\mathbb E_{\substack{ \hat \rho \sim \rho^{\otimes n},\ \bx\sim\rho,\\ t\sim \mathrm{U}[\delta,1],\ \varepsilon\sim\mathcal N(0,I)}}
    \left[
        \left\|
            T_\theta(\hat \rho)\big(\bx(t),t\big)
            + \frac{\varepsilon}{\sigma(t)}
        \right\|_2^2
    \right],
\end{equation}

\paragraph{Task-discretized loss.}
We next discretize the outer expectation over tasks in \eqref{eqn:DSM_semi_pop} by replacing it with the empirical average over $K$ training tasks $\{\rho_k\}_{k=1}^K$, sampled i.i.d.\ from the meta-distribution $\mu$. This gives
\begin{equation}\label{eqn:DSM_K}
\mL_{n,K}(T_\theta)
:=
\frac{1}{K}
\sum_{k=1}^K
\mathbb E_{\substack{\hat \rho_k \sim \rho_k^{\otimes n},\ \bx_k\sim\rho_k,\\ t\sim \mathrm{U}
[\delta,1],\ \varepsilon\sim\mathcal N(0,I)}}
\left[
    \left\|
        T_\theta(\hat \rho_k)\big(\bx_k(t),t\big)
        + \frac{\varepsilon}{\sigma(t)}
    \right\|_2^2
\right].
\end{equation}

\paragraph{Task- and episode-discretized loss.}
For each sampled task $\rho_k$, we further discretize the expectation over context-query draws by replacing it with a finite average over $J$ sampled episodes indexed by $j$. Let $\hat\rho_{k,j}$ denote the empirical measure of the context set in the $j$-th episode for task $k$, and let $\bx_{k,j}\sim\rho_k$ be the corresponding query sample. For each pair $(k,j)$, define
\[
\bx_{k,j}(t):=\alpha(t)\bx_{k,j}+\sigma(t)\varepsilon,
\qquad
t\sim \mathrm{U}[\delta,1],\ \varepsilon\sim\mathcal N(0,I).
\]
Then we define
\begin{equation}\label{eqn:DSM_KJ}
\mL_{n,K,J}(T_\theta)
:=
\frac{1}{KJ}
\sum_{k=1}^K\sum_{j=1}^J
\mathbb E_{t\sim \mathrm{U}[\delta,1],\ \varepsilon\sim\mathcal N(0,I)}
\left[
    \left\|
        T_\theta(\hat \rho_{k,j})\big(\bx_{k,j}(t),t\big)
        + \frac{\varepsilon}{\sigma(t)}
    \right\|_2^2
\right].
\end{equation}

\paragraph{Fully discretized empirical loss.}
Finally, we introduce the fully discretized empirical loss, which serves as the computable training objective in practice. Let $\{\rho_k\}_{k=1}^K$ be training tasks drawn i.i.d.\ from the meta-distribution $\mu$. For each task $\rho_k$, we construct an episodic training set and define
\begin{equation}\label{eqn:training_set}
    \mathfrak S := \{\mathfrak s_k\}_{k=1}^K,
\qquad
\mathfrak s_k := \{\mathbf s_{k,j,l}\}_{j=1,\ldots,J;\,l=1,\ldots,N_T},
\end{equation}
where $j$ indexes independent context--query pair draws within task $k$, and $l$ indexes independent draws of diffusion time and noise. More precisely, for each triplet $(k,j,l)$, we first draw $n+1$ i.i.d.\ samples from $\rho_k$,
\[
\bx_{k,j}^{(1)},\ldots,\bx_{k,j}^{(n)},\bx_{k,j}
\stackrel{\mathrm{i.i.d.}}{\sim}\rho_k,
\]
where $\bx_{k,j}^{(1)},\ldots,\bx_{k,j}^{(n)}$ are used to form the context and $\bx_{k,j}$ serves as the clean query sample. We then sample a diffusion time
\(
t_{k,j,l}\sim \mathrm{U}[\delta,1]
\)
and noise
\(
\varepsilon_{k,j,l}\sim \mathcal N(0,I),
\)
and construct the corresponding noised query
\[
\bx_{k,j}(t_{k,j,l})
:=
\alpha(t_{k,j,l})\,\bx_{k,j}
+
\sigma(t_{k,j,l})\,\varepsilon_{k,j,l}.
\]
The associated empirical context is
\(
\hat\rho_{k,j}
:=
\frac{1}{n}\sum_{i=1}^n \delta_{\bx_{k,j}^{(i)}}.
\) Hence each training episode is represented by the tuple
\[
\mathbf s_{k,j,l}
=
\Big\{\hat\rho_{k,j},\ \bx_{k,j}(t_{k,j,l}),\ t_{k,j,l}\Big\}.
\]
By replacing the remaining expectation over the diffusion time and Gaussian noise in \eqref{eqn:DSM_KJ} with the empirical average over the $N_T$ sampled pairs, we obtain the fully discretized objective:
\begin{equation}\label{eqn:DSM_empirical}
\mathcal L_{\mathfrak S}(\theta)
=
\frac{1}{KJN_T}
\sum_{k=1}^K\sum_{j=1}^J\sum_{l=1}^{N_T}
\left\|
T_\theta(\hat\rho_{k,j})\big(\bx_{k,j}(t_{k,j,l}),\, t_{k,j,l}\big)
+
\frac{\varepsilon_{k,j,l}}{\sigma(t_{k,j,l})}
\right\|_2^2.
\end{equation}
which serves as the computable training loss used in practice. We then define the empirical loss minimizer by
\[
T_{\theta^*,\mathfrak S}
\in
\arg\min_{T_\theta\in\mathcal T}\ \mathcal L_{\mathfrak S}(\theta).
\]
In practice, training is performed in an amortized manner across tasks by minimizing the empirical denoising score-matching loss. At inference time, given the empirical context of a previously unseen task, the learned neural operator is used as a score estimator in the reverse-time SDE, from which new samples are generated using the Euler--Maruyama scheme. The training and sampling procedures are summarized in Algorithms~\ref{alg:icldm_train} and~\ref{alg:icldm_sampling}, respectively.

\begin{algorithm}[H]
\caption{Training of the In-Context Diffusion Model}
\label{alg:icldm_train}
\small
\begin{algorithmic}[1]

\Statex \textbf{Input:} training set $\mathfrak{S}=\{\mathfrak{s}_k\}_{k=1}^K$, neural operator $T_\theta$, number of epochs $N_{\mathrm{epoch}}$
\Statex \hspace{\algorithmicindent} where
\[
\mathfrak{s}_k=\{\bs_{k,j,l}\}_{j=1,\ldots,J;\,l=1,\ldots,N_T},
\qquad
\bs_{k,j,l}
=
\{\hat{\rho}_{k,j},\,\bx_{k,j}(t_{k,j,l}),\,t_{k,j,l}\}.
\]

\State Initialize $\theta$
\For{$e=1,\ldots,N_{\mathrm{epoch}}$}
    \State Compute the empirical DSM loss $\mathcal{L}_{\mathfrak S}(\theta)$ in \eqref{eqn:DSM_empirical}
    \State Update $\theta$ using Adam
\EndFor
\State \Return trained operator $T_{\theta^*,\mathfrak S}$

\end{algorithmic}
\end{algorithm}

\begin{algorithm}[H]
\caption{Sampling on a New Task}
\label{alg:icldm_sampling}
\small
\begin{algorithmic}[1]

\Statex \textbf{Input:} empirical context $\hat{\rho}_{\mathrm{new}}=\{\bx_{\mathrm{new}}^{(i)}\}_{i=1}^n$, trained operator $T_{\theta^*,\mathfrak S}$, number of reverse steps $N_T$

\State Set $\Delta t=(1-\delta)/N_T$ and define $t_m=\delta+m\Delta t$ for $m=0,\ldots,N_T$
\State Sample $\tilde{\bx}_{t_{N_T}} \sim \mathcal N(0,I)$
\For{$m=N_T-1, N_T-2, \ldots, 0$}
    \State Sample $\bz_m \sim \mathcal N(0,I)$
    \State $\displaystyle
    \begin{aligned}[t]
    \tilde{\bx}_{t_m} \leftarrow\;& \tilde{\bx}_{t_{m+1}}
    + \Big(
    \tfrac12 \beta(t_{m+1})\,\tilde{\bx}_{t_{m+1}} \\
    &\qquad
    + \beta(t_{m+1})
    T_{\theta^*,\mathfrak S}(\hat{\rho}_{\mathrm{new}})
    (\tilde{\bx}_{t_{m+1}}, t_{m+1})
    \Big)\Delta t \\
    &\qquad
    + \sqrt{\beta(t_{m+1})\,\Delta t}\,\bz_m
    \end{aligned}$
\EndFor
\State \Return generated sample $\tilde{\bx}_{t_0}$

\end{algorithmic}
\end{algorithm}

\section{Theoretical results}\label{sec:thm}
We now present the theoretical results for the proposed in-context score operator. We begin with the observation that the task-dependent population score admits a Gaussian-kernel conditional-mean representation. Replacing the unknown task distribution with its empirical context measure naturally yields a vector-valued Nadaraya–Watson estimator and, consequently, an empirical score estimator. This kernel representation observation motivates both our transformer architecture and its theoretical analysis: we show that the empirical estimator can be represented exactly by a model in our constructed transformer class, revealing a direct correspondence between Gaussian-kernel regression and the attention mechanism. Guided by this connection, we learn an in-context score operator through amortized denoising score matching across task distributions and analyze its generalization to unseen tasks drawn from the meta-distribution. Section~\ref{sec:kernel_transformer} develops the kernel representation and its exact transformer realization. Section~\ref{sec:main_theorem} then presents and interprets the main generalization theorem, followed by Section~\ref{sec:proof_strategy}, which outlines the corresponding error decomposition and the proof strategy.

\subsection{Understanding the Target Score learning: From Kernel Regression to Transformers}\label{sec:kernel_transformer}

In this subsection, we connect the task-dependent score function with a kernel regression estimator that can be represented by a transformer network. For a fixed task distribution $\rho$, recall that the perturbed law at time $t$ admits density $p_t$. The corresponding score function can be written as
\begin{equation}
T^*(\rho)(\bx(t), t) = S(\bx(t), t)
=
\nabla_{\bx}\log p_t(\bx)\big|_{\bx=\bx(t)}
=
-\frac{1}{\sigma^2(t)}
\Big[
\bx(t)-\alpha(t)\,\mathbb E(\bx\mid \bx(t))
\Big].
\end{equation}
Here the conditional expectation is given by
\begin{equation}\label{eqn:cond_exp}
\mathbb E(\bx\mid \bx(t))
=
\frac{
\int \bx \exp\left(-\frac{\|\bx(t)-\alpha(t)\bx\|^2}{2\sigma^2(t)}\right)\, \rd\rho(\bx)
}{
\int \exp\left(-\frac{\|\bx(t)-\alpha(t)\bx\|^2}{2\sigma^2(t)}\right)\, \rd\rho(\bx)
}.
\end{equation}

Given context samples $\hat \rho = \{\bx^{(1)},\dots,\bx^{(n)}\}$, we approximate \eqref{eqn:cond_exp} by the empirical kernel regression estimator
\begin{equation}\label{eqn:cond_exp_discretized}
\mathcal K(\bs)
=
\frac{
\sum_{i=1}^n \bx^{(i)}
\exp\left(-\frac{\|\bx(t)-\alpha(t)\bx^{(i)}\|^2}{2\sigma^2(t)}\right)
}{
\sum_{i=1}^n
\exp\left(-\frac{\|\bx(t)-\alpha(t)\bx^{(i)}\|^2}{2\sigma^2(t)}\right)
},
\end{equation}
where ${\bs}=\{\bx^{(1)},\dots,\bx^{(n)},\bx(t),t\}$. We first observe that \eqref{eqn:cond_exp} admits a kernel representation and that its empirical discretization, given in \eqref{eqn:cond_exp_discretized}, is a vector-valued Nadaraya--Watson kernel estimator. Each coordinate of this estimator is a scalar-valued kernel regressor and, by Lemma~1 of \cite{shen2026understanding}, can be approximated by a transformer network with bounded parameters. Applying this construction coordinatewise yields a transformer representation of the vector-valued estimator. This connection between kernel regression and transformer attention is the key observation motivating our use of transformers to approximate the task-dependent score operator. For completeness, we first define the transformer class considered in our analysis and then state the representation lemma for the empirical kernel estimator.

\paragraph{Scalar Transformer Network Class.} The scalar transformer network architecture that defined in \cite{shen2026understanding} is as follows. Given an input $\bs = \{ \bx^{(1)}, \bx^{(2)}, \cdots, \bx^{(n)}; \bx(t), t\}$, where $x_h^{(i)}$ is the $h$-th component of vector $\bx^{(i)}$, for $h=1, \cdots, D$ and $i=1, \cdots, n$. We define the output of the transformer network $F_\theta$ by

\begin{equation}\label{eqn:F_architech}
    H = \mathrm{PE}+E(\bs) =
    \begin{bmatrix}
    \bx^{(1)}  & \cdots &  \bx^{(n)} & \bx(t) \\
    0 & \cdots & \cdots &  0 \\
    \vdots & \vdots & \vdots & \vdots \\
    0 & \cdots & \cdots  & 0 \\
    \mathcal{I}_1 & \cdots & \cdots &  \mathcal{I}_\ell \\
    1 & \cdots & \cdots  & 1
    \end{bmatrix}
    \in \mathbb{R}^{d_{\mathrm{embed}}\times \ell}
    =
    \mathbb{R}^{(2D+6)\times (n+1)}.
\end{equation} 
    In matrix $H$, each column is a token, and each token has dimension $d_{\mathrm{embed}} = 2D + 6$ . The first $D$ rows are data terms and
    the $(D+1)$th row to $(2D+3)$th row are $0$. The $(2D+4)$-th and $(2D+5)$-th rows contain the
    well-known sinusoidal positional encodings
    \[
    \mathcal{I}_j
    =
    \left(
    \cos\left(\frac{j\pi}{2\ell}\right),
    \sin\left(\frac{j\pi}{2\ell}\right)
    \right)^\top,
    \]
    which determines how each token will interact with another through the
    attention mechanism. It is crucial to note that the data terms are dynamic,
    whereas the positional encoding and constant terms remain static.
    Furthermore, 
    \[
    B_1,\cdots,B_{L_T}:\mathbb{R}^{d_{\mathrm{embed}}\times \ell}
    \to
    \mathbb{R}^{d_{\mathrm{embed}}\times \ell}
    \]
    are the transformer blocks which have a residual composition of the form
\begin{equation}
\label{eqn:Tblock}
       {\rm B}(\theta,t;H) = {\rm FFN}({\rm MHA}_{\theta,t}(H) + H) + {\rm MHA}_{\theta,t}(H) + H.
\end{equation}
where {\rm FFN} is a feed-forward neural network operating tokenwise on the input,  and the multi-head attention (MHA) with $m$ heads is given by
$
  \textstyle  {\rm MHA}_{\theta,t}(H)=\sum_{j=1}^m V_{\theta,j}(t)H\phi((K_{\theta,j}(t)H)^\top Q_{\theta,j}(t)H)
$
with ReLU activation from the first to the
    penultimate layers and Softmax activation for the last layer in the attention
    module  (see definition 2 in~\cite{shen2026understanding}). The final attention mask restricts normalization to the $n$ prompt tokens. The attention matrices use prescribed time modulation: for $A\in\{Q,K,V\}$, each matrix has the form
\begin{equation}\label{eq:fixed-schedule}
A_{\theta,j}(t)=A_{\theta,j}^{(0)}+\alpha(t)A_j^{(\alpha)}+\sigma(t)^{-1}A_j^{(\sigma)}.
\end{equation}
The block index is suppressed. Here $\alpha$ and $\sigma$ are the prescribed diffusion schedules, and $A_j^{(\alpha)},A_j^{(\sigma)}$ are fixed architectural templates, selected before training and shared by all networks in the class. Only the base matrices $A_{\theta,j}^{(0)}$ are trainable within these attention maps. A modulation term may be absent by setting its template to zero. The parameter vector $\theta$ contains the trainable base matrices and the trainable embedding, decoding, and feed-forward coefficients, all shared across time; the fixed templates are excluded from $\theta$. Both the trainable coefficients and the fixed template entries are bounded in magnitude by $\kappa$. The output bound below is imposed uniformly over $t\in[\delta,1]$. Thus ``fixed'' refers to the schedules and templates, while the input time $t$ still varies. Lastly,
    \[
    \mathrm{DE}:\mathbb{R}^{d_{\mathrm{embed}}\times \ell}\to\mathbb{R}
    \]
    is the decoding layer which outputs the desired element. We then denote a class of bounded scalar-value transformer networks as following, 
    \begin{equation}
    \begin{aligned}
    &\mathcal{F}(L_T,m_T,d_{\mathrm{embed}},\ell,L_{\mathrm{FFN}},w_{\mathrm{FFN}},R,\kappa) \\
    &= \Bigl\{\, F_\theta(\cdot)\  \Big|\ 
    F_\theta(\cdot)\ \text{has the form }\eqref{eqn:F_architech} \text{ with }L_T\text{ transformer blocks, at most }m_T\text{ attention heads in} \\
    & \quad \quad \quad \quad \quad \text{each block, embedded dimension }d_{\mathrm{embed}},\text{ number of hidden tokens }\ell,\text{ and }L_{\mathrm{FFN}}\text{ layers} \\
    & \quad \quad \quad \quad \quad \text{of feed-forward networks with hidden width }w_{\mathrm{FFN}},\text{ with output }\|F_\theta(\cdot)\|_{L^\infty(\mathbb{R}^D)} \le R \\
    & \quad \quad \quad \quad \quad \text{and weight magnitude }\|\theta\|_\infty \le \kappa \,\Bigr\}.
    \end{aligned}
    \end{equation}

We now extend the above defintion to a vector-valued transformer network class. 

\paragraph{Vector-Valued Transformer Network Class.} For $D \in \mathbb{N}_1$, we define the vector-valued transformer network class by $\mathcal{G}(L_T,m_T,d_{\mathrm{embed}},\ell,L_{\mathrm{FFN}},w_{\mathrm{FFN}},R,  \kappa, D)$, whose components belong to $\mathcal{F}(L_T,m_T,d_{\mathrm{embed}},\ell,L_{\mathrm{FFN}},w_{\mathrm{FFN}},R,  \kappa)$. That is for $G_\theta \in \mathcal{G}(L_T,m_T,d_{\mathrm{embed}},\ell,L_{\mathrm{FFN}},w_{\mathrm{FFN}},R,  \kappa, D)$, we have
\[
G_\theta(\bs) = \left( F_{1, \theta}(\bs), \cdots, F_{D, \theta}(\bs) \right)
\]
where $F_{i, \theta} \in \mathcal{F}(L_T,m_T,d_{\mathrm{embed}},\ell,L_{\mathrm{FFN}},w_{\mathrm{FFN}},R,  \kappa)$ for $i = 1, \cdots, D$. 

Using the coordinatewise construction motivated by Lemma 1 of \cite{shen2026understanding}, we obtain a transformer network $F_\theta \in \mathcal{F}\left(L_T,m_T,d_{\mathrm{embed}},\ell,L_{\mathrm{FFN}},w_{\mathrm{FFN}},R,\kappa\right)$ that, on the noisy-query region specified below, satisfies 
\begin{equation}\label{eqn:kernel_regression_component}
    F_{\theta}(\bs) = \frac{
    \sum_{i=1}^n x_j^{(i)}
    \exp\left(-\frac{\|\bx(t)-\alpha(t)\bx^{(i)}\|^2}{2\sigma^2(t)}\right)
    }{
    \sum_{i=1}^n
    \exp\left(-\frac{\|\bx(t)-\alpha(t)\bx^{(i)}\|^2}{2\sigma^2(t)}\right)
    }.
\end{equation}
Note that the \eqref{eqn:kernel_regression_component} represents the $j$-th component of the kernel regression estimator \eqref{eqn:cond_exp_discretized}. By stacking $D$ such networks componentwise, we can represent the full vector-valued kernel regression estimator in \eqref{eqn:cond_exp_discretized}, which yields the following lemma indicate relation between kernel regression formula and vector-valued network class.

Let $n\ge1$, $R\ge b\ge1$, and let the fixed schedule satisfy Assumption~\ref{ass:noise}. For $C>0$, define
\begin{equation}\label{eq:query-region}
\mathcal X_t^C:=\{\alpha(t)\bx+\sigma(t)\eps:
\bx\in\mM,\ \|\eps\|_2\le C\}.
\end{equation}

\begin{lemma}[Transformer realization]\label{lemma:transformer_kernel_regression}
There exists a network
\[G_{\hat\theta}\in\mathcal G(L_T,m_T,d_{\mathrm{embed}},\ell,L_{\mathrm{FFN}},w_{\mathrm{FFN}},R,\kappa,D)\]
with parameters
\begin{equation}\label{eq:architecture}
\begin{gathered}
L_T=5,\qquad m_T=n,\qquad d_{\mathrm{embed}}=2D+6,\qquad\ell=n+1,\\
L_{\mathrm{FFN}}=O(1),\qquad w_{\mathrm{FFN}}=O(D),
\end{gathered}
\end{equation}
whose parameters are independent of the task, prompt, query, and time, such that
\begin{equation}\label{eq:realization}
G_{\hat\theta}(\bs)=\mathcal K(\bs)\quad(\bx(t)\in\mathcal X_t^C),
\qquad \|G_{\hat\theta}(\bs)\|_\infty\le b\quad(\bx(t)\in\R^D).
\end{equation}
A sufficient weight budget is $\kappa=c_0D^4(n+1)^2M^4$, where
\begin{equation}\label{eq:M}
M=\max\left\{1,R,\frac12\left(\frac{2b\sqrt D}{\sigma(\delta)}+C\right)^2\right\}.
\end{equation}
One may take $L_{\mathrm{FFN}}=6$ and $w_{\mathrm{FFN}}=2D+6$.
\end{lemma}

It is worth emphasizing that the transformer parameters are independent of the task distribution. Rather than encoding a particular distribution, the network implements the generic computations underlying kernel regression: evaluating query--context similarities, normalizing the resulting kernel weights, and aggregating the context samples. All task-specific information therefore enters through the empirical context set, while the shared parameters encode a distribution-agnostic inference algorithm. Consequently, the same transformer can be applied uniformly to any task distribution within the considered class; adapting to a new task requires changing only the prompt, not the network parameters. This observation clarifies the operator-learning nature of the construction and provides a theoretical basis for generalization to previously unseen task distributions. We defer the constructive proof of Lemma~\ref{lemma:transformer_kernel_regression} to Appendix~\ref{sec:exact_transformer_realization}.

\paragraph{Vector-Valued Transformer Network Class for In-Context Score Matching.}
We now define the transformer network class used to parameterize the task-dependent score function. Recall that the score admits the representation
\[
S(\bx(t),t)
=
-\frac{1}{\sigma^2(t)}
\Big[\bx(t)-\alpha(t)\,\E(\bx\mid \bx(t))\Big].
\]
Motivated by this identity, we let
\(
G_\theta \in \mathcal G\left(L_T,m_T,d_{\mathrm{embed}},\ell,L_{\mathrm{FFN}},w_{\mathrm{FFN}},R,\kappa,D\right)
\)
be a vector-valued transformer network that approximates the conditional expectation term. We then define the corresponding score network
\[
T_\theta \in \mathcal T\left(L_T,m_T,d_{\mathrm{embed}},\ell,L_{\mathrm{FFN}},w_{\mathrm{FFN}},R,\kappa,D\right)
\]
by
\begin{equation}\label{eqn:final_score_transformer}
T_\theta\left(\bx^{(1)},\ldots,\bx^{(n)},\bx(t),t\right)
:=
-\frac{1}{\sigma^2(t)}
\left[
\bx(t)-\alpha(t)\,
G_\theta\left(\bx^{(1)},\ldots,\bx^{(n)},\bx(t),t\right)
\right].
\end{equation}

For the analysis, take $R\ge b$. The componentwise output bound of $\mathcal G$ implies
\begin{equation}\label{eq:mean-bound}
\sup_{\bx^{(1:n)}\in\mM^n,\,\bx(t)\in\R^D,\,t\in[\delta,1]}
\|G_\theta(\bs)\|_\infty\le R.
\end{equation}
The bound applies to the mean network, not to the score in~\eqref{eqn:final_score_transformer} on all of $\R^D$. The mean network uses the prescribed time modulation defined in~\eqref{eq:fixed-schedule}. The fixed templates are those of the existing kernel construction; they are independent of its coordinate, localization constant, and clipping threshold. Thus the same construction belongs to this parameter class. Neither an additional trainable time network nor an additional transformer block is introduced.

Consequently, the kernel conditional-mean estimator \eqref{eqn:cond_exp_discretized} can be represented exactly on $\mathcal X_t^C$ by a mean network in $\mathcal G$. Applying~\eqref{eqn:final_score_transformer} yields the corresponding empirical score estimator in $\mathcal T$ on the same region. This connection establishes a direct link between the target score operator, kernel regression, and the attention mechanism.

\subsection{Generalization Error Analysis}\label{sec:main_theorem}

This subsection analyzes the generalization performance of the empirical risk minimizer
$T_{\theta^*;\mathfrak S}$ defined in \eqref{eqn:DSM_empirical}. The training set
$\mathfrak S$ contains $K$ task distributions, with $J$ context--query episodes sampled from
each task and $N_T$ diffusion samples generated for each episode. Consider an independent test task $\rho \sim \mu$, observed through context samples,
\(
\bx^{(1)},\ldots,\bx^{(n)} \overset{\mathrm{i.i.d.}}{\sim} \rho,
\hat{\rho} = \frac{1}{n}\sum_{i=1}^{n}\delta_{\bx^{(i)}}.
\)
At the task level, we compare the risk of the learned operator, evaluated using the empirical
representation $\hat{\rho}$, with that of the target distribution-level score operator
$T^*(\rho)$. The corresponding excess risk is
\[
L_\rho^n \left(T_{\theta^*;\mathfrak S}\right)
-
L_\rho \left(T^*(\rho)\right).
\]
Averaging over unseen tasks drawn from the meta-distribution yields the mean excess risk
\[
\mathcal L_n \left(T_{\theta^*;\mathfrak S}\right)
-
\mathcal L(T^*)
=
\mathbb E_{\rho\sim\mu}
\left[
L_\rho^n \left(T_{\theta^*;\mathfrak S}\right)
-
L_\rho \left(T^*(\rho)\right)
\right].
\]
This quantity measures the performance gap between the empirical risk minimizer and the target
score operator on an unseen task drawn from the same meta-distribution and represented by only
$n$ context samples. Because $T_{\theta^*;\mathfrak S}$ depends on the randomly sampled training
set $\mathfrak S$, we evaluate its performance by taking the expectation of the excess risk with respect to the randomness of $\mathfrak S$.  Throughout this subsection, we therefore use
the term \emph{generalization error} to denote the expected excess risk
\[
\mathbb E_{\mathfrak S}
\left[
\mathcal L_n \left(T_{\theta^*;\mathfrak S}\right)
-
\mathcal L(T^*)
\right],
\]
where the outer expectation accounts for the randomness in the meta-training set.

The accuracy of the empirical kernel estimator depends on how well the context samples cover the support of the task distribution $\rho$. In a general ambient space $\mathbb R^D$, achieving sufficient coverage may require an impractically large context set. Fortunately, many high-dimensional datasets possess substantially lower-dimensional intrinsic structure. Natural images provide a familiar example: although they are represented as high-dimensional pixel arrays, semantically meaningful images are often modeled as concentrating near a low-dimensional image manifold. Similar structure arises in scientific applications. For instance, in the Navier--Stokes problem considered in Section~\ref{sec:numerical}, task distributions induced by different viscosity parameters retain shared coherent flow structures. Although the probability law varies with viscosity, the resulting solution fields may remain concentrated near a common low-dimensional solution manifold. Motivated by these observations, we assume that all task distributions are supported on a common $d$-dimensional manifold $\mathcal M \subset \mathbb R^D$, where $d \ll D$. Under this assumption, the complexity of context-based score estimation is governed primarily by the intrinsic dimension $d$, rather than by the potentially much larger ambient dimension $D$. For the finite-sample analysis, a uniform loss bound follows from the bounded mean-network output and the positive time cutoff. The precise assumptions are stated below.

\begin{assumption}[Regularity of the task distributions]
\label{assump: assump_task_regularity}
We assume the following:
\begin{enumerate}
    \item There exists a compact $d$-dimensional Riemannian manifold $\mM \subset \mathbb{R}^D$ such that
every task distribution $\rho$ under consideration is supported on $\mM$. Moreover, $\mM$ is isometrically
embedded in $\mathbb{R}^D$, satisfies $\mM \subset [-b,b]^D$ for some $b>0$, and has positive reach $
\tau_{\mM}>0$.
    \item Each task distribution $\rho$ is quasi-uniform on $\mM$. That is, if $\bar{\rho}$ denotes the
uniform probability measure on $\mM$, then there exist constants $C_{\rho,1},C_{\rho,2}>0$ such that
    \[
    C_{\rho,1}\bar{\rho}(B)\le \rho(B)\le C_{\rho,2}\bar{\rho}(B)
    \]
    for every measurable set $B\subseteq \mM$.
\end{enumerate}
\end{assumption}
We take the support bound $b\ge1$, enlarging it if necessary. Let
\begin{equation}\label{eq:cover-geometry}
A_{\mM}:=\sup_{0<r\le1}r^d\mN(\mM,r)<\infty,
\qquad \mN(\mM,r)\le A_{\mM}r^{-d},\quad 0<r\le1,
\end{equation}
where $\mN(\mM,r)$ is the Euclidean covering number with centers on $\mM$.
For a compact smooth embedded manifold, this bound follows by covering finitely many bounded coordinate charts, whose parametrizations are Lipschitz, by $d$-dimensional grids. The constant $A_{\mM}$ retains the dependence on the geometry of $\mM$. The quasi-uniformity condition in Assumption~\ref{assump: assump_task_regularity} is retained, although the approximation argument below does not use it.

\begin{assumption}[Gaussian perturbations and diffusion schedule]\label{ass:noise}\label{assump: assump_loss_boundedness}
For $t\in[\delta,1]$, the known measurable schedule satisfies
$0\le\alpha(t)\le1$ and $0<\sigma(\delta)\le\sigma(t)\le1$.
The noisy query is
\[
\bx(t)=\alpha(t)\bx+\sigma(t)\eps,
\qquad \eps\sim\mathcal N(0,I_D),
\]
where $\eps$ is independent of the clean query and the prompt. The time variable is independent and uniform on $[\delta,1]$.
\end{assumption}
All population and empirical losses below use this same Gaussian perturbation law. A finite radius introduced in the proof restricts the region of exact network realization, rather than the distribution of the noise.

The denoising score-matching loss is bounded under the full Gaussian law. Indeed,
\begin{equation}\label{eq:residual-cancel}
T_\theta(\hat\rho)(\bx(t),t)+\frac{\eps}{\sigma(t)}
=\frac{\alpha(t)}{\sigma(t)^2}\bigl(G_\theta(\bs)-\bx\bigr),
\end{equation}
so that, uniformly over the hypothesis class and all noise realizations,
\begin{equation}\label{eq:loss-envelope}
0\le\left\|T_\theta(\hat\rho)(\bx(t),t)+\frac{\eps}{\sigma(t)}\right\|_2^2\le R^{\mL},
\qquad R^{\mL}=\frac{D(R+b)^2}{\sigma(\delta)^4}.
\end{equation}
This bound uses the output restriction already present in $\mathcal G$. Its dependence on $\delta$ must be retained when the time cutoff is subsequently chosen as a function of $n$.

Under Assumption~\ref{assump: assump_task_regularity} and \ref{assump: assump_loss_boundedness}, we obtain the following generalization error bound.

\begin{theorem}[Generalization error]\label{thm:generalization_error}
Under Assumption~\ref{assump: assump_task_regularity} and Assumption~\ref{ass:noise}, let $\mathfrak S$ be the training set defined above.
Let $n\ge2$ and choose $\delta\in(0,1)$ satisfying $\sigma(\delta)=n^{-1/(d+6)}$.
Consider the class $\mathcal T(L_T,m_T,d_{\mathrm{embed}},\ell,L_{\mathrm{FFN}},w_{\mathrm{FFN}},R,\kappa,D)$ with parameters
\[
L_T=5,\qquad m_T=n,\qquad d_{\mathrm{embed}}=O(D),\qquad\ell=n+1,
\]
\[
L_{\mathrm{FFN}}=O(1),\qquad w_{\mathrm{FFN}}=O(D),\qquad
\kappa=O\left(D^8 n^{\frac{2d+20}{d+6}}b^8R^4\right),
\]
Take the concrete choice $d_{\mathrm{embed}}=w_{\mathrm{FFN}}=2D+6$ and $L_{\mathrm{FFN}}=6$, with a sufficiently large absolute constant in the weight budget to contain the candidate from Proposition~\ref{prop:approximation_error}.
Let $T_{\theta^*;\mathfrak S}$ be a measurable empirical minimizer of $\mL_{\mathfrak S}$ over this class. Then
\begin{equation}\label{eqn:generalization_error_bound}
\begin{aligned}
&\E_{\mathfrak S}\left[\mL_n(T_{\theta^*;\mathfrak S})-\mL(T^*)\right]\\
&\quad\le C_1D^{5/2}n^{\frac12+\frac4{d+6}}\sqrt{\log(n^2DKJN_T)}
\left(\frac1{\sqrt K}+\frac1{\sqrt{KJ}}+\frac1{\sqrt{KJN_T}}\right)
+C_2D n^{-\frac2{d+6}}.
\end{aligned}
\end{equation}
Here $C_1$ depends only on $b,R$ and the fixed depths, and $C_2$ depends only on $b,d,A_{\mM},R$. Both are independent of $n,D,K,J,N_T$ and $\delta$. The cutoff choice is required to be attained by the diffusion schedule.
\end{theorem}

The theorem yields the following implications:
\begin{enumerate}

    \item Context efficiency is governed by intrinsic dimension.
    The finite-context error decays as $n^{-2/(d+6)}$. As the intrinsic dimension $d$ increases, the exponent becomes smaller and the convergence rate deteriorates, reflecting a curse of intrinsic dimensionality.
    The ambient dimension $D$ enters the bound only through polynomial and logarithmic factors. Thus, high-dimensional observations need not be prohibitive when the task distributions concentrate on a low-dimensional manifold.

    \item The statistical error separates across sampling levels.
    The terms $K^{-1/2}$, $(KJ)^{-1/2}$, and $(KJN_T)^{-1/2}$ quantify the effects of finite sampling at the task, episode, and diffusion-sample levels, respectively. Increasing $J$ reduces the episode-level and diffusion-level errors, while increasing $N_T$ reduces only the diffusion-level error.

    \item Task diversity is essential for generalization. The persistent $K^{-1/2}$ term is unaffected by increases in either $J$ or $N_T$. Thus, collecting additional episodes or diffusion samples from existing tasks cannot compensate for limited meta-training task coverage. Generalization to unseen distributions instead requires a sufficiently large and diverse collection of meta-training tasks. From a data-allocation perspective, the theorem therefore suggests prioritizing broader task coverage once the within-task estimation error has been sufficiently reduced.
\end{enumerate}

\subsection{Proof Outline}\label{sec:proof_strategy}
This subsection outlines the proof of Theorem~\ref{thm:generalization_error}. Let $T_{\hat\theta}$ denote the transformer realization of the empirical kernel estimator established in Lemma~\ref{lemma:transformer_kernel_regression}. By adding and subtracting $\mathcal L_{\mathfrak S}(T_{\theta^*;\mathfrak S})$, $\mathcal L_{\mathfrak S}(T_{\hat\theta})$, and $\mathcal L_n(T_{\hat\theta})$, we obtain
\[
\begin{aligned}
\mathcal L_n(T_{\theta^*;\mathfrak{S}})-\mathcal L(T^*)
&=
\bigl[\mathcal L_n(T_{\theta^*;\mathfrak{S}})-{\mathcal L}_{\mathfrak{S}}(T_{\theta^*;\mathfrak{S}})\bigr]
+\bigl[{\mathcal L}_{\mathfrak{S}}(T_{\theta^*;\mathfrak{S}})-{\mathcal L}_{\mathfrak{S}}(T_{\hat\theta})\bigr]\\
&\quad+\bigl[{\mathcal L}_{\mathfrak{S}}(T_{\hat\theta})-\mathcal L_n(T_{\hat\theta})\bigr]
+\bigl[\mathcal L_n(T_{\hat\theta})-\mathcal L(T^*)\bigr].
\end{aligned}
\]
Since $\mathcal L_{\mathfrak S}(T_{\theta^*;\mathfrak S}) - \mathcal L_{\mathfrak S}(T_{\hat\theta})\le 0$, we then have
\begin{equation}\label{eqn:gen_error_decomp}
\E_{\mathfrak{S}}\left[\mathcal L_n(T_{\theta^*;\mathfrak{S}})-\mathcal L(T^*)\right]
\;\le\;
\underbrace{\E_{\mathfrak{S}}\left[\mathcal L_n(T_{\theta^*;\mathfrak{S}})-{\mathcal L}_{\mathfrak{S}}(T_{\theta^*;\mathfrak{S}})\right]}_{\text{statistical error I}}
+
\underbrace{\E_{\mathfrak{S}}\left[{\mathcal L}_{\mathfrak S}(T_{\hat\theta})-\mathcal L_n(T_{\hat\theta})\right]}_{\text{statistical error II}}
+
\underbrace{\bigl(\mathcal L_n(T_{\hat\theta})-\mathcal L(T^*)\bigr)}_{\text{approximation error III}}.
\end{equation}

In what follows, we control the three terms in \eqref{eqn:gen_error_decomp} separately. We first bound the approximation error by showing that the target score operator $T^*$ admits a kernel regression representation and can therefore be approximated by a transformer-based neural operator. We then bound Statistical Error~$\mathrm{I}$ using covering numbers and Rademacher complexity, followed by Statistical Error~$\mathrm{II}$ via Hoeffding's inequality. Combining these estimates establishes the desired generalization bound.

\paragraph{Approximation error term $\mathrm{III}$.}By the exact transformer realization in
Section~\ref{sec:kernel_transformer}, this term reduces to the approximation error of the empirical Nadaraya--Watson estimator. Concentration of the kernel estimator on the underlying $d$-dimensional manifold gives the
following result, whose proof is deferred to Appendix~\ref{sec:approx_error_III}.

\begin{proposition}[Approximation error]\label{prop:approximation_error}
Under Assumption~\ref{assump: assump_task_regularity}
and Assumption~\ref{ass:noise}, there exists
$T_{\hat\theta}\in\mathcal T(L_T,m_T,d_{\mathrm{embed}},\ell,L_{\mathrm{FFN}},w_{\mathrm{FFN}},R,\kappa,D)$
with parameters
\begin{equation}\label{eq:prop-parameters}
\begin{gathered}
L_T=5,\qquad m_T=n,\qquad d_{\mathrm{embed}}=O(D),\qquad\ell=n+1,\\
L_{\mathrm{FFN}}=O(1),\qquad w_{\mathrm{FFN}}=O(D),\qquad
\kappa=O\left(\frac{D^8n^2b^8R^4}{\sigma(\delta)^8}\right),
\end{gathered}
\end{equation}
such that
\begin{equation}\label{eqn:final_approx_error}
\mL_n(T_{\hat\theta})-\mL(T^*)
\le\frac{C_{b,d}A_{\mM}D}{n\sigma(\delta)^{d+4}}
+4Db^2\sigma(\delta)^2.
\end{equation}
The constants hidden in the network budgets are absolute, with $R\ge b\ge1$; the geometric dependence is explicit in $A_{\mM}$.
\end{proposition}

\paragraph{Statistical Error I.}
Term~$\mathrm{I}$ measures the deviation between the finite-context population risk and the empirical risk of the learned estimator. Since $T_{\theta^*;\mathfrak S}$ depends on the training set, controlling this term requires a uniform generalization analysis over the transformer hypothesis class. Using covering number and Rademacher complexity estimates, we obtain the following bound.

\begin{lemma}[Statistic error term I]\label{lemma:stat_error_I}
Use the width convention of Lemma~\ref{lem:cover}. Let $T_{\theta^*;\mathfrak S}\in\mathcal T$ be a measurable empirical minimizer of $\mL_{\mathfrak S}(T_\theta)$. Then
\begin{equation}\label{eq:gen_bound}
\begin{aligned}
&\E_{\mathfrak S}\left[\mL_n(T_{\theta^*;\mathfrak S})-\mL_{\mathfrak S}(T_{\theta^*;\mathfrak S})\right]\\
&\quad\le4R^{\mL}\left(\frac1{\sqrt K}+\frac1{\sqrt{KJ}}+\frac1{\sqrt{KJN_T}}\right)
\left(6\sqrt{DP_T\log\frac{A_{0,\delta}\sqrt{KJN_T}}{R^{\mL}}}+1\right)
+R^{\mL}q_\delta,
\end{aligned}
\end{equation}
where $R^{\mL}$ is given in~\eqref{eq:loss-envelope}, $q_\delta$ in~\eqref{eq:Cdelta}, and $P_T,A_{0,\delta}$ in~\eqref{eqn:PT},~\eqref{eq:A-delta}.

Furthermore, with parameters
\[
L_T=5,\qquad m_T=n,\qquad d_{\mathrm{embed}}=O(D),\qquad\ell=n+1,
\]
\[
L_{\mathrm{FFN}}=O(1),\qquad w_{\mathrm{FFN}}=O(D),\qquad
\kappa=O\left(\frac{D^8n^2b^8R^4}{\sigma(\delta)^8}\right),
\]
we have
\begin{equation}\label{eq:stat_error1_final_bound}
\begin{aligned}
&\E_{\mathfrak S}\left[\mL_n(T_{\theta^*;\mathfrak S})-\mL_{\mathfrak S}(T_{\theta^*;\mathfrak S})\right]\\
&\quad\le O\!\left(
\frac{D(R+b)^2}{\sigma(\delta)^4}
\sqrt{nD^3\log\frac{nDKJN_T}{\sigma(\delta)}}
\left[\frac1{\sqrt K}+\frac1{\sqrt{KJ}}+\frac1{\sqrt{KJN_T}}\right]\right)
+D(R+b)^2\sigma(\delta)^2.
\end{aligned}
\end{equation}
Here $n\ge2$, and the notation $O(\cdot)$ hides constants depending only on $b,R$ and the fixed depths. In particular, the dependence of $R^{\mL}$ on $\delta$ is displayed explicitly.
\end{lemma}

The proof of Lemma~\ref{lemma:stat_error_I} is provided in Appendix~\ref{sec:stat_error_I}.

\paragraph{Statistical error II.}
Term~$\mathrm{II}$ measures the deviation between the empirical and finite-context population risks evaluated at $T_{\hat\theta}$. Since $T_{\hat\theta}$ is fixed independently of the training set, the deviations arising from the task, episode, and diffusion-sample levels can be controlled separately using concentration inequalities. This yields the following bound, whose proof is deferred to Appendix~\ref{sec:stat_error_II}.

\begin{lemma}[Statistic error term II]\label{lemma:stat_error_II}. The Statistic error term II can be bounded as
\begin{equation}\label{eqn:statistic_error_II}
    \E_{\mathfrak{S}}\left({\mathcal L}_{\mathfrak S}(T_{\hat\theta})-\mathcal L_n(T_{\hat\theta})\right) \le O\left(
    \sqrt{\log(1/\sigma(\delta))}\left[
    \frac1{\sqrt K}
    +\frac1{\sqrt{KJ}}
    +\frac1{\sqrt{KJN_T}}
    \right] +
    \sigma^2(\delta)
    \right).
\end{equation}
\end{lemma}

Combining Lemmas~\ref{lemma:stat_error_I} and~\ref{lemma:stat_error_II} with Proposition~\ref{prop:approximation_error} in the decomposition~\eqref{eqn:gen_error_decomp} yields an upper bound consisting of the three hierarchical statistical error terms and the finite-context approximation error term. Choosing $\sigma(\delta)=n^{-1/(d+6)}$ balances the dominant approximation error terms and yields the rate $n^{-2/(d+6)}$. Substituting the transformer parameter scaling and the loss bound $R^{\mL}=D(R+b)^2\sigma(\delta)^{-4}$ establishes the bound in Theorem~\ref{thm:generalization_error}. The detailed derivation is postponed to Appendix~\ref{sec:proof_generalization_error}.

\section{Numerical results}\label{sec:numerical}
In this section, we evaluate our method on four benchmark datasets spanning scientific fields and visual data: solution fields of the 2D Navier--Stokes equation, ERA5 climate data, QuickDraw sketches, and Food-101 images. For Navier--Stokes, we examine both in-distribution and out-of-distribution task generalization. For the remaining datasets, we evaluate in-distribution generation on observed task categories.

\paragraph{Data preparation and mini-batch construction}
Algorithm~\ref{alg:icldm_train} is presented using an explicit training set $\mathfrak{S}=\{\mathfrak{s}_k\}_{k=1}^{K}$, where for each task $\rho_k$ one pre-generates $J$ independent context--query episodes and, for each episode, $N_T$ independent draws of diffusion time and noise, yielding tuples $\bs_{k,j,l}=\{\hat{\rho}_{k,j},\bx_{k,j}(t_{k,j,l}),t_{k,j,l}\}$ that are stored in full before training begins. In practice, however, pre-generating and storing the full dataset $\mathfrak{S}$ is prohibitively memory-intensive: since each episode $\hat{\rho}_{k,j}$ requires $n$ fresh samples from $\rho_k$, the total number of samples grows as $K\times J\times n$, which quickly becomes infeasible for large $J$. We therefore construct mini-batches on the fly. Specifically, for the NS experiments, we pre-sample $K$ task distributions $\{\rho_k\}_{k=1}^{K}\stackrel{\text{i.i.d.}}{\sim}\mu$ and, for each task $\rho_k$, draw a fixed bank of $B$ i.i.d.\ samples $\{\bx_k^{(i)}\}_{i=1}^{B}\stackrel{\text{i.i.d.}}{\sim}\rho_k$; for fixed labeled datasets (ERA5, QuickDraw, and Food-101), we treat each class as a task, select $K$ classes, and randomly sample $B$ datapoints per class. At each training step, we sample a task index $k\sim\mathrm{U}\{1,\ldots,K\}$, randomly draw $n+1$ samples from the fixed bank $\{\bx_k^{(i)}\}_{i=1}^{B}$, and split them into a context set $\hat{\rho}_{k}:=\{\bx_k^{(1)},\ldots,\bx_k^{(n)}\}$ and a clean query $\bx_k^{\mathrm{qry}}:=\bx_k^{(n+1)}$; we then sample $t\sim\mathrm{Unif}[\delta,1]$ and $\varepsilon\sim\mathcal{N}(0,I)$ independently on the fly to form the noised query $\bx_{k}(t)=\alpha(t)\bx_k^{\mathrm{qry}}+\sigma(t)\varepsilon$, yielding the training episode $\bs=\{\hat{\rho}_{k},\bx_{k}(t),t\}$ used in~\eqref{eqn:DSM_empirical}.

\paragraph{Metrics.} To quantify the performance of the in-context diffusion model and the baseline methods, we compare the generated and reference empirical distributions using the Wasserstein-2 distance, the Jensen--Shannon divergence (JSD) between scalar-valued empirical probability density functions, energy-spectrum-based metrics, and Fr\'echet Inception Distance (FID). Further details are provided in Appendix~\ref{sec:metrics}.

\paragraph{Baseline models.}
We compare three methods: our proposed \textbf{in-context learning diffusion model (ICL DM)}, a \textbf{conditional diffusion model (Cond DM)}, and the closely related baseline \textbf{SNO} \cite{liao2024sno}. All three models share the same time-conditioned, multi-resolution U-Net denoising backbone and training protocol, and differ only in how the prompt set is represented and injected. In \textbf{ICL DM}, the prompt set is encoded into multi-scale features that interact with the noisy query through cross-attention. In \textbf{Cond DM}, the prompt set is compressed into a single learned embedding that modulates the U-Net through FiLM-style conditioning. In \textbf{SNO}, this learned prompt representation is replaced by a projected kernel mean embedding (KME), while the rest of the denoising architecture remains the same as in \textbf{Cond DM}. Additional implementation details are given in Appendix~\ref{sec:network_details}.

\subsection{2D Navier--Stokes equation}
We consider the vorticity form of the 2D Navier--Stokes equation with Kolmogorov forcing:
\begin{equation}\label{eq:ns_vorticity}
\left\{
\begin{aligned}
\partial_t \omega(t,\bx) + \bu(t,\bx)\cdot \nabla \omega(t,\bx) - \nu \nabla^2 \omega(t,\bx) &= f(\bx),
&& \bx\in[0,2\pi]^2,\ t\in[0,T],\\
\nabla\cdot \bu(t,\bx) &= 0,
&& \bx\in[0,2\pi]^2,\ t\in[0,T],\\
\omega(0,\bx) &= \omega_0(\bx),
&& \bx\in[0,2\pi]^2.
\end{aligned}
\right.
\end{equation}
Here $\bx=(x_1,x_2)$, $\omega(t,\bx)$ denotes the vorticity, and the velocity field is given by
\[
\bu(t,\bx) = \bigl(\partial_{x_2}\psi(t,\bx),\ -\partial_{x_1}\psi(t,\bx)\bigr),
\qquad
\omega(t,\bx) = -\nabla^2 \psi(t,\bx).
\]
The Kolmogorov forcing is defined as $f(\bx)=\sin(k_0 x_1)$. In this example, we set the wavenumber $k_0=4$.
To generate initial conditions, we first draw samples from a log-normal distribution on a uniform $256\times 256$ mesh over the spatial domain $\bx\in[0,2\pi]^2$. We parameterize Navier-Stokes tasks by $\xi\in\Xi$, here the $\xi:=\nu$ is the viscosity. Let $\pi_0$ be a probability measure over initial conditions $\omega_0$, and for each $\xi\in\Xi$, we let $S^{T,\xi}$ denote the solution operator mapping $\omega_0$ to the vorticity field at time $T$, i.e., $S^{T,\xi}:\omega_0(\bx) \mapsto \omega(T,\bx;\omega_0,\xi)$; this induces the task distribution
\[
\rho_{\xi} \;:=\; (S^{T,\xi})_{\#}\pi_0 \in \mathcal{P}(\mathcal{X}).
\]
Let $\nu$ be a probability distribution on $\Xi$, and define the mapping $F:\Xi\to\mathcal{P}(\mathcal{X})$ by $F(\xi)=\rho_{\xi}$. The meta-distribution over task distributions is then the pushforward
\[
\mu \;:=\; F_{\#}\nu \in \mathcal{P}(\mathcal{P}(\mathcal{X})).
\]
In the following experiments, we set $T=1$ and use the finite volume method (FVM) implemented in JAX-CFD \cite{Kochkov2021-ML-CFD} for generating data. To evaluate both in-distribution and out-of-distribution generalization, we construct 100 tasks by uniformly sampling $\nu$ from the union of two intervals,
\begin{equation}\label{eqn:nu_sampling}
\nu \sim \mathrm{Unif}\left([1.5e-3,\,1.2e-2] \cup [1.5e-2,\,5e-2]\right).
\end{equation}
For the in-distribution setting, we randomly sample five values of $\nu$ from the first interval and five values from the second interval, and evaluate the model on these 10 values. The results of the different methods are reported in Table~\ref{tab:NS_results}. We also present three representative examples with $\nu=1.5e-3$, $\nu=2.2e-3$, and $\nu=3e-2$, as shown in Figures~\ref{fig:NS_ind_1} and~\ref{fig:NS_ind_2}. For the out-of-distribution setting, we evaluate the model at ${\nu}_1=1.2e-3$, ${\nu}_2=1.3e-2$, and ${\nu}_3=6.4e-2$, all of which lie outside the training intervals. The corresponding results are presented in Figures~\ref{fig:NS_ood_1} and~\ref{fig:NS_ood_2}. To assess sample quality, we compare the generated and reference distributions using the weighted mean energy log ratio (MELRw) and the JSD between empirical scalar-valued PDFs. The results are reported in Table~\ref{tab:NS_results}. As shown in Table~\ref{tab:NS_results}, our method outperforms the baselines in both in-distribution and out-of-distribution settings, except for MELRw in the OOD(high) setting, where Cond DM achieves a lower value.

\begin{figure}
    \centering
    \begin{subfigure}[t]{0.9\linewidth}
        \centering
        \includegraphics[width=\linewidth]{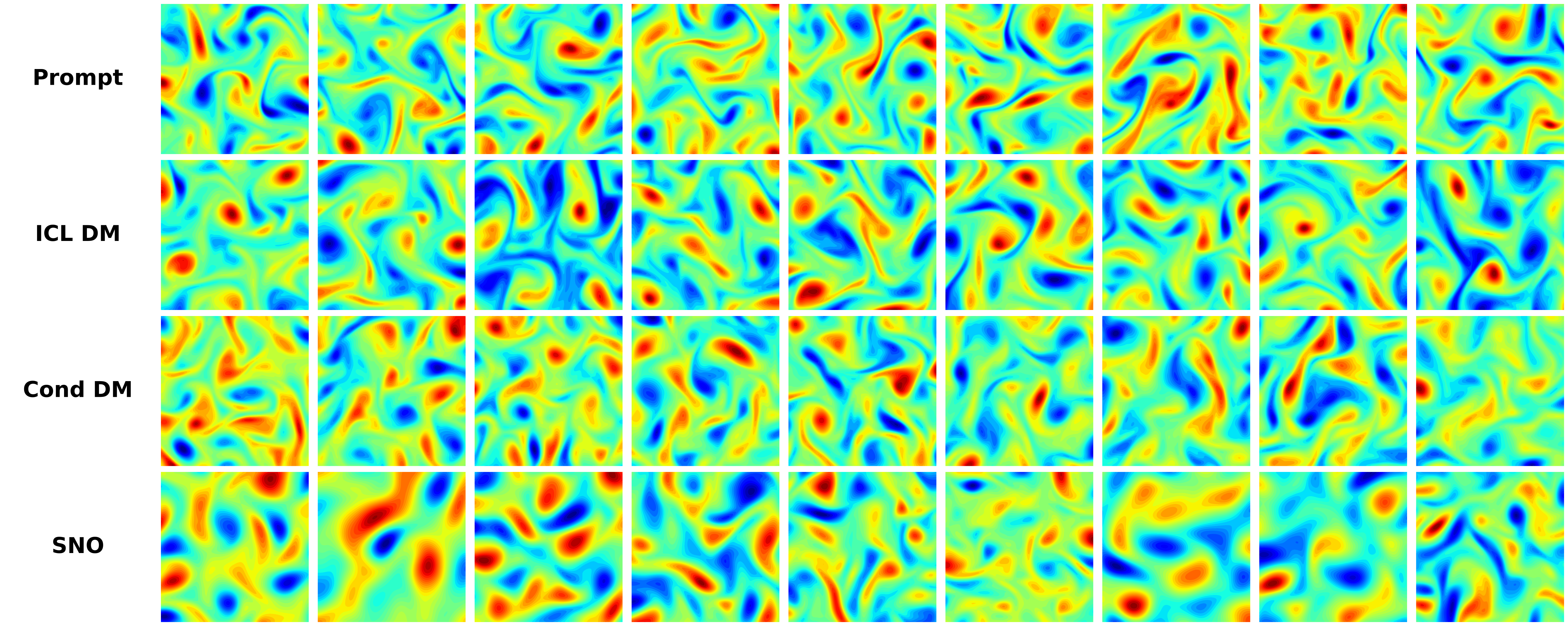}
        \caption{$\nu = 1.5\times 10^{-3}$}
    \end{subfigure}
    \hfill
    \begin{subfigure}[t]{0.9\linewidth}
        \centering
        \includegraphics[width=\linewidth]{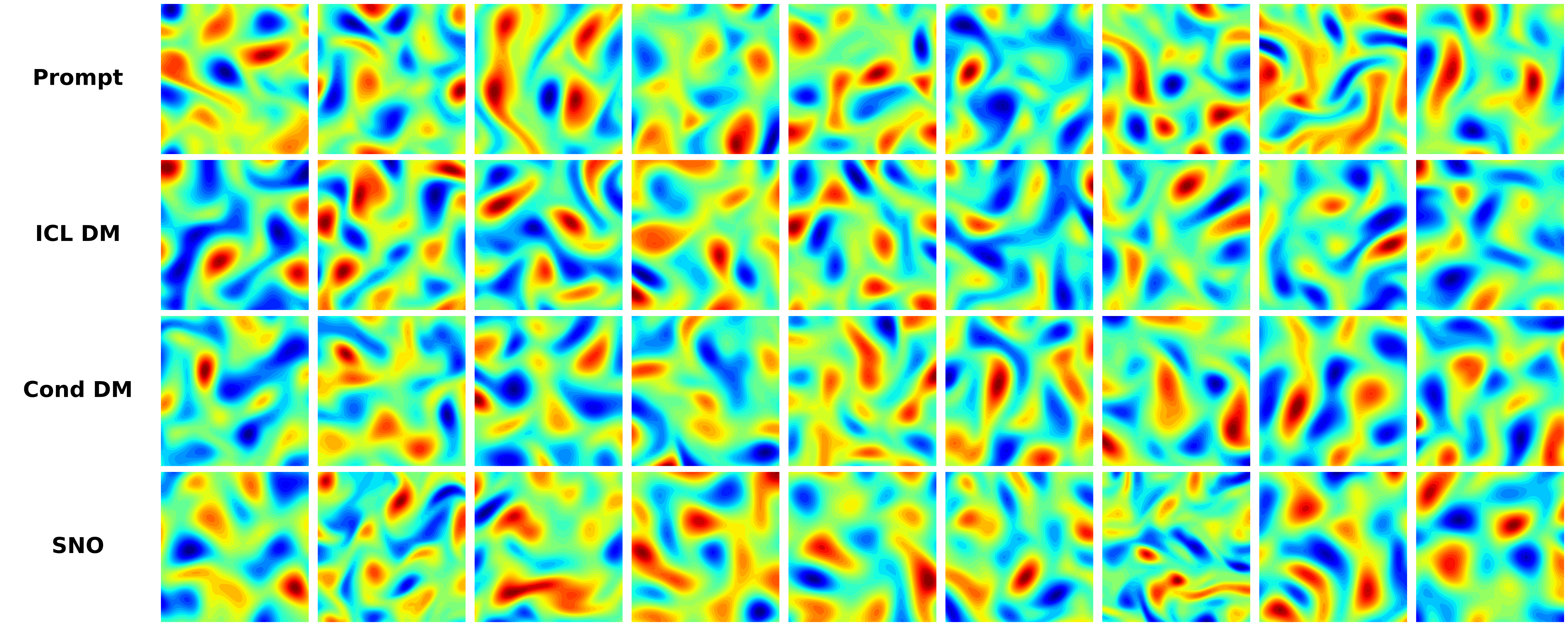}
        \caption{$\nu = 2.2\times 10^{-3}$}
    \end{subfigure}
    \hfill
    \begin{subfigure}[t]{0.9\linewidth}
        \centering
        \includegraphics[width=\linewidth]{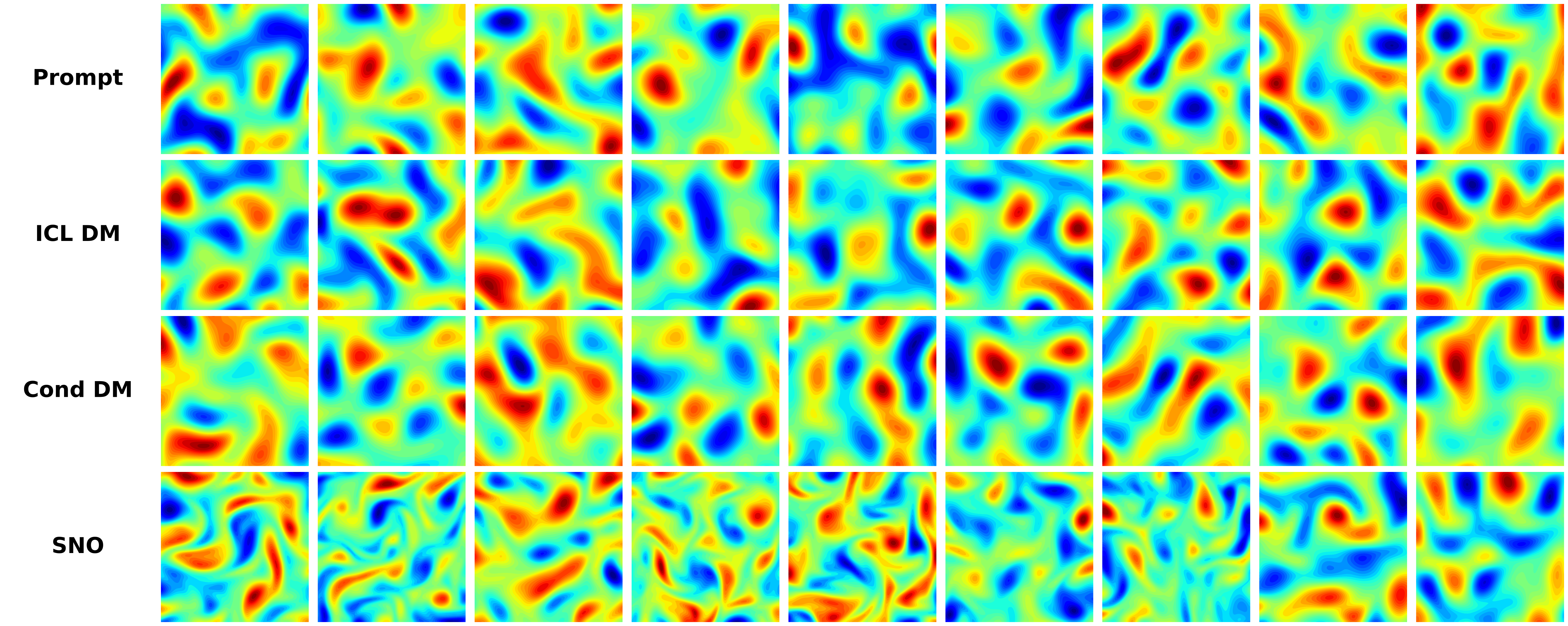}
        \caption{$\nu = 3\times 10^{-2}$}
    \end{subfigure}
    \caption{Representative in-distribution inferencesamples for the Navier--Stokes equation. Panels (a)--(c) correspond to viscosities $\nu = 1.5 \times 10^{-3}$, $2.2 \times 10^{-3}$, and $3 \times 10^{-2}$, respectively. Within each panel, the rows show prompt samples and samples generated by ICL DM, Cond DM, and SNO, from top to bottom, while the columns show representative realizations. All methods use $n = 100$ prompt samples.}
    \label{fig:NS_ind_1}
\end{figure}

\begin{figure}
    \centering
    \begin{subfigure}[t]{0.3\linewidth}
        \centering
        \includegraphics[width=\linewidth]{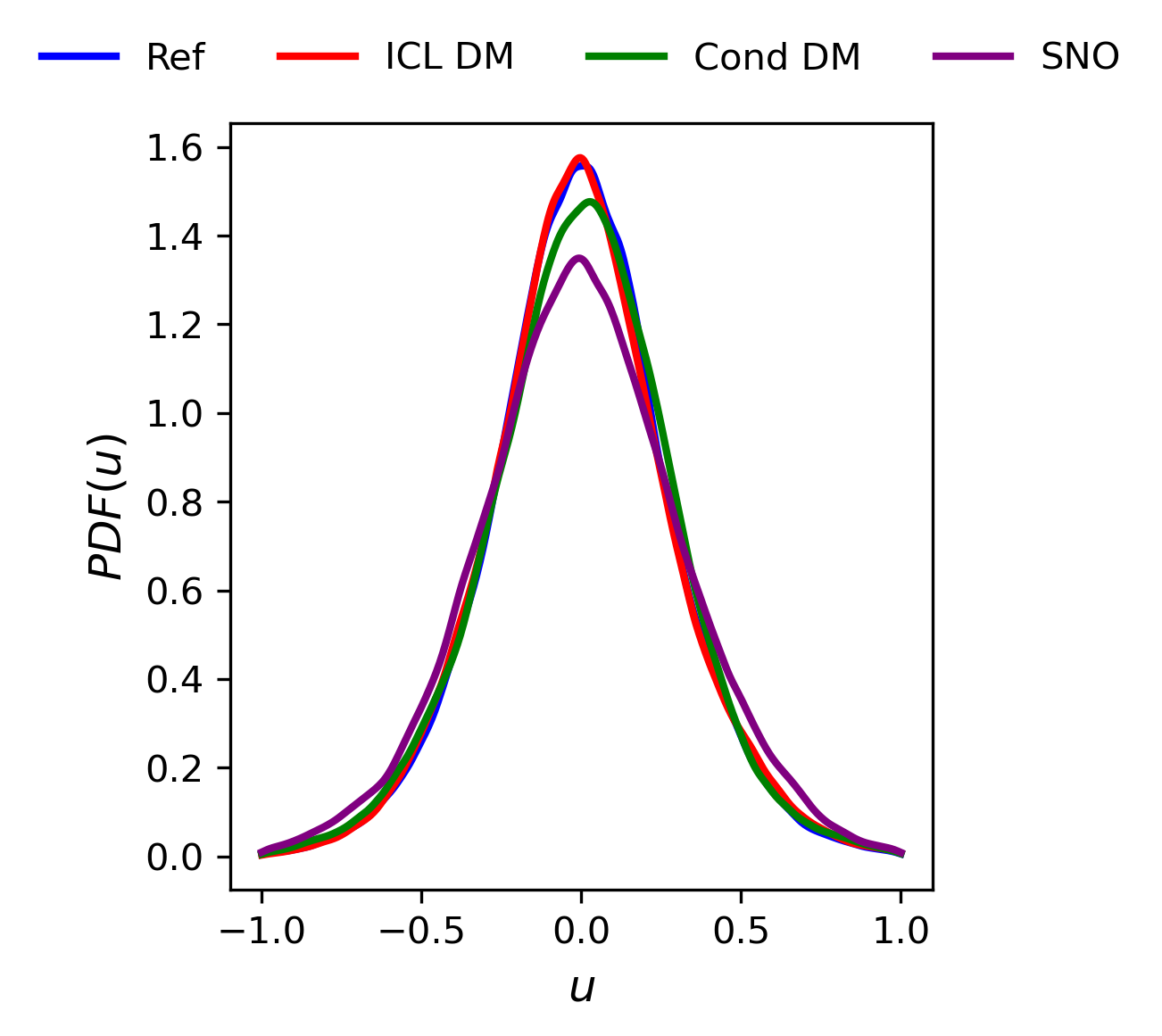}
        \caption{$\nu = 1.5\times 10^{-3}$}
    \end{subfigure}
    \hfill
    \begin{subfigure}[t]{0.3\linewidth}
        \centering
        \includegraphics[width=\linewidth]{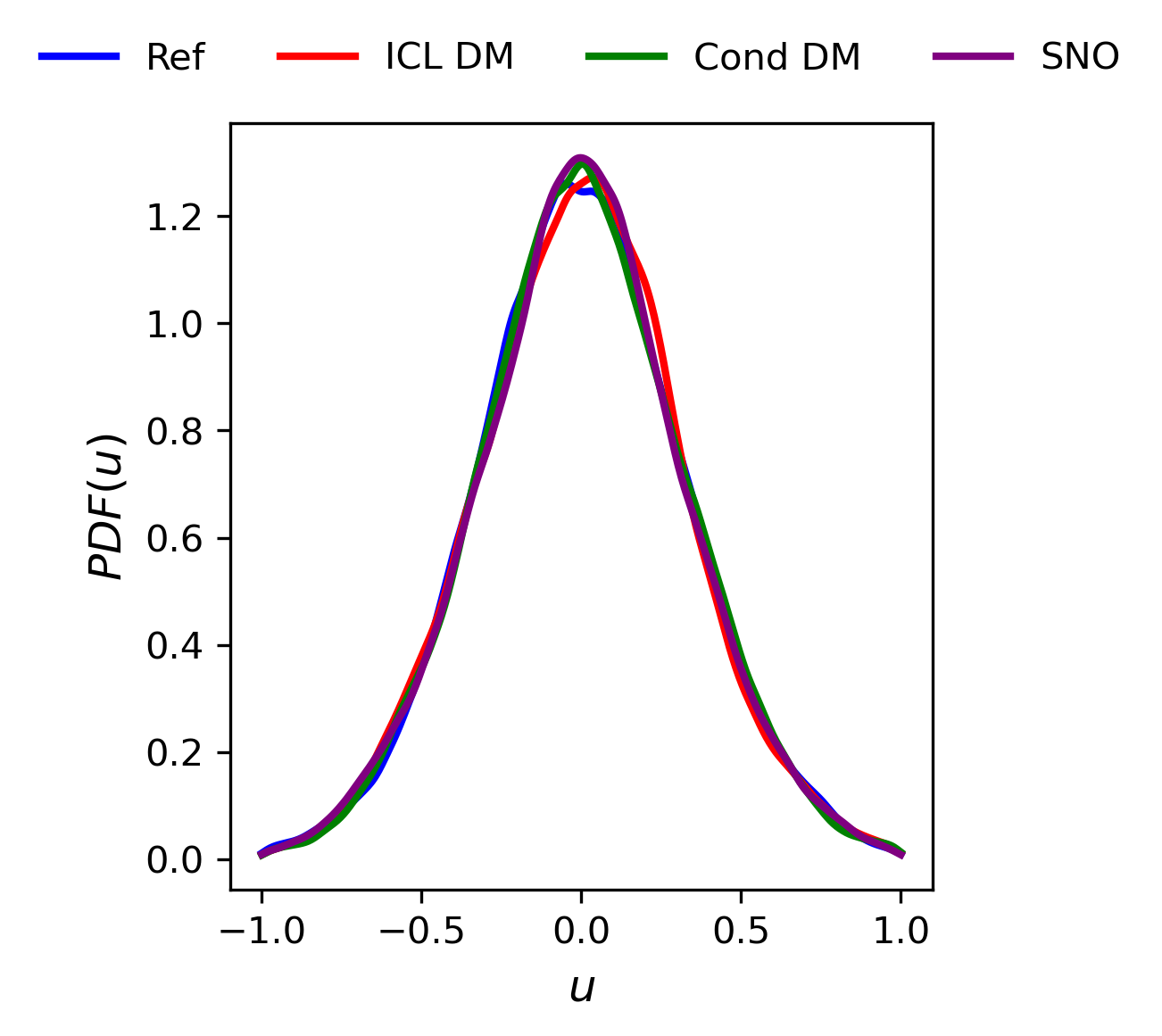}
        \caption{$\nu = 2.2\times 10^{-3}$}
    \end{subfigure}
    \hfill
    \begin{subfigure}[t]{0.3\linewidth}
        \centering
        \includegraphics[width=\linewidth]{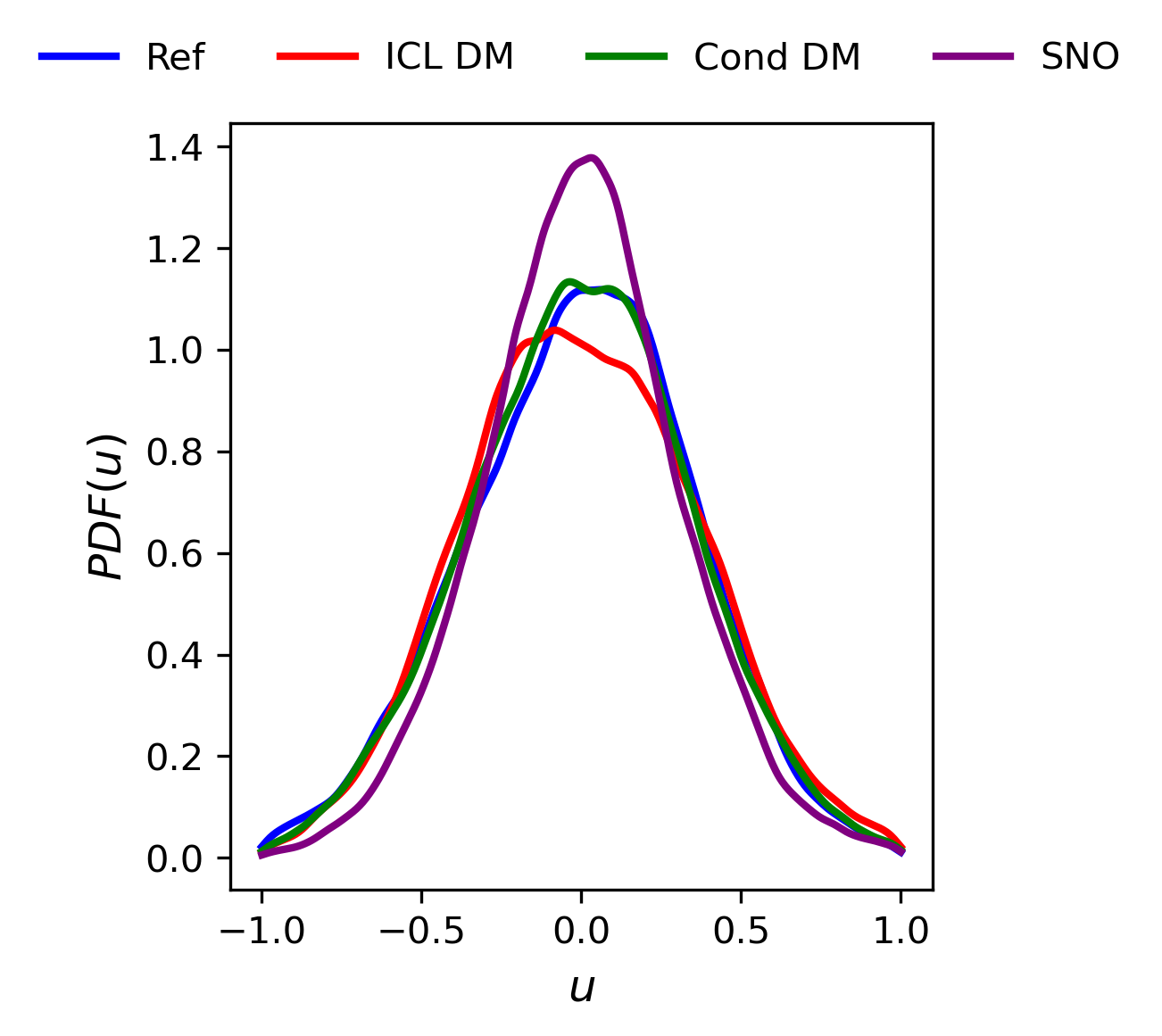}
        \caption{$\nu = 3\times 10^{-2}$}
    \end{subfigure}
    \caption{Scalar-valued empirical PDF of the Navier--Stokes equation in the in-distribution setting. Panels (a)--(c) correspond to $\nu = 1.5 \times 10^{-3}$, $2.2 \times 10^{-3}$, and $3 \times 10^{-2}$, respectively. The blue curve denotes the reference distribution, while the red, green, and purple curves represent ICL DM, Cond DM, and SNO, respectively. Greater overlap with the reference curve indicates better agreement with the reference distribution.}
    \label{fig:NS_ind_2}
\end{figure}

\begin{figure}
    \centering
    \begin{subfigure}[t]{0.95\linewidth}
        \centering
        \includegraphics[width=\linewidth]{new_figs/OOD_NS_type_1_K_100_B_1000_n_100_m_100_np_1_epoch_150_idx_4.pdf}
        \caption{$\nu = 1.2\times 10^{-3}$}
    \end{subfigure}
    \hfill
    \begin{subfigure}[t]{0.95\linewidth}
        \centering
        \includegraphics[width=\linewidth]{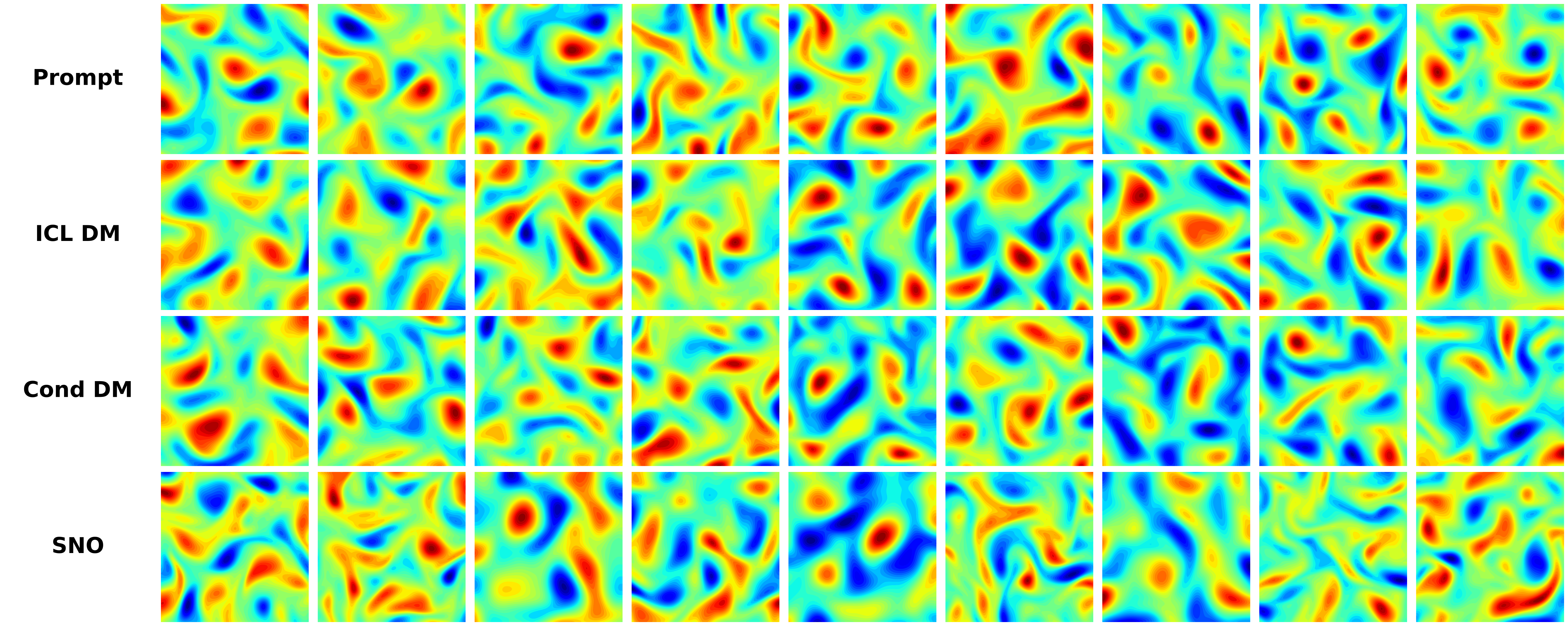}
        \caption{$\nu = 1.3\times 10^{-2}$}
    \end{subfigure}
    \hfill
    \begin{subfigure}[t]{0.95\linewidth}
        \centering
        \includegraphics[width=\linewidth]{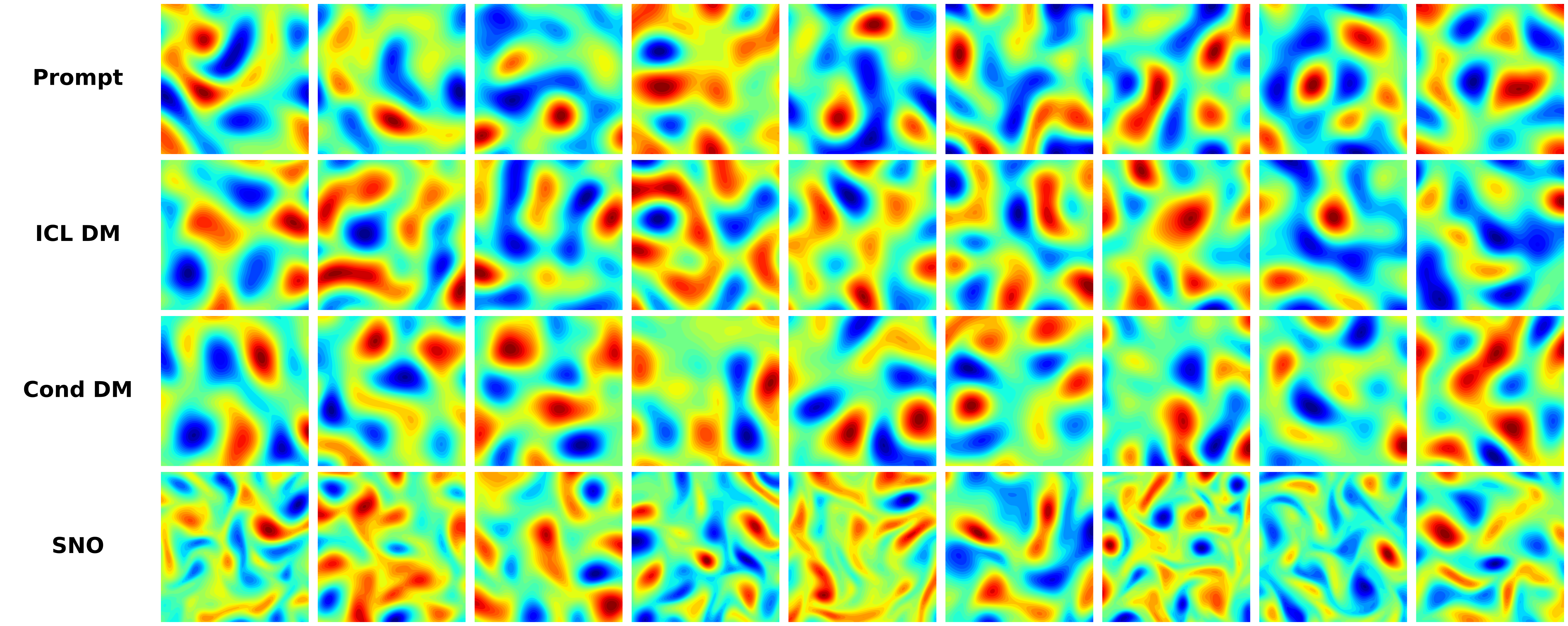}
        \caption{$\nu = 6.4\times 10^{-2}$}
    \end{subfigure}
    \caption{Out-of-distribution inference samples for Navier--Stokes equation. Panels (a)--(c) correspond to $\nu = 1.2 \times 10^{-3}$, $1.3 \times 10^{-2}$, and $6.4 \times 10^{-2}$, respectively. These viscosities lie below, between, and above the two training intervals defined in \eqref{eqn:nu_sampling}, respectively. Within each panel, the rows show prompt samples and samples generated by ICL DM, Cond DM, and SNO, from top to bottom.}
    \label{fig:NS_ood_1}
\end{figure}

\begin{figure}
    \centering
    \begin{subfigure}[t]{0.33\linewidth}
        \centering
        \includegraphics[width=\linewidth]{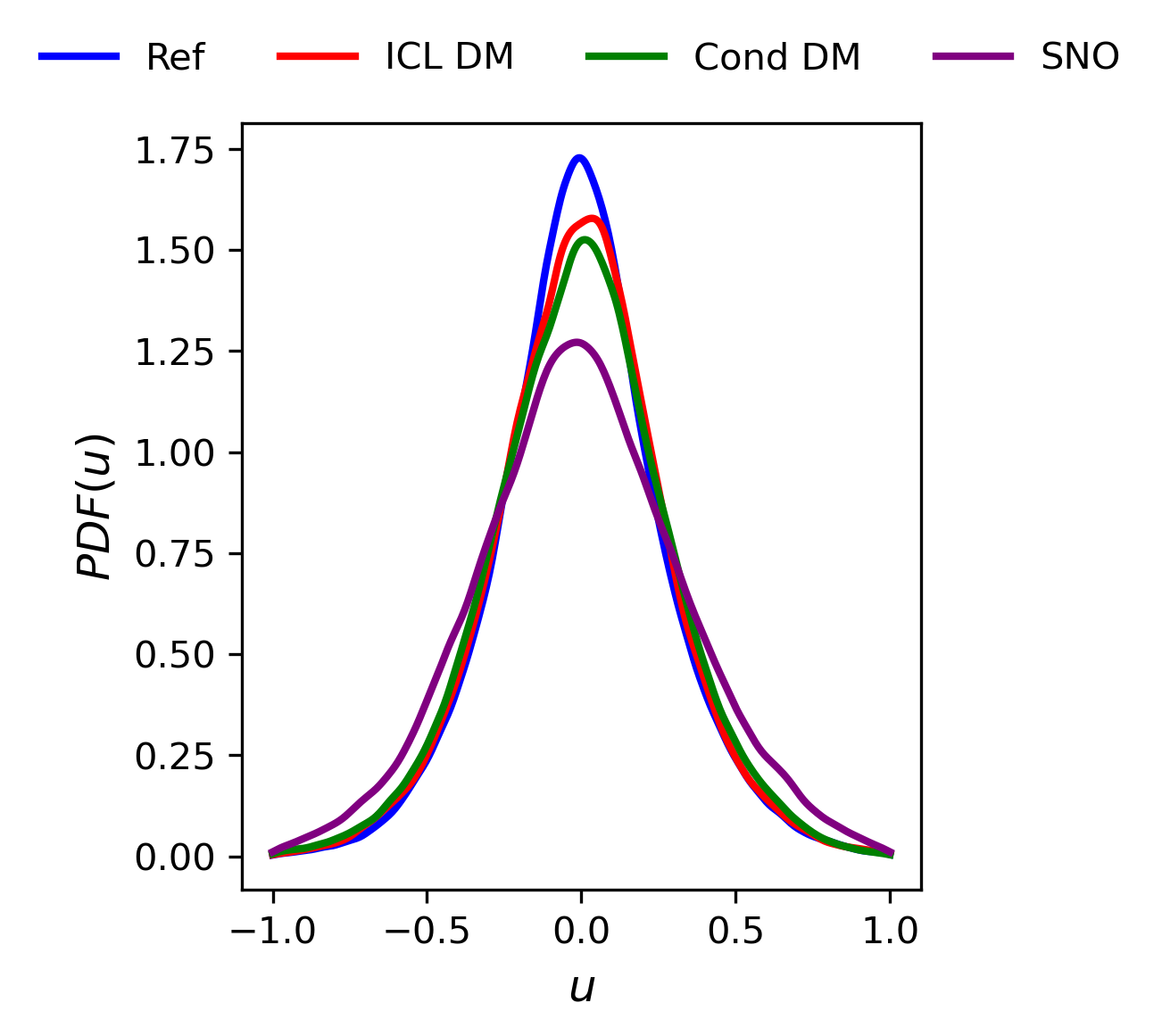}
        \caption{$\nu = 1.2\times 10^{-3}$}
    \end{subfigure}
    \hfill
    \begin{subfigure}[t]{0.33\linewidth}
        \centering
        \includegraphics[width=\linewidth]{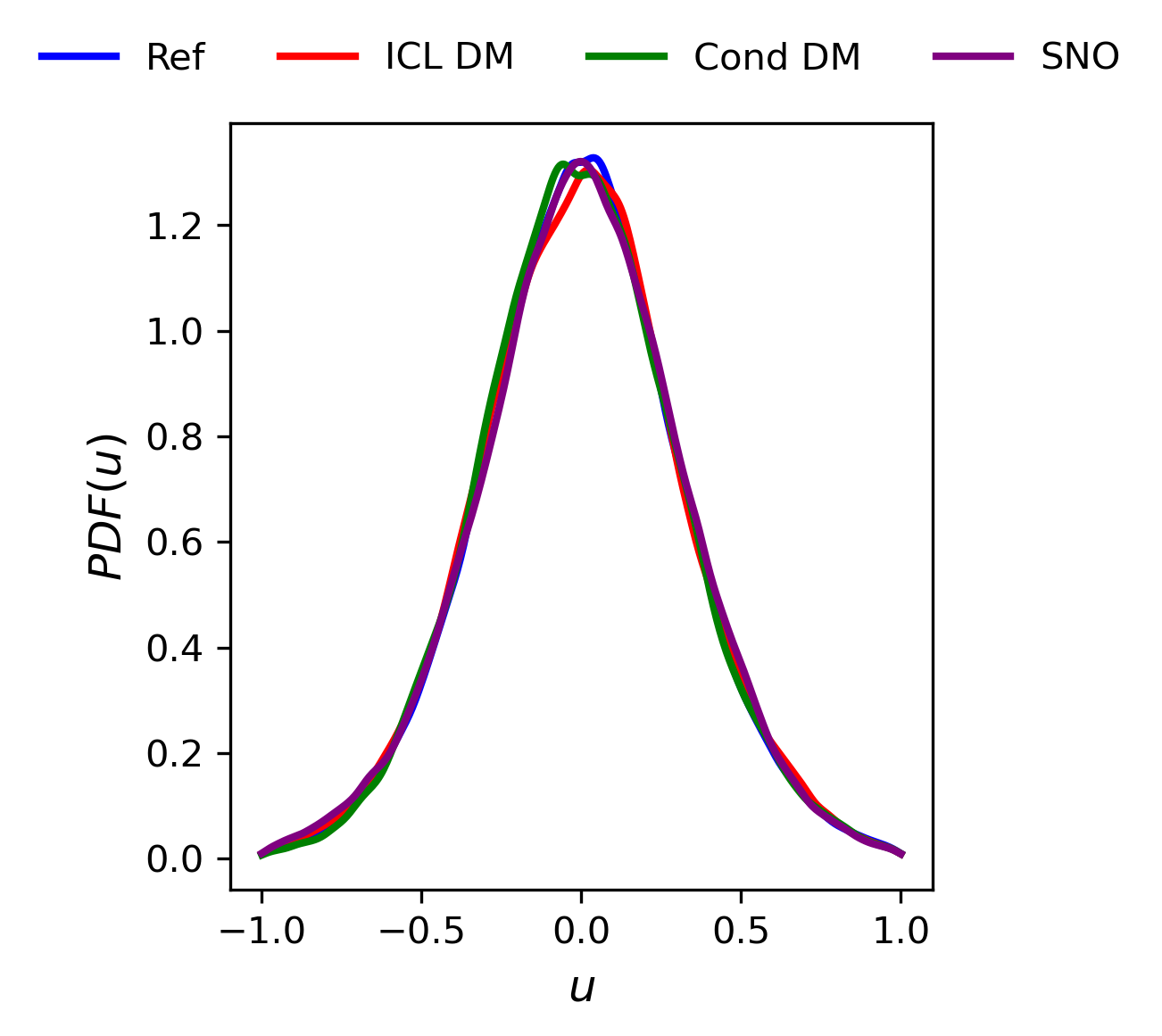}
        \caption{$\nu = 1.3\times 10^{-2}$}
    \end{subfigure}
    \hfill
    \begin{subfigure}[t]{0.33\linewidth}
        \centering
        \includegraphics[width=\linewidth]{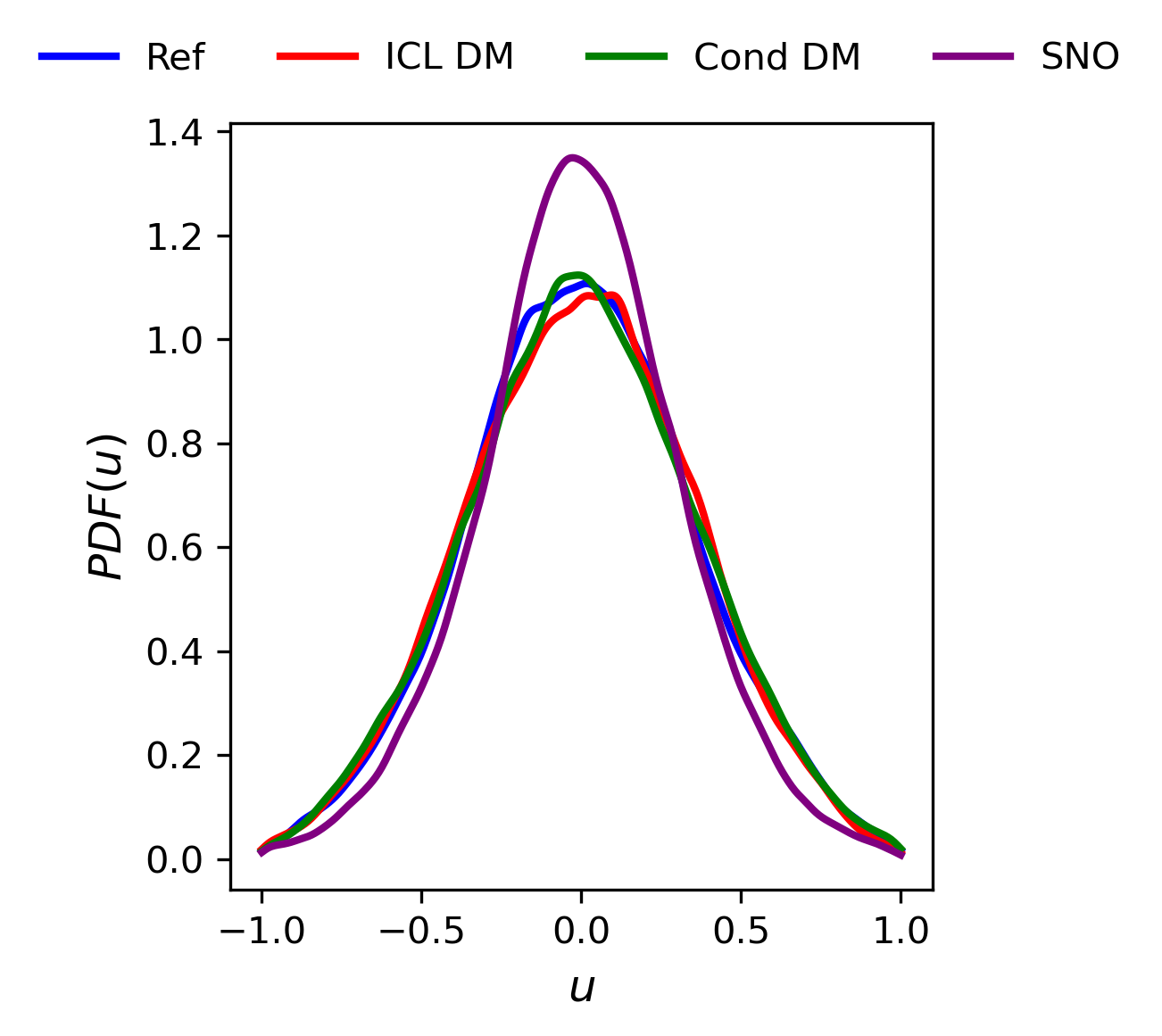}
        \caption{$\nu = 6.4\times 10^{-2}$}
    \end{subfigure}

    \caption{Scalar-valued empirical PDFs for the Navier--Stokes equation in the out-of-distribution setting. Panels (a)--(c) correspond to $\nu = 1.2 \times 10^{-3}$, $1.3 \times 10^{-2}$, and $6.4 \times 10^{-2}$, respectively. The blue curve denotes the reference distribution, while the red, green, and purple curves represent ICL DM, Cond DM, and SNO, respectively. Greater overlap with the reference curve indicates better recovery of the target marginal distribution. All methods use $n = 100$ prompt samples.}
    \label{fig:NS_ood_2}
\end{figure}

\begin{table}[H]
    \centering
    \begin{tabular}{|c|cc|cc|cc|cc|}
        \hline
        & \multicolumn{2}{c|}{IND} 
        & \multicolumn{2}{c|}{OOD(low)} 
        & \multicolumn{2}{c|}{OOD(mid)} 
        & \multicolumn{2}{c|}{OOD(high)} \\
        & MELRw $\downarrow$ & JSD $\downarrow$
        & MELRw $\downarrow$ & JSD $\downarrow$
        & MELRw $\downarrow$ & JSD $\downarrow$
        & MELRw $\downarrow$ & JSD $\downarrow$ \\
        \hline
        ICL DM  & \textbf{1.29e-1} & \textbf{5.35e-4} & \textbf{1.17e-1} & \textbf{4.74e-4} & \textbf{9.79e-2} & \textbf{3.82e-4} & {2.30e-1} & \textbf{5.92e-4} \\
        Cond DM & 1.48e-1 & 5.67e-4 & 1.34e-1 & 1.44e-3 & 1.23e-1 & 4.14e-4 & \textbf{1.84e-1} & 8.15e-4 \\
        SNO     & 3.97e-1 & 2.79e-3 & 6.05e-1 & 1.02e-2 & 1.29e-1 & 7.16e-4 & 6.49e-1 & 3.30e-3 \\
        \hline
    \end{tabular}
    \caption{Navier--Stokes results comparing ICL DM, Cond DM, and SNO in terms of in-distribution (IND) and out-of-distribution (OOD) generalization, as measured by weighted MELR and the JSD between empirical PDFs. The IND results are averaged over viscosity values sampled according to \eqref{eqn:nu_sampling}. The OOD(low), OOD(mid), and OOD(high) settings correspond to $\nu = 1.2 \times 10^{-3}$, $1.3 \times 10^{-2}$, and $6.4 \times 10^{-2}$, respectively. Lower values indicate better performance, and the best result in each column is shown in bold. During both training and evaluation, we use $n = 100$ prompt samples.}
    \label{tab:NS_results}
\end{table}

\subsection{ERA5 dataset}
ERA5 \cite{hersbach2020era5} provides hourly estimates of a wide range of atmospheric, land-surface, and ocean-wave variables on a global grid, with a horizontal resolution of approximately 31 km for the main product and 137 vertical levels. In this study, we prepare the dataset as follows. First, we select nine variables: 0 = temperature, 1 = u-wind, 2 = v-wind, 3 = geopotential, 4 = vertical velocity, 5 = specific humidity, 6 = relative humidity, 7 = divergence, and 8 = vorticity. Second, we select 16 pressure levels: 150, 175, 200, 225, 250, 300, 350, 400, 450, 500, 550, 600, 650, 700, 750, and 800 hPa. For each variable--pressure-level pair, we record data at 8 hour intervals during 2023 and 2024. The resulting pairs define 144 candidate task classes, from which we select $K=100$ tasks for the experiments. For each selected task, we randomly sample $B=1000$ data points. During training and testing, the number of prompts is set to $n=100$ in both cases. Figures~\ref{fig:ERA5_ind_1}, \ref{fig:ERA5_ind_2} and \ref{fig:ERA5_ind_3} present a comparison of the data generated by the different methods. From the visual comparison, we observe that our method consistently produces samples that shares the similar pattern to the reference samples compared to the baselines. To quantitatively assess generation quality, we compute the Wasserstein-2 distance ($W_2$), the weighted and unweighted mean energy log ratios ($\mathrm{MELR}_w$ and $\mathrm{MELR}_u$), and the JSD between empirical scalar-valued PDFs. The results are reported in Table~\ref{tab:ERA5_icl_cond_metrics}, where our method outperforms the baselines on most of the evaluated metrics, demonstrating its strong ability to capture the underlying data distribution and generate high-quality samples that closely match the reference distribution.

\begin{table}[ht]
\centering
\small
\setlength{\tabcolsep}{8pt}
\begin{tabular}{c c c c c c}
\toprule
\textbf{Idx} & \textbf{Method} & $\mathbf{W_2}~\downarrow$ & $\mathbf{MELR_w}~\downarrow$ & $\mathbf{MELR_u}~\downarrow$ & $\mathbf{JSD}~\downarrow$ \\
\midrule
\multirow{3}{*}{0}
& ICL DM  & \textbf{2.68e1} & 2.43e-1 & 2.32e-1 & 1.33e-2 \\
& Cond DM & 4.39e1 & \textbf{1.38e-1} & 3.18e-1 & \textbf{1.30e-2} \\
& SNO     & 3.94e1 & 2.64e-1 & \textbf{1.85e-1} & 1.40e-2 \\
\midrule

\multirow{3}{*}{10}
& ICL DM  & \textbf{1.85e1} & \textbf{2.27e-2} & \textbf{2.87e-2} & \textbf{1.54e-3} \\
& Cond DM & 3.58e1 & 2.36e-1 & 3.01e-1 & 2.70e-2 \\
& SNO     & 2.60e1 & 3.10e-2 & 3.33e-1 & 3.37e-3 \\
\midrule

\multirow{3}{*}{20}
& ICL DM  & \textbf{2.56e1} & 1.42e-1 & \textbf{1.20e-1} & \textbf{4.59e-3} \\
& Cond DM & 3.04e1 & \textbf{1.39e-1} & 1.91e-1 & 2.08e-2 \\
& SNO     & 3.06e1 & 2.21e-1 & 2.61e-1 & 2.07e-2 \\
\midrule

\multirow{3}{*}{30}
& ICL DM  & \textbf{2.59e1} & \textbf{1.13e-1} & 2.03e-1 & \textbf{4.93e-3} \\
& Cond DM & 3.32e1 & 2.00e-1 & \textbf{9.53e-2} & 3.20e-2 \\
& SNO     & 3.01e1 & 1.35e-1 & 1.20e-1 & 1.52e-2 \\
\midrule

\multirow{3}{*}{40}
& ICL DM  & \textbf{2.32e1} & \textbf{1.81e-2} & 1.37e-1 & \textbf{1.25e-3} \\
& Cond DM & 2.96e1 & 3.45e-1 & \textbf{1.14e-1} & 6.20e-2 \\
& SNO     & 2.92e1 & 2.40e-1 & 2.58e-1 & 4.37e-2 \\
\midrule

\multirow{3}{*}{50}
& ICL DM  & \textbf{1.51e1} & \textbf{2.98e-2} & \textbf{1.38e-1} & \textbf{4.55e-3} \\
& Cond DM & 3.55e1 & 1.52e-1 & 7.47e-1 & 8.42e-3 \\
& SNO     & 3.03e1 & 1.61e-1 & 8.28e-1 & 3.47e-2 \\
\midrule

\multirow{3}{*}{60}
& ICL DM  & \textbf{1.54e1} & 5.25e-2 & \textbf{1.27e-1} & \textbf{4.45e-3} \\
& Cond DM & 2.26e1 & \textbf{1.86e-2} & 7.04e-1 & 8.89e-2 \\
& SNO     & 2.52e1 & 3.72e-2 & 1.22e0 & 1.09e-1 \\
\midrule

\multirow{3}{*}{70}
& ICL DM  & \textbf{1.25e1} & \textbf{1.33e-1} & \textbf{3.78e-1} & \textbf{5.60e-2} \\
& Cond DM & 3.77e1 & 8.06e-1 & 1.72e0 & 2.24e-1 \\
& SNO     & 2.33e1 & 1.66e-1 & 1.18e0 & 8.55e-2 \\
\midrule

\multirow{3}{*}{80}
& ICL DM  & \textbf{2.46e1} & \textbf{7.60e-2} & \textbf{6.87e-2} & \textbf{3.06e-3} \\
& Cond DM & 3.84e1 & 9.16e-1 & 3.18e-1 & 1.06e-1 \\
& SNO     & 3.16e1 & 6.55e-1 & 2.42e-1 & 5.48e-2 \\
\midrule

\multirow{3}{*}{90}
& ICL DM  & \textbf{1.91e1} & \textbf{2.17e-1} & \textbf{1.37e-1} & 7.29e-3 \\
& Cond DM & 3.79e1 & 8.63e-1 & 2.29e-1 & 8.84e-2 \\
& SNO     & 3.04e1 & 2.96e-1 & 1.64e-1 & \textbf{4.86e-3} \\
\bottomrule
\end{tabular}
\caption{In-distribution quantitative comparison of ICL DM, Cond DM, and SNO across ten ERA5 tasks, where \emph{Idx} denotes the task defined by a specific variable and pressure level. The discrepancy between the generated and reference distributions is evaluated using the Wasserstein-2 distance ($W_2$), weighted and unweighted mean energy log ratios ($\mathrm{MELR}_w$ and $\mathrm{MELR}_u$), and the Jensen--Shannon divergence (JSD). Lower values indicate better agreement with the reference distribution, and the best result for each task and metric is highlighted in bold. All methods use $n=100$ prompt samples.}
\label{tab:ERA5_icl_cond_metrics}
\end{table}

\begin{figure}
    \centering
    \includegraphics[width=0.95\linewidth]{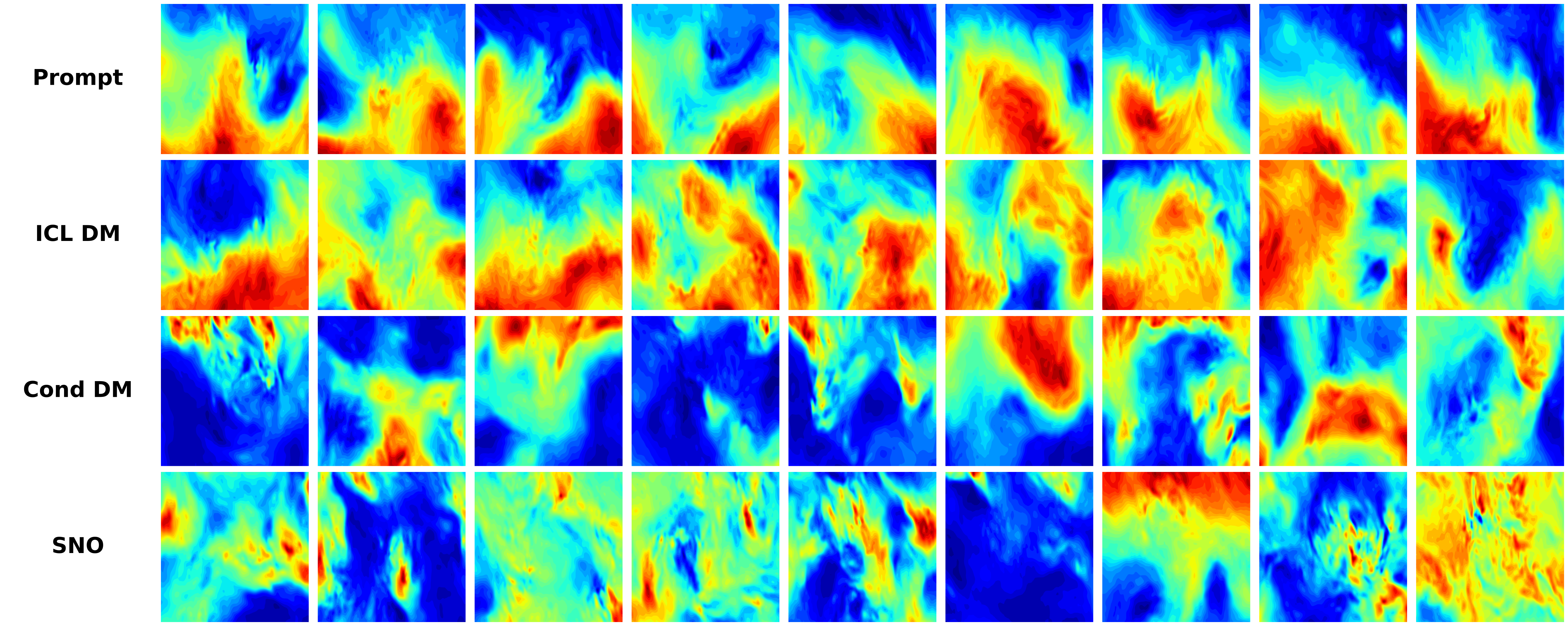}
    \includegraphics[width=0.35\linewidth]{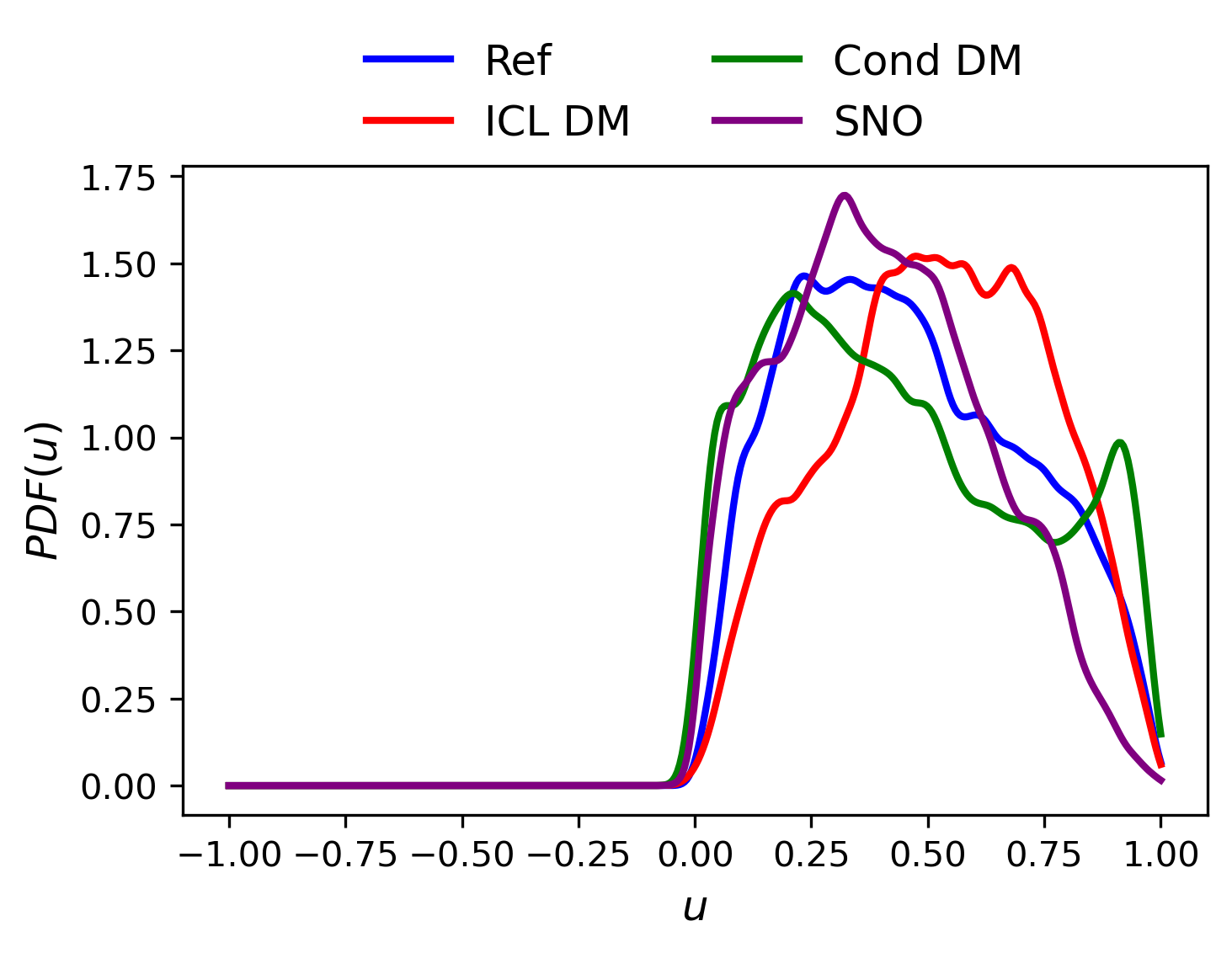}
    \includegraphics[width=0.35\linewidth]{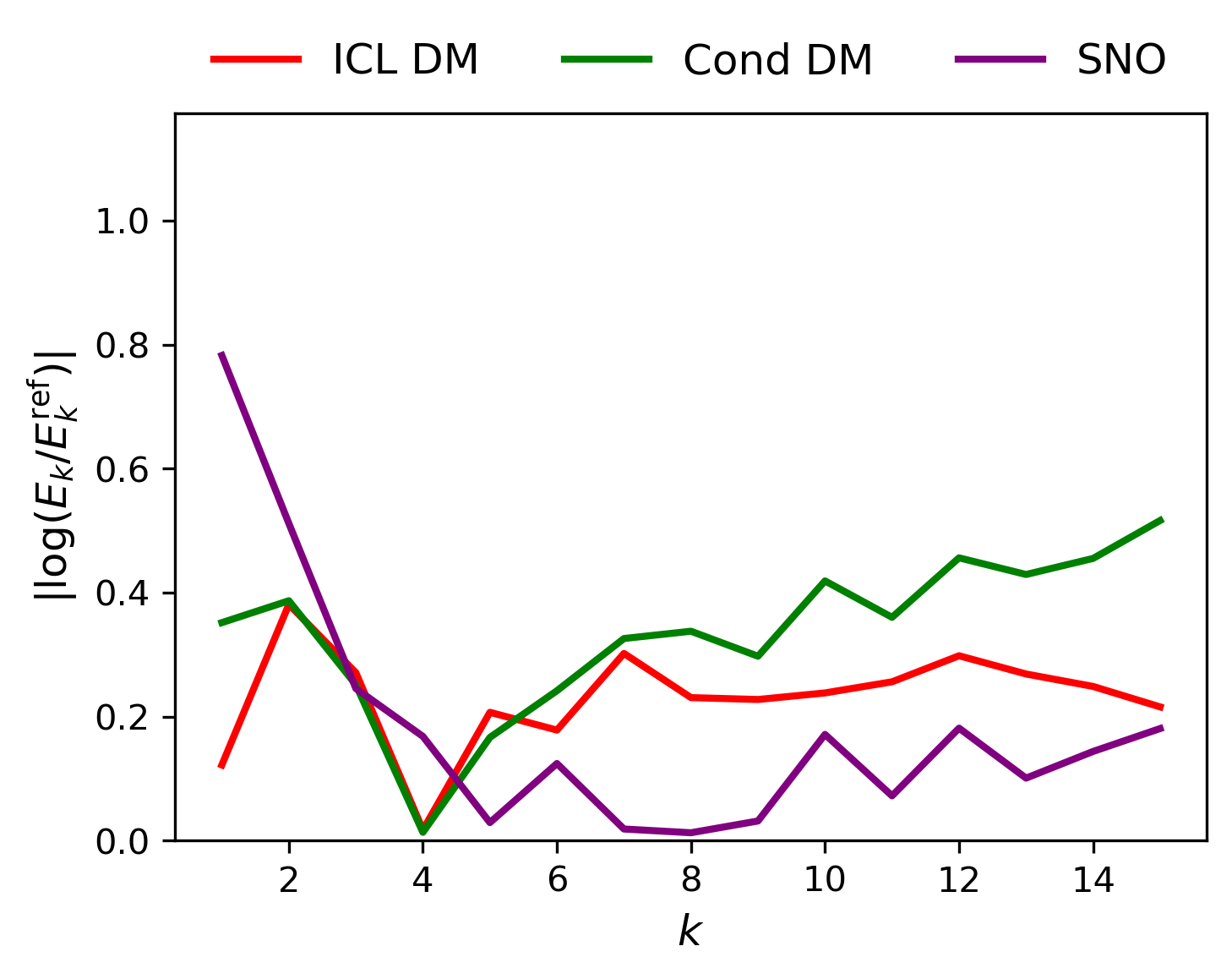}

    \caption{In-distribution generation for ERA5 temperature at 150 hPa. The upper panel shows 9 prompt samples followed by 9 samples generated by ICL DM, Cond DM, and SNO, respectively. The lower-left panel compares the scalar-valued empirical PDFs, where blue denotes the reference distribution, while red, green, and purple denote ICL DM, Cond DM, and SNO, respectively. The lower right panel reports the spectral discrepancy $\lvert \log(E_k/E_k^{\mathrm{ref}}) \rvert$ as a function of wavenumber $k$; values closer to zero indicate better agreement with the reference energy spectrum. All methods use $n = 100$ prompt samples.}
    \label{fig:ERA5_ind_1}
\end{figure}

\begin{figure}
    \centering
    \includegraphics[width=0.95\linewidth]{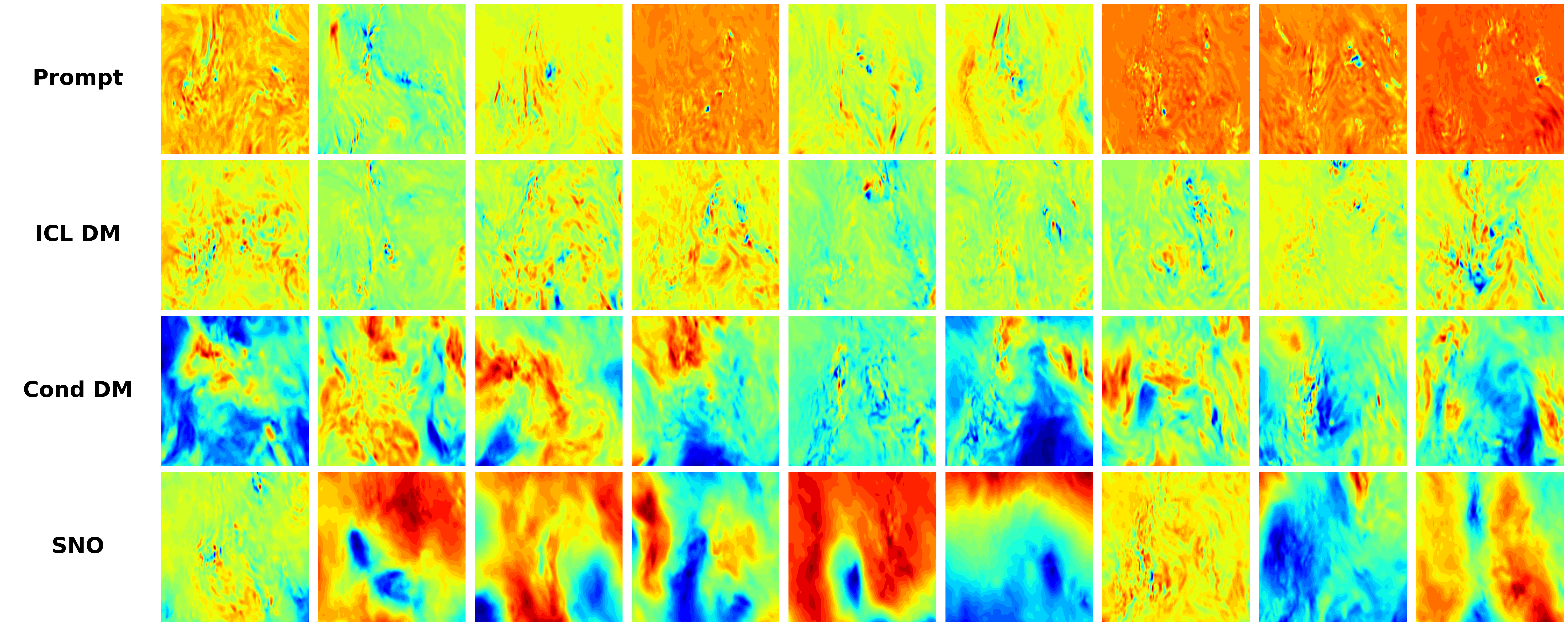}
    \includegraphics[width=0.35\linewidth]{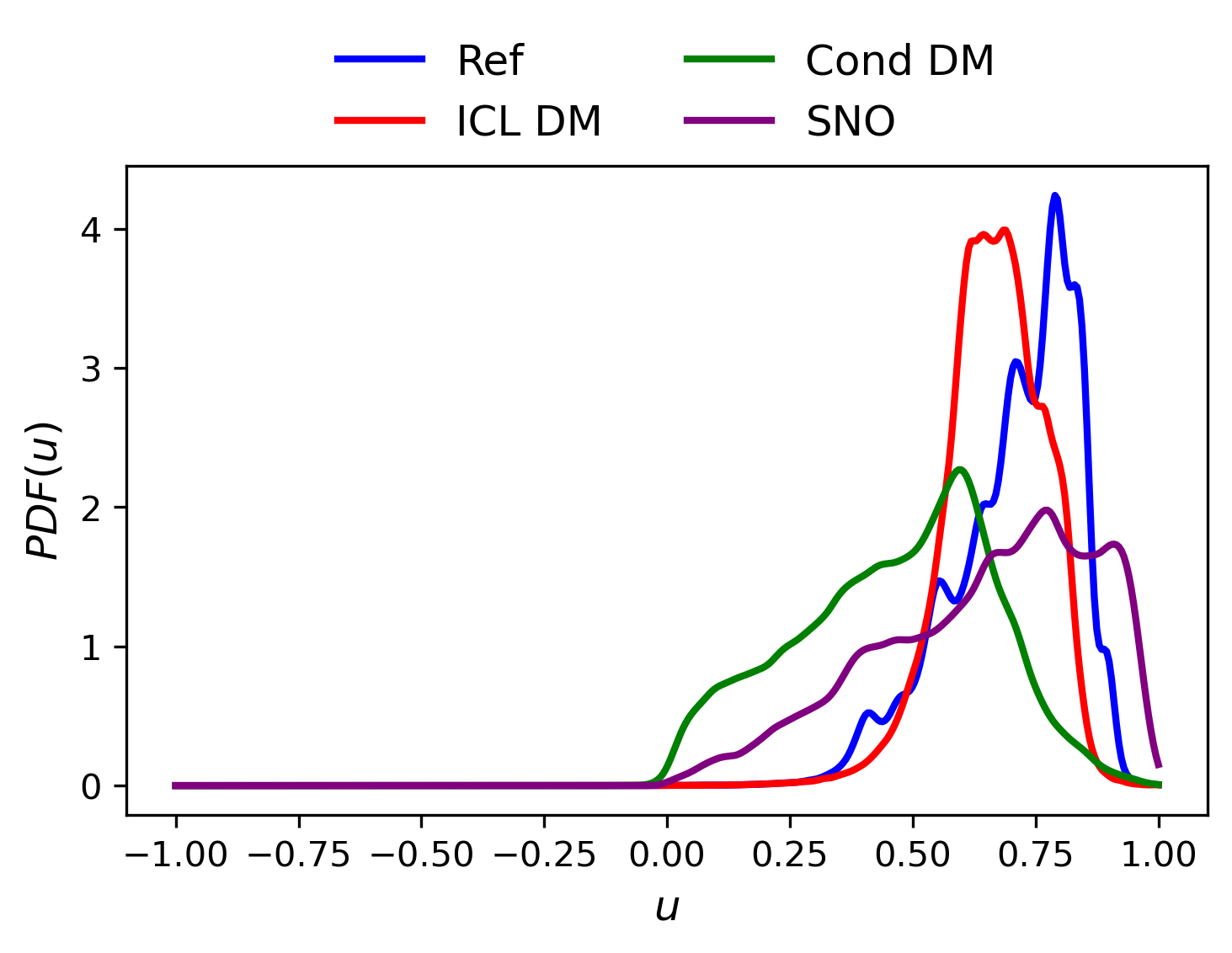}
    \includegraphics[width=0.35\linewidth]{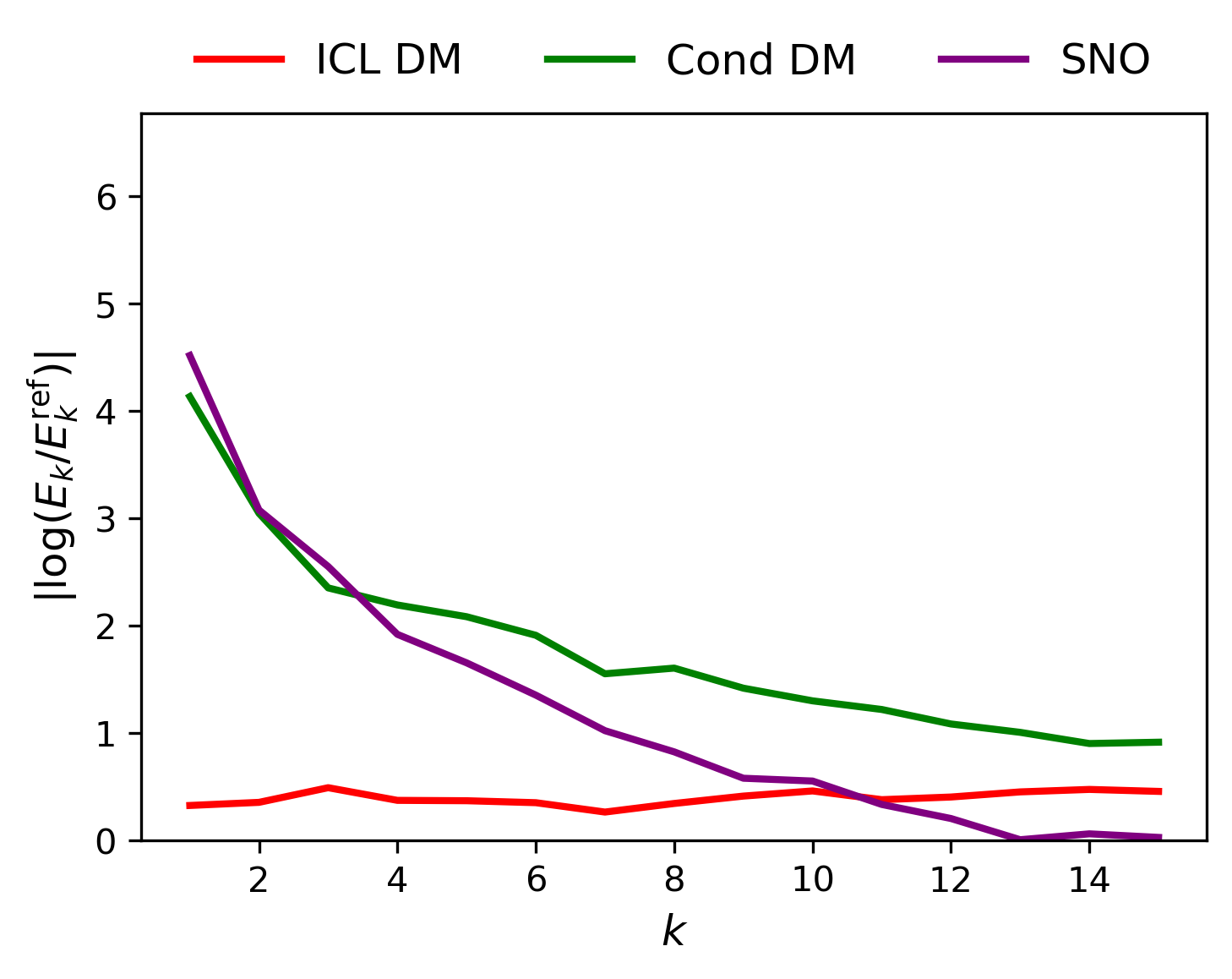}

    \caption{In-distribution generation for ERA5 humidity at 150 hPa. The upper panel shows 9 prompt fields followed by 9 samples generated by ICL DM, Cond DM, and SNO, respectively. The lower left panel compares the scalar-valued empirical PDFs, where blue denotes the reference distribution, while red, green, and purple denote ICL DM, Cond DM, and SNO, respectively. The lower right panel reports the spectral discrepancy $\lvert \log(E_k/E_k^{\mathrm{ref}}) \rvert$ as a function of wavenumber $k$; values closer to zero indicate better agreement with the reference energy spectrum. All methods use $n = 100$ prompt samples.}
    \label{fig:ERA5_ind_2}
\end{figure}

\begin{figure}
    \centering
    \includegraphics[width=0.95\linewidth]{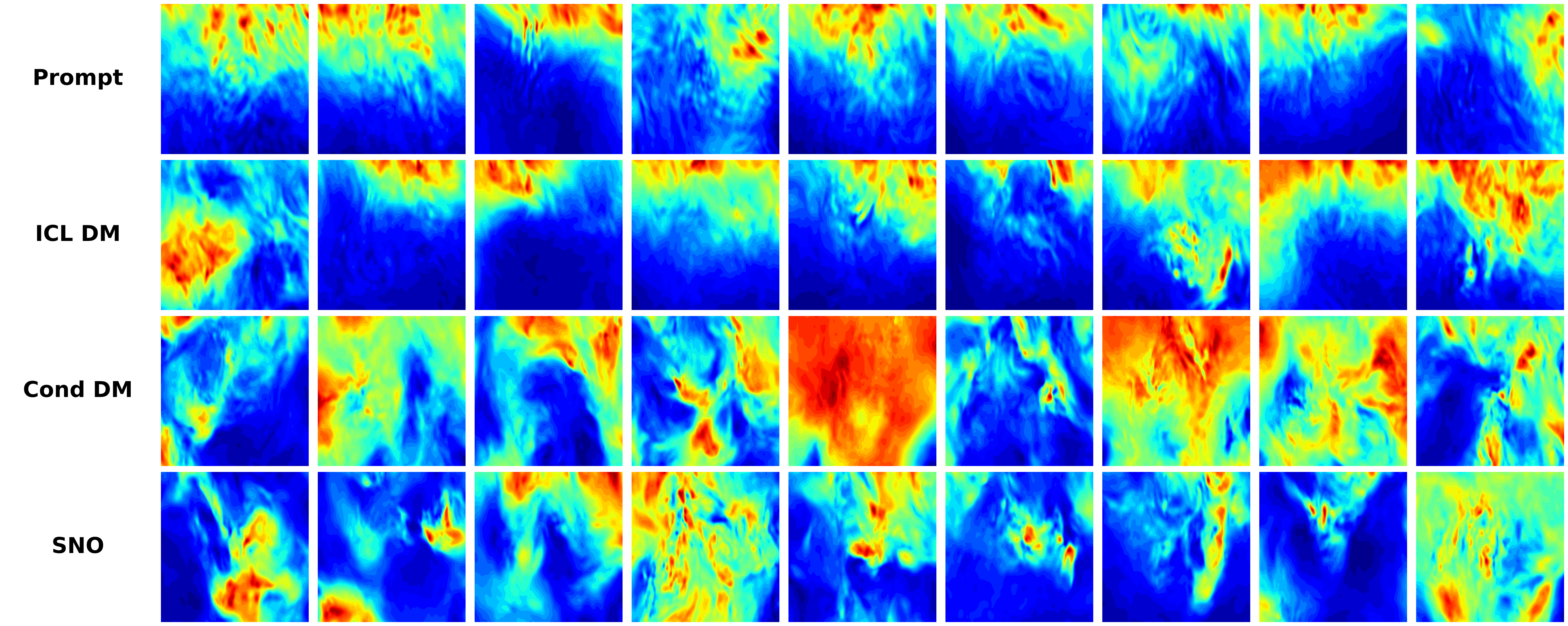}
    \includegraphics[width=0.35\linewidth]{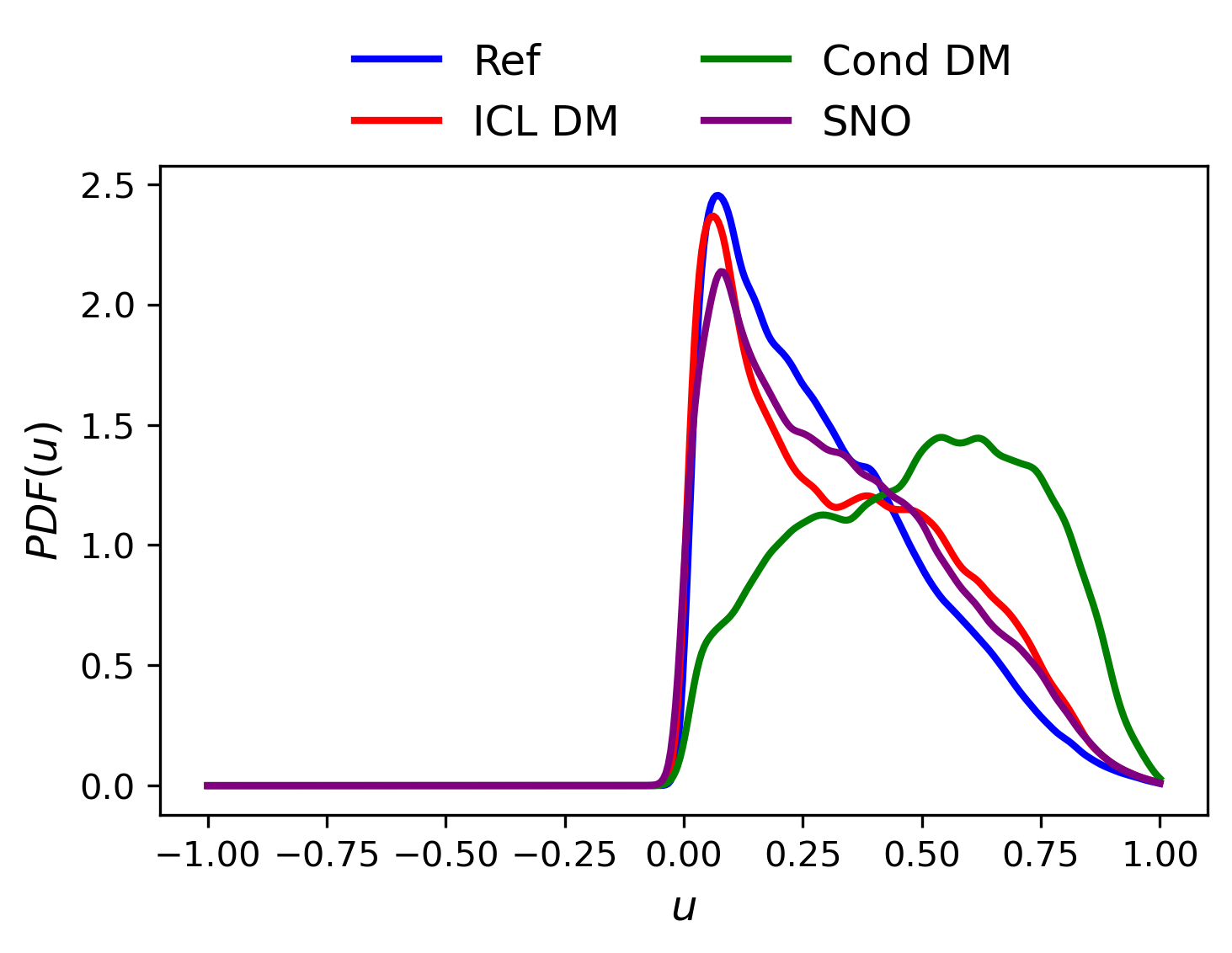}
    \includegraphics[width=0.35\linewidth]{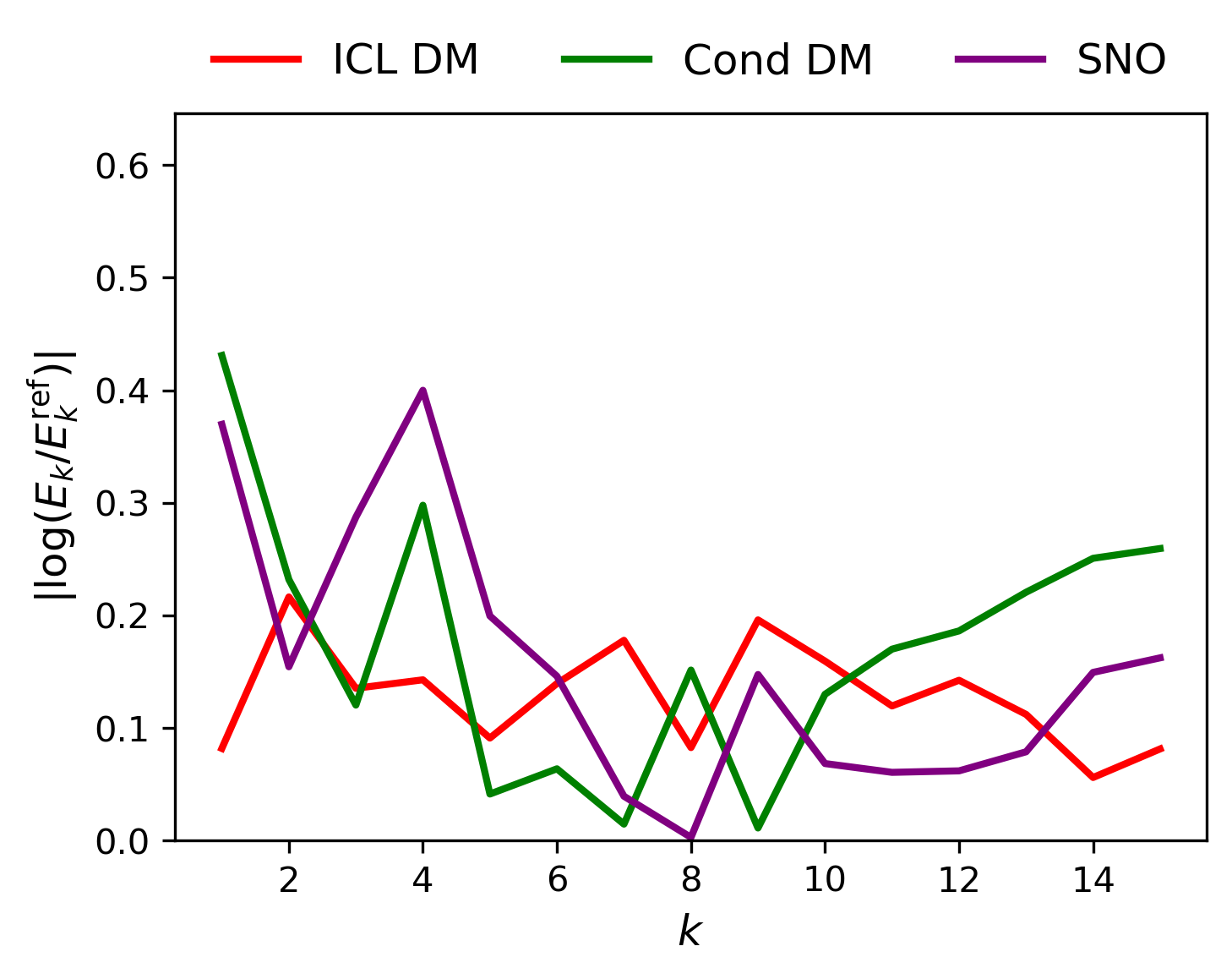}

    \caption{In-distribution generation for ERA5 divergence at 150 hPa. The upper panel shows 9 prompt fields followed by 9 samples generated by ICL DM, Cond DM, and SNO, respectively. The lower left panel compares the scalar-valued empirical PDFs, where blue denotes the reference distribution, while red, green, and purple denote ICL DM, Cond DM, and SNO, respectively. The lower right panel reports the spectral discrepancy $\lvert \log(E_k/E_k^{\mathrm{ref}}) \rvert$ as a function of wavenumber $k$; values closer to zero indicate better agreement with the reference energy spectrum. All methods use $n = 100$ prompt samples.}
    \label{fig:ERA5_ind_3}
\end{figure}

\subsection{Google sketch dataset}
Google’s Quick, Draw! (QuickDraw) sketch dataset \cite{quickdraw} is a large-scale collection of human doodles obtained through the online game “Quick, Draw!,” in which players are asked to draw an object within a short time. The dataset contains approximately 50 million drawings spanning 345 categories. In this paper, we use the first 100 categories and randomly sample 2000 sketches from each category. During both training and testing, we set the number of prompts to $n=200$. Figure~\ref{fig:IND_sketch} compares the samples generated by different methods for the case $K=5$. From the visual comparison, we observe that our method consistently produces samples that more closely resemble the reference samples than those generated by the baselines. For quantitative evaluation, we compute the FID, FID-QD, and JSDPDF metrics to measure the discrepancy between the generated and reference distributions. The results in Table~\ref{tab:Sketch} for $K=5$, $10$, and $100$ show that our method outperforms the baselines across most metrics and settings. These results indicate that our method better captures the underlying data distribution and generates higher-quality samples that more closely match the reference distribution. {We also observe an overall degradation in performance as $K$ increases. Here, increasing $K$ introduces more distinct categorical task modes into the meta-distribution, increasing task heterogeneity and making generalization more challenging. This differs from the theoretical $K^{-1/2}$ scaling, which considers increasing the number of tasks under a fixed meta-distribution.} Nevertheless, our method maintains a clear performance advantage over the baselines even for larger $K$, further demonstrating its effectiveness in complex multi-class generation tasks.

\begin{figure}
    \centering
    \begin{subfigure}[t]{0.95\linewidth}
        \centering
    \includegraphics[width=0.95\linewidth]{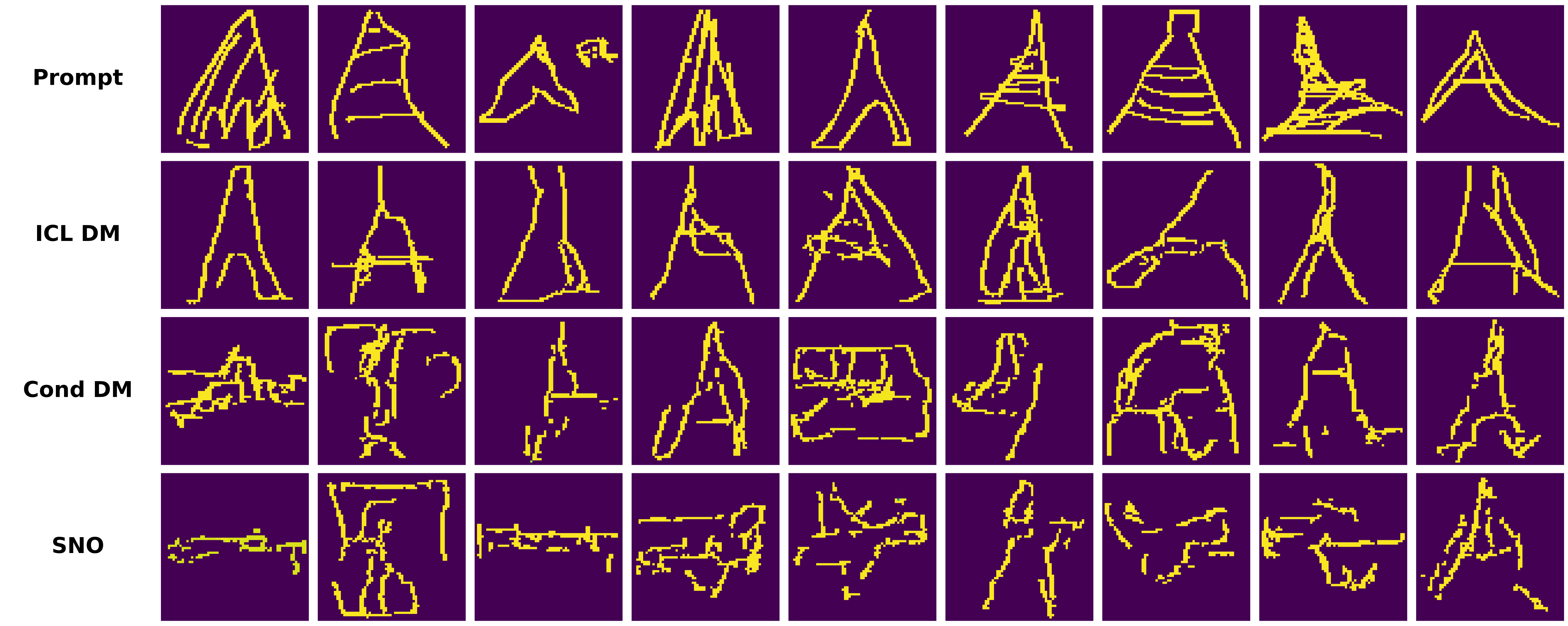}
    \caption{Eiffel Tower.}
    \end{subfigure}
    \vfill

    \begin{subfigure}[t]{0.95\linewidth}
        \centering
    \includegraphics[width=0.95\linewidth]{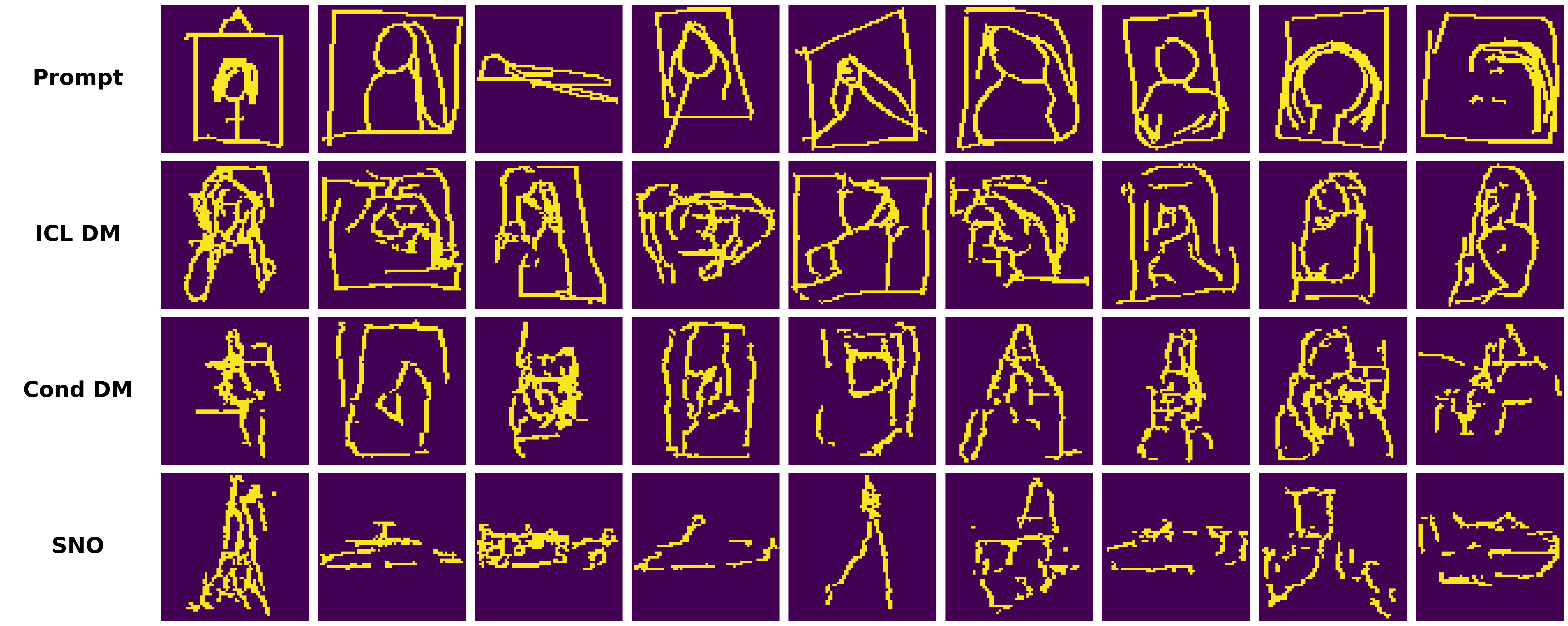}
    \caption{Mona Lisa.}
    \end{subfigure}
    \vfill

    \begin{subfigure}[t]{0.95\linewidth}
        \centering
    \includegraphics[width=0.95\linewidth]{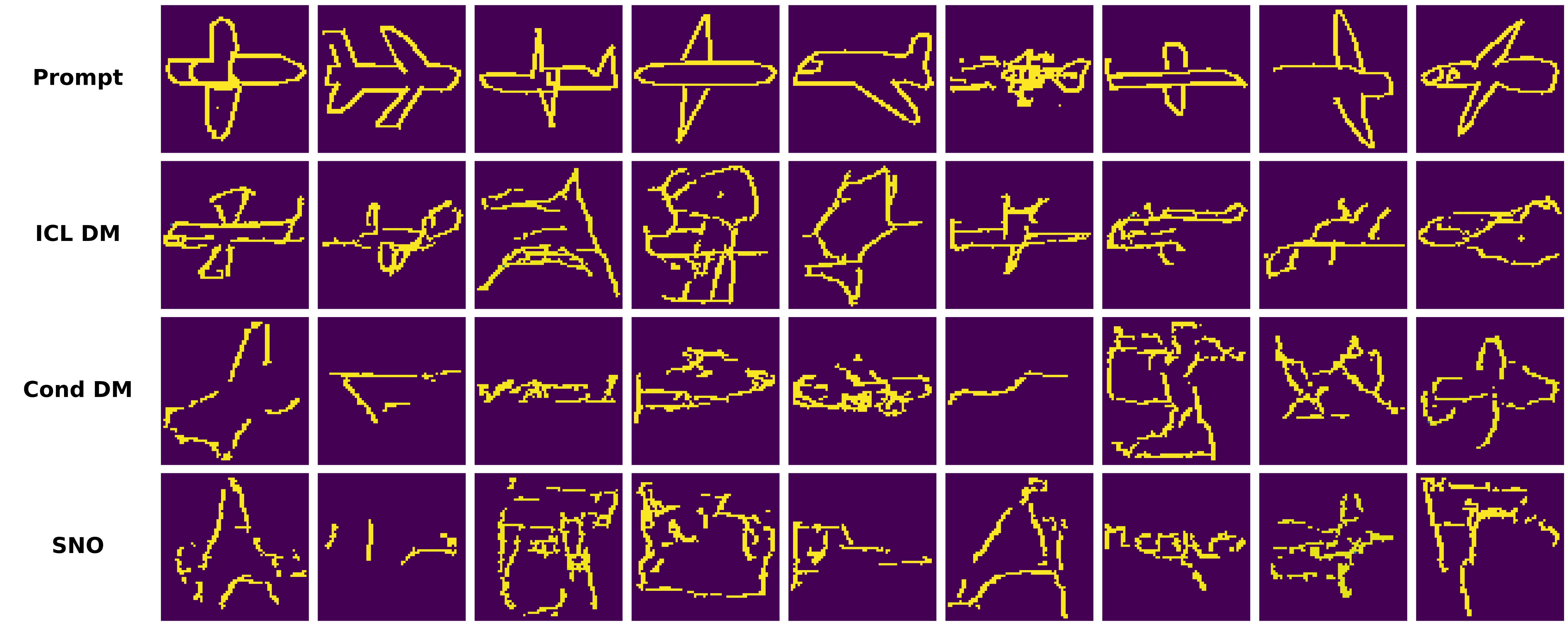}
    \caption{Airplane.}
    \end{subfigure}
    \caption{Representative in-distribution generation results on the Quick, Draw! dataset with $K = 5$ task categories. Panels (a)--(c) correspond to Eiffel Tower, Mona Lisa, and Airplane, respectively. Within each panel, the rows show prompt sketches and samples generated by ICL DM, Cond DM, and SNO, from top to bottom. All methods use $n = 200$ prompt samples, nine of which are displayed.}
    \label{fig:IND_sketch}
\end{figure}

\begin{table}[H]
    \centering
    \begin{tabular}{|c|ccc|ccc|ccc|}
        \hline
        & \multicolumn{3}{c|}{Eiffel Tower} 
        & \multicolumn{3}{c|}{Mona Lisa} 
        & \multicolumn{3}{c|}{Airplane} \\
        & FID $\downarrow$ & FID-QD $\downarrow$ & JSD $\downarrow$
        & FID $\downarrow$ & FID-QD $\downarrow$ & JSD $\downarrow$
        & FID $\downarrow$ & FID-QD $\downarrow$ & JSD $\downarrow$ \\
        \hline
        ICL DM ($K=5$)  & \textbf{122.42} & \textbf{49.72} & 1.87e-3 
                        & \textbf{102.00} & \textbf{48.96} & \textbf{1.34e-3} 
                        & \textbf{146.75} & \textbf{67.16} & \textbf{2.76e-3} \\
        Cond DM ($K=5$) & 253.54 & 96.68 & \textbf{1.81e-3} 
                        & 150.14 & 109.17 & 3.77e-3 
                        & 182.03 & 104.82 & 5.22e-3 \\
        SNO ($K=5$)     & 328.88 & 133.77 & 5.64e-3 
                        & 195.21 & 156.29 & 2.36e-2 
                        & 207.13 & 125.73 & 3.80e-3 \\
        \hline
        ICL DM ($K=10$)  & \textbf{144.59} & \textbf{50.72} & \textbf{8.45e-4} 
                         & \textbf{103.04} & \textbf{50.51} & \textbf{3.88e-3} 
                         & \textbf{146.90} & \textbf{72.76} & 1.67e-3 \\
        Cond DM ($K=10$) & 287.01 & 116.46 & 3.70e-3 
                         & 150.31 & 114.13 & 6.76e-3 
                         & 200.89 & 104.64 & 2.31e-3 \\
        SNO ($K=10$)     & 333.72 & 138.12 & 3.19e-3 
                         & 182.83 & 135.98 & 9.06e-3 
                         & 215.78 & 129.79 & \textbf{1.36e-3} \\
        \hline
        ICL DM ($K=100$)  & \textbf{217.64} & \textbf{70.31} & 2.08e-3 
                          & \textbf{142.77} & \textbf{93.19} & \textbf{4.66e-3} 
                          & \textbf{143.42} & \textbf{104.52} & 4.28e-3 \\
        Cond DM ($K=100$) & 316.93 & 116.63 & \textbf{1.54e-3} 
                          & 171.23 & 148.58 & 1.41e-2 
                          & 197.50 & 128.99 & 3.38e-3 \\
        SNO ($K=100$)     & 320.16 & 132.71 & 4.91e-3 
                          & 176.47 & 144.63 & 1.16e-2 
                          & 181.37 & 134.70 & \textbf{3.18e-3} \\
        \hline
    \end{tabular}
    \caption{In-distribution quantitative comparison on three representative QuickDraw categories as the number of meta-training tasks varies over $K\in\{5,10,100\}$. ICL DM, Cond DM, and SNO are evaluated using the standard Fr\'echet Inception Distance (FID), the QuickDraw-specific FID (FID-QD), and the Jensen--Shannon divergence (JSD) between empirical pixel-value distributions. FID-QD is computed using features extracted by a ResNet-18 pretrained on the QuickDraw dataset. Lower values indicate better distributional agreement, and the best result for each combination of $K$, category, and metric is highlighted in bold. All methods use $n=200$ prompt samples.}
    \label{tab:Sketch}
\end{table}

\subsection{Food-101 Dataset}
Food-101 \cite{bossard14} is a large-scale food image dataset consisting of 101 food categories, with 1,000 images in each category. The images exhibit substantial variation in viewpoint, scale, illumination, occlusion, and plating style. In this setting, the meta-distribution can be viewed as the overall distribution over food images, while each task distribution corresponds to a specific food category. Following the same procedure as in the QuickDraw experiments, we consider $K=5$, $10$, and $100$ categories and perform the same quantitative evaluation. {Note that all 1,000 images in each selected category are used to construct the empirical task distribution from which both training samples and inference prompts and reference samples are drawn. Thus, this experiment evaluates in-distribution generation only and does not assess out-of-distribution generalization.}

As shown in Table~\ref{tab:Food101}, our method outperforms the baselines across most metrics and settings. This suggests that our method more effectively captures the class-conditioned data distribution and generates samples that are closer to the reference distribution. {As $K$ increases, the performance of all methods generally degrades. As in the QuickDraw experiment, this reflects the increased complexity of the categorical task family under a fixed per-task data budget, prompt size, and shared model capacity, rather than the $K^{-1/2}$ scaling under a fixed meta-distribution.} Even in this more challenging regime, however, our method retains a substantial advantage over the baselines.

{These results should not be interpreted as an inherent limitation of the proposed method to categorical task distributions. Rather, they reflect the resource-limited regime considered here, where the per-task data and model capacity are fixed relative to an increasingly broad task family. Whether additional data and model capacity can mitigate this degradation is left for future work.}


\begin{figure}
    \centering
    \begin{subfigure}[t]{0.95\linewidth}
        \centering
    \includegraphics[width=0.95\linewidth]{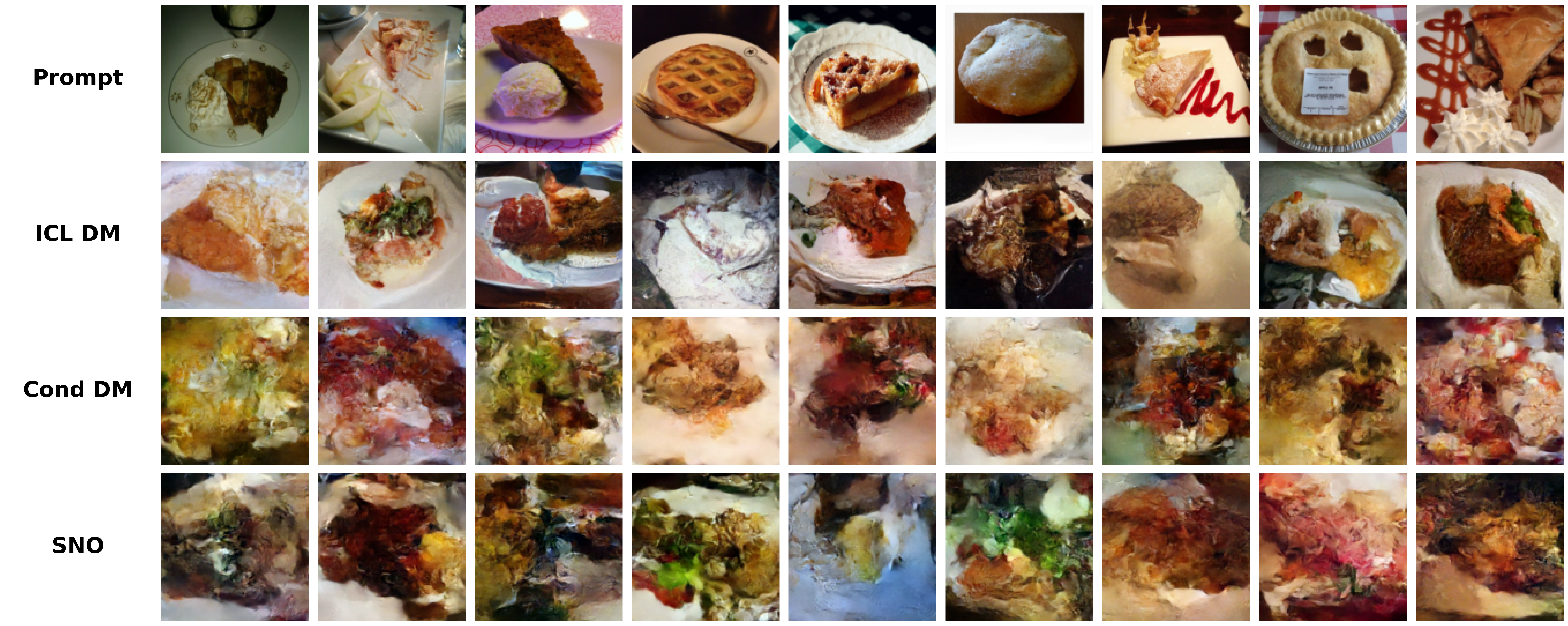}
    \caption{Apple Pie.}
    \end{subfigure}
    \vfill

    \begin{subfigure}[t]{0.95\linewidth}
        \centering
    \includegraphics[width=0.95\linewidth]{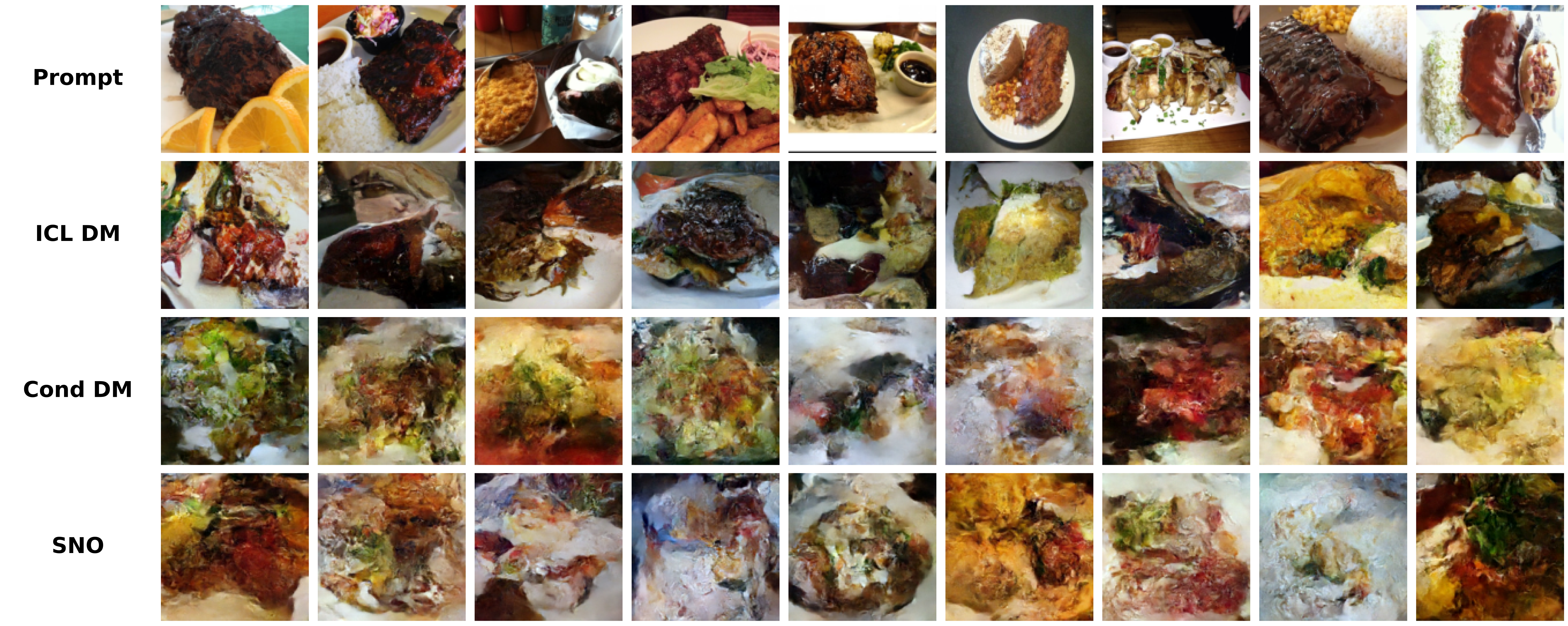}
    \caption{Baby Back Ribs.}
    \end{subfigure}
    \vfill

    \begin{subfigure}[t]{0.95\linewidth}
        \centering
    \includegraphics[width=0.95\linewidth]{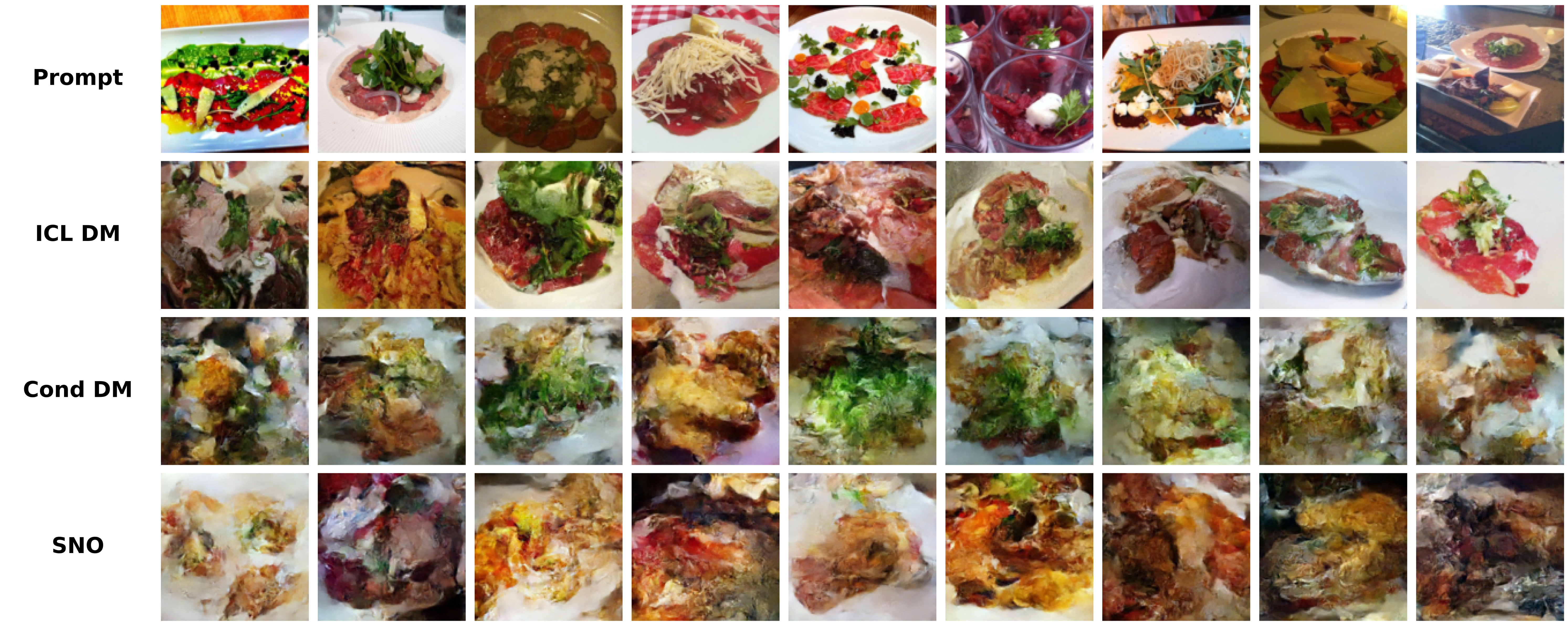}
    \caption{Beef Carpaccio.}
    \end{subfigure}
    \caption{Representative in-distribution generation results on the Food-101 dataset with $K = 5$ task categories. Panels (a)--(c) correspond to Apple Pie, Baby Back Ribs, and Beef Carpaccio, respectively. Within each panel, the rows show prompt images and samples generated by ICL DM, Cond DM, and SNO, from top to bottom. All methods use $n = 200$ prompt samples, nine of which are displayed.}
    \label{fig:IND_food}
\end{figure}

\begin{table}[H]
    \centering
    \begin{tabular}{|c|ccc|ccc|ccc|}
        \hline
        & \multicolumn{3}{c|}{Apple Pie} 
        & \multicolumn{3}{c|}{Ribs} 
        & \multicolumn{3}{c|}{Beef Carpaccio} \\
        & FID $\downarrow$ & FID-Food $\downarrow$ & JSD $\downarrow$ 
        & FID $\downarrow$ & FID-Food $\downarrow$ & JSD $\downarrow$ 
        & FID $\downarrow$ & FID-Food $\downarrow$ & JSD $\downarrow$ \\
        \hline
        ICL DM ($K=5$) 
        & \textbf{173.38} & \textbf{42.74} & \textbf{8.18e-3} 
        & \textbf{185.44} & \textbf{45.57} & \textbf{1.23e-2} 
        & \textbf{153.11} & \textbf{62.51} & \textbf{1.74e-2} \\
        Cond DM ($K=5$) 
        & 232.05 & 67.75 & 1.38e-2 
        & 246.44 & 84.22 & 4.16e-2 
        & 238.29 & 120.05 & 3.00e-2 \\
        SNO ($K=5$) 
        & 231.43 & 68.22 & 1.91e-2
        & 260.38 & 73.82 & 5.14e-2
        & 230.45 & 105.07 & 1.89e-2 \\
        \hline
        ICL DM ($K=10$) 
        & \textbf{201.67} & \textbf{46.49} & 2.54e-2
        & \textbf{195.81} & \textbf{53.34} & \textbf{2.05e-2}
        & \textbf{165.45} & \textbf{83.27} & 2.69e-2 \\
        Cond DM ($K=10$) 
        & 231.73 & 68.71 & 2.38e-2
        & 241.53 & 77.92 & 3.87e-2
        & 236.25 & 116.03 & \textbf{2.31e-2} \\
        SNO ($K=10$) 
        & 241.85 & 72.62 & \textbf{2.13e-2}
        & 259.45 & 83.33 & 1.10e-1
        & 247.94 & 131.94 & 6.52e-2 \\
        \hline
        ICL DM ($K=100$) 
        & \textbf{197.98} & \textbf{55.60} & \textbf{1.98e-2}
        & \textbf{218.12} & \textbf{60.53} & \textbf{2.81e-2}
        & \textbf{192.40} & \textbf{102.73} & \textbf{1.75e-2} \\
        Cond DM ($K=100$) 
        & 221.63 & 69.52 & 2.64e-2
        & 255.99 & 87.90 & 5.79e-2
        & 238.23 & 119.79 & 2.08e-2 \\
        SNO ($K=100$) 
        & 235.84 & 70.97 & 2.44e-2
        & 250.45 & 98.45 & 5.97e-2
        & 250.09 & 127.23 & 1.97e-2 \\
        \hline
    \end{tabular}
    \caption{In-distribution quantitative comparison of ICL DM, Cond DM, and SNO on three representative Food-101 categories as the number of meta-training tasks varies over $K\in\{5,10,100\}$. Performance is evaluated using the standard Fr\'echet Inception Distance (FID), the Food-101-specific FID (FID-Food), and the Jensen--Shannon divergence (JSD) between empirical pixel-value distributions. FID-Food is computed using features extracted by a ResNet-18 pretrained on the Food-101 dataset. Lower values indicate better distributional agreement, and the best result for each combination of $K$, category, and metric is highlighted in bold. All methods use $n=200$ prompt samples.}
    \label{tab:Food101}
\end{table}

\section{Conclusion}\label{sec:conclusion}

We introduced an in-context solution operator learning framework for diffusion models. In our formulation, each target distribution is treated as a task drawn from a meta-distribution and is specified at inference time by a finite set of i.i.d.\ context samples. The proposed in-context score operator acts directly on this empirical measure, constructing task-dependent score estimates through interactions among the context samples, the noisy query point, and the diffusion time. By training the operator through amortized denoising score matching across task distributions, reverse-diffusion sampling can be performed on previously unseen tasks without retraining, fine-tuning, or task-specific parameter updates.

A central observation underlying our method is that the task-dependent score admits a Gaussian-kernel conditional-mean representation. Its empirical approximation leads to a Nadaraya--Watson estimator that can be represented exactly by a model in our constructed transformer class. This connection provides a principled explanation for the proposed architecture: the transformer parameters encode a distribution-independent kernel-regression algorithm, while all task-specific information is supplied through the context set. We further established generalization-error bounds that separate the effects of the context size, the number of meta-training tasks, the number of episodes per task, the number of diffusion samples per episode, model complexity, and data geometry. The finite-context rate is governed by the intrinsic dimension of the task manifold, whereas the ambient dimension enters only through polynomial and logarithmic factors. Moreover, the persistent task-level error term shows that additional within-task sampling cannot compensate for insufficient meta-training task coverage, highlighting the importance of training on a sufficiently large and diverse family of distributions.

Our empirical results further demonstrate the effectiveness of the proposed framework. For task families of lower complexity, such as Navier-Stokes equation, the method achieves both strong in-distribution performance and modest out-of-distribution generalization. For task families of intermediate complexity, such as ERA5, it remains effective in the in-distribution setting. By contrast, for more complex task families, such as QuickDraw and Food-101, generalization becomes substantially more challenging, and even strong in-distribution performance can be difficult to achieve when both the number of in-context prompts and the amount of training data are limited. Importantly, these findings should not be interpreted as evidence that the method fails on more complex task families. Rather, they indicate that generalization in these settings are intrinsically more challenging under limited context and finite training data.

\section{Acknowledgments}
\newpage

\bibliographystyle{plain} 
\bibliography{reference}

\appendix
\section{Notation Table}\label{sec:notation}
\begin{table}[ht]
\centering
\small
\caption{Notation for in-context learning of diffusion-model score operators.}
\vspace{2mm}
\label{tab:notation_icldm}
\begin{tabular}{@{}ll@{}}
\toprule
\textbf{Notation} & \textbf{Meaning} \\
\midrule
$\bx \in \mathbb{R}^D$ & Data sample (e.g., an image) with ambient dimension $D$. \\
$\mM \subset \R^D$ & Data manifold \\
$d$ & Intrinsic dimension of  data manifold \\
$\mu \in \mathcal P (\mathcal{P}(\mathcal{\mathbb R^D}))$ & Meta-distribution over tasks (distribution of distributions). \\
$\rho \sim \mu$ & Task distribution sampled from the meta-distribution. \\
\midrule
$t \in [\delta,1]$ & Diffusion time. \\
$\beta(t), \alpha(t), \sigma(t)$ & Forward SDE coefficients. \\
$\varepsilon \sim \mathcal{N}(0,I)$ & Standard Gaussian noise. \\
$\bx_{k} \sim \rho_k$ & Clean sample from task distribution $\rho_k$. \\
$\bx_{k}(t) =\alpha(t)\bx_k+\sigma(t)\varepsilon$ & Noisy sample from task distribution $\rho_k$ at time $t$. \\
$\rho_{k,t}$ & Law/distribution of $\bx_t$ under task $\rho_k$. \\
$p_{\rho,t}$
& Density of the perturbed task distribution $\rho_t$; for task $\rho_k$, we write
$p_{k,t}:=p_{\rho_k,t}$. \\
$S_k(\bx_{k}(t),t)$ & Time-$t$ score function: $S_k(\bx_{k}(t),t)=\nabla_\bx\log \rho_{k,t}(\bx)\Big|_{\bx=\bx_k(t)}.$ \\
$T(\rho_k)$ & In-context solution operator. \\
\midrule
$n$ & Number of samples in a prompt/context set. \\
$\hat\rho=\{\hat\bx^{(1)},\dots,\hat\bx^{(n)}\}$ & Empirical distribution (also used as the in-context prompt). \\
$\{ \hat \rho_k \}_{k=1}^K$ & $K$ training tasks sampled from $\mu$. \\
$T_\theta(\hat\rho_k)(\bx_{k}(t),t)$ & Neural operator predicting score function at $\bx_{t,k}, t$ given prompt $\hat\rho_k$. \\
\midrule
$K$
& Number of task distributions in the meta-training set. \\

$J$
& Number of context-query episodes sampled for each training task. \\

$N_T$
& Number of diffusion-time and noise samples for each episode. \\
$\mathfrak{S}=\{\mathfrak{s}_{k}\}_{k=1}^{K}$ & Total training set. \\
$\mathfrak s_k = \{\mathbf s_{k,j,l}\}_{j=1,\;l=1}^{J,\;N_t}$ & Training set for $k$th task. \\
$\bs_{k,j, l}=\{\hat \rho_{k,j}, \bx_{k,j}(t_{k,j,l}), t_{k,j,l}\}$ & One instance containing prompt + a query sample. \\
$\hat\rho_{k,j}=\{\bx_{k,j}^{(1)},\ldots,\bx_{k,j}^{(n)} \}$ & In-context prompt for the $j$th episode for $k$th task. \\
\midrule
$\mathcal{T}(L_T,m_T,d_{\mathrm{embed}},\ell,L_{\mathrm{FFN}},w_{\mathrm{FFN}},R,  \kappa, D)$
& Transformer \emph{network class} (hypothesis class). \\
$L_T$ & Number of transformer blocks (depth). \\
$m_T$ & Maximum number of attention heads per block. \\
$d_{\mathrm{embed}}$ & Embedding dimension. \\
$\ell$ & Number of hidden tokens (sequence length inside the transformer). \\
$L_{\mathrm{FFN}}$ & Number of layers in each feed-forward subnetwork. \\
$w_{\mathrm{FFN}}$ & Hidden width of each feed-forward subnetwork. \\
$R^{\mL}$ & Uniform bound on the pointwise DSM loss; see~\eqref{eq:loss-envelope}. \\
$\kappa$ & Uniform bound on parameter magnitude: $\|\theta\|_\infty\le \kappa$. \\
\midrule
$\mathcal{L}(T)$ & Population DSM loss. \\
$\mathcal L_n(T_\theta)$
& Finite-context population DSM risk with $n$ context samples. \\
$\mathcal L_{n,K}(T_\theta)$
& Task-discretized risk over $K$ sampled task distributions. \\
$\mathcal L_{n,K,J}(T_\theta)$
& Task- and episode-discretized risk with $J$ episodes per task. \\
$\mathcal{L}_{\mathfrak{S}}(T_\theta)$ & Empirical DSM loss over $\mathfrak{S}$. \\
$T^* \in \arg \min \mL(T)$ & Population-risk minimizer. \\
$\displaystyle T_{\hat\theta}\in\mathcal T$ & Transformer approximator of $T^*$ obtained by exactly realizing the empirical kernel
  score estimator. \\
$\displaystyle T_{\theta^*; \mathfrak{S}} \in \arg \min_{T_\theta \in \mathcal{T}} \mL_{\mathfrak{S}}(T_\theta)$ & Empirical risk minimizer within hypothesis class $\mathcal T$ given dataset $\mathfrak{S}$. \\
\bottomrule
\end{tabular}
\end{table}

\section{Proofs for Section~\ref{sec:thm}}\label{sec:proof_main}
This section provides the proofs of the theoretical results stated in Section~\ref{sec:thm}. We first establish the exact transformer realization in Lemma~\ref{lemma:transformer_kernel_regression}, then derive the statistical error bounds in Lemmas~\ref{lemma:stat_error_I} and \ref{lemma:stat_error_II} and the approximation error bound in Proposition~\ref{prop:approximation_error}. Finally, we combine these results to prove Theorem~\ref{thm:generalization_error}.

\subsection{Proof of Lemma~\ref{lemma:transformer_kernel_regression}: Exact Transformer Realization}\label{sec:exact_transformer_realization}

For notational convenience, throughout the remainder of the appendix, let $\mathbf e_i \in \mathbb R^{2D+6}$ denote the $i$th standard basis vector, whose $i$th entry is one and whose remaining entries are zero. We also use $\mathbf 0_m \in \mathbb R^m$ to denote the $m$-dimensional zero column vector. Finally, $\mathbf I_{m\times n}$ and $\mathbf 0_{m\times n}$ denote the $m\times n$ rectangular identity matrix and zero matrix, respectively.

\begin{proof}
    It suffices to show that for each coordinate $j=1,\cdots,D$, there exists a $F_{\theta}\in\mathcal F$ such that for $\bx(t)\in\mathcal X_t^C$,
    \begin{equation}
        F_{\theta}(\bs) =\mathcal{K}_{j}(\bs):=\frac{
    \sum_{i=1}^n x_j^{(i)}
    \exp \left(-\frac{\|\bx(t)-\alpha(t)\bx^{(i)}\|^2}{2\sigma^2(t)}\right)
    }{
    \sum_{i=1}^n
    \exp \left(-\frac{\|\bx(t)-\alpha(t)\bx^{(i)}\|^2}{2\sigma^2(t)}\right)
    }.
    \end{equation}
Note that since $t$ is given as an input, both $\alpha(t)$ and $\sigma(t)$ are accessible. After the embedding layer $E$ and the positional encoding layer ${\rm PE}$, we have
\begin{equation}
    H = \mathrm{PE}+E(\bs) =
    \begin{bmatrix}
    \bx^{(1)}  & \cdots &  \bx^{(n)} & \bx(t) \\
    \mathbf{0}_{D+3} & \cdots & \cdots &  \mathbf{0}_{D+3} \\
    \mathcal{I}_1 & \cdots & \cdots &  \mathcal{I}_\ell \\
    1 & \cdots & \cdots  & 1
    \end{bmatrix}
    \in \mathbb{R}^{d_{\mathrm{embed}}\times \ell}
    =
    \mathbb{R}^{(2D+6)\times (n+1)}.
\end{equation} 

Here $\mathcal I_q=(\cos(q\pi/(2\ell)),\sin(q\pi/(2\ell)))^\top$ and $M$ is defined in~\eqref{eq:M}. Throughout the proof, take $u=c\sqrt M$, with an absolute constant $c$, as a uniform bound for the entries of $H,H_1,H_2,H_3$ on $\mathcal X_t^C$. We write $\mathcal B(m,L,w)$ for a transformer block with at most $m$ heads and an FFN of depth $L$ and width $w$.

We use $h_i$ to denote the $i$-th column/token in the current embedding matrix. Let us present a step-by-step construction of
$\mathcal{K}_j(\bs)$ based on $H$ as the input to the first layer of transformer, guided by Interaction Lemma~\ref{lemmaIA}.

\vspace{5mm}
\noindent
$\bullet$ Implementation of $x^{(i)}_j$ for $i=1,\cdots,n$.

For each $1\leq i\leq n$, we define each attention head $A_i$, with $V_i=\mathbf{e}_{D+1}\mathbf{e}_{j}^{\top}$ The data kernels are in the form 
\begin{equation}
\label{eqn:QK1}
Q^{data}_i =  K^{data}_i = 
\left[\begin{array}{cccc|ccc} 
	   & &  & &  & \mathbf{0}_{(2D+2)\times 3} &  \\  
     &    \mathbf{0}_{(2D+3)\times (2D+3)}& & &   &  & \\
      &   &  &  & 0 & 0 & 1
\end{array}\right] 
\end{equation}
where $Q^{data}_i,K^{data}_i\in\mathbb{R}^{(2D+3)\times (2D+6)}$. Both the $Q^{data}_i$ and $K^{data}_i$ have the last row and last column's element equal to $1$ for all $i=1,\cdots,n$.  By Interaction Lemma \ref{lemmaIA}, we can construct $A_i$, $1\leq i\leq n$, such that $h_{i}$ interacts with itself only, i.e.,
\begin{equation*}
    [A_i(H)]_{i} = \mathrm{ReLU}\left(\langle Q^{data}_ih_{i}, K^{data}_ih_{i}\rangle \right) V_ih_{i} = \mathrm{ReLU}(1)
    \begin{bmatrix}     
    \mathbf{0}_D \\
    x_j^{(i)} \\
    \mathbf{0}_{D+5}
    \end{bmatrix} = \begin{bmatrix}  
    \mathbf{0}_D \\
    x_j^{(i)} \\
    \mathbf{0}_{D+5}
    \end{bmatrix}, 
\end{equation*} 
and the output $[A_i(H)]_q=\mathbf{0}$ when $q\neq i$.

Let $\mathrm{MHA}(H):= \sum_i A_i(H)$ denote the multi-head attention, then the residual multi-head attention yields 
\begin{equation*}
    {\rm MHA}(H)+H = \begin{bmatrix}
    \bx^{(1)} & \cdots & \bx^{(n)} & \bx(t)  \\
    x^{(1)}_j & \cdots & x^{(n)}_j & 0 \\
    0 & \cdots & \cdots & 0 \\
    \vdots & \vdots & \vdots & \vdots & \\
    0 & \cdots & \cdots & 0 \\
    \mathcal{I}_1 & \cdots & \cdots & \mathcal{I}_{n+1} \\
    1 & \cdots & \cdots & 1
\end{bmatrix}.
\end{equation*}
Let the feed-forward component have all weights and biases equal to zero. Then
\begin{equation*}
    H_1:={\rm B}_1(H) = \begin{bmatrix}
    \bx^{(1)} & \cdots & \bx^{(n)} & \bx(t)  \\
    x^{(1)}_j & \cdots & x^{(n)}_j & 0 \\
    0 & \cdots & \cdots & 0 \\
    \vdots & \vdots & \vdots & \vdots & \\
    0 & \cdots & \cdots & 0 \\
    \mathcal{I}_1 & \cdots & \cdots & \mathcal{I}_{n+1} \\
    1 & \cdots & \cdots & 1
\end{bmatrix}.
\end{equation*}
In total, this process needs ${\rm B}_1\in\mathcal{B}(n,1,2D+6)$. The upper bound of the weights parameter in ${\rm B}_1$ is $\|\theta_{{\rm B}_1}\|_{\infty}\leq O(D^4\ell^2u^2b^2)=O(D^4 n^2u^2b^2)$. 

\vspace{5mm}
\noindent
$\bullet$ Implementation of $-\alpha(t)\bx^{(i)}$ for $i=1,\cdots,n$.
\\
\noindent
For each $1\leq i\leq n$, we define each attention head $A_i$, with 
\[
  V_i(t)=
  \begin{bmatrix}
  \mathbf{0}_{(D+1)\times D}
      & \mathbf{0}_{(D+1)\times(D+6)}
  \\
  -\alpha(t)I_{D\times D}
      & \mathbf{0}_{D\times(D+6)}
  \\
  \mathbf{0}_{5\times D}
      & \mathbf{0}_{5\times(D+6)}
  \end{bmatrix}
  \in\mathbb{R}^{(2D+6)\times(2D+6)}.
  \]

We chose $Q^{data}_i,K^{data}_i$ the same as those used in \eqref{eqn:QK1}.   By Interaction Lemma \ref{lemmaIA}, we can construct $A_i$, $1\leq i\leq n$, such that $h_{i}$ interacts with itself only, i.e.,
\begin{equation*}
    [A_i(H_1)]_{i} = \mathrm{ReLU}(\langle Q^{data}_ih_{i}, K^{data}_ih_{i}\rangle) V_ih_{i} = \mathrm{ReLU}(1)
    \begin{bmatrix}     
    \mathbf{0}_{D+1} \\
    -\alpha(t)\bx^{(i)} \\
    \mathbf{0}_5
    \end{bmatrix} = \begin{bmatrix}  
    \mathbf{0}_{D+1} \\
    -\alpha(t)\bx^{(i)} \\
    \mathbf{0}_5
    \end{bmatrix},
\end{equation*}
and $[A_i(H_1)]_t=0$ when $t\neq i$.  

Therefore, the residual multi-head attention yields 
\begin{equation*}
    {\rm MHA}(H_1)+H_1 = \begin{bmatrix}
    \bx^{(1)} & \cdots & \bx^{(n)} & \bx(t)  \\
    x^{(1)}_j & \cdots & x^{(n)}_j & 0 \\
    -\alpha(t)\bx^{(1)} & \cdots & -\alpha(t)\bx^{(n)} & 0\\
    0 & \cdots & \cdots & 0 \\
    0 & \cdots & \cdots & 0 \\
    \mathcal{I}_1 & \cdots & \cdots & \mathcal{I}_{n+1} \\
    1 & \cdots & \cdots & 1
\end{bmatrix}.
\end{equation*}
Let the feed-forward component have all weights and biases equal to zero. Then
\begin{equation*}
    H_2:={\rm B}_2(H_1) = \begin{bmatrix}
    \bx^{(1)} & \cdots & \bx^{(n)} & \bx(t)  \\
    x^{(1)}_j & \cdots & x^{(n)}_j & 0 \\
    -\alpha(t)\bx^{(1)} & \cdots & -\alpha(t)\bx^{(n)} & 0\\
    0 & \cdots & \cdots & 0 \\
    0 & \cdots & \cdots & 0 \\
    \mathcal{I}_1 & \cdots & \cdots & \mathcal{I}_{n+1} \\
    1 & \cdots & \cdots & 1
\end{bmatrix}.
\end{equation*}
In total, this process needs ${\rm B}_2\in\mathcal{B}(n,1,2D+6)$. The upper bound of the weights parameter in ${\rm B}_2$ is $\|\theta_{{\rm B}_2}\|_{\infty}\leq O(D^4\ell^2u^2b^2)=O(D^4 n^2u^2b^2)$. 

\vspace{5mm}
\noindent
$\bullet$ Implementation of $\bx(t)-\alpha(t)\bx^{(i)}$ for $i=1,\cdots,n$.
\\
\noindent
For each $1\leq i\leq n$, we define each attention head $A_i$, with 
\[
  V_i=
  \begin{bmatrix}
  \mathbf{0}_{(D+1)\times D}
      & \mathbf{0}_{(D+1)\times(D+6)}
  \\
  I_{D\times D}
      & \mathbf{0}_{D\times(D+6)}
  \\
  \mathbf{0}_{5\times D}
      & \mathbf{0}_{5\times(D+6)}
  \end{bmatrix}
  \in\mathbb{R}^{(2D+6)\times(2D+6)},
  \]
We chose $Q^{data}_i,K^{data}_i$ the same as those used in \eqref{eqn:QK1}.   By Interaction Lemma \ref{lemmaIA}, we can construct $A_i$, $1\leq i\leq n$, such that $h_{i}$ interacts with $h_{n+1}$ only, i.e.,
\begin{equation*}
    [A_i(H_2)]_{i} = \mathrm{ReLU}(\langle Q^{data}_ih_{i}, K^{data}_ih_{n+1}\rangle) V_ih_{n+1} = \mathrm{ReLU}(1)
    \begin{bmatrix}     
    \mathbf{0}_{D+1} \\
    \bx(t) \\
    \mathbf{0}_5
    \end{bmatrix} = \begin{bmatrix}  
    \mathbf{0}_{D+1} \\
    \bx(t) \\
    \mathbf{0}_5
    \end{bmatrix},
\end{equation*}
and $[A_i(H_2)]_t=0$ when $t\neq i$.  
Therefore, the residual multi-head attention yields 

\begin{equation*}
    {\rm MHA}(H_2)+H_2 = \begin{bmatrix}
    \bx^{(1)} & \cdots & \bx^{(n)} & \bx(t)  \\
    x^{(1)}_j & \cdots & x^{(n)}_j & 0 \\
    \bx(t)-\alpha(t)\bx^{(1)} & \cdots & \bx(t)-\alpha(t)\bx^{(n)} & 0\\
    0 & \cdots & \cdots & 0 \\
    0 & \cdots & \cdots & 0 \\
    \mathcal{I}_1 & \cdots & \cdots & \mathcal{I}_{n+1} \\
    1 & \cdots & \cdots & 1
\end{bmatrix}.
\end{equation*}
Let the feed-forward component have all weights and biases equal to zero. Then
\begin{equation*}
    H_3:={\rm B}_3(H_2) = \begin{bmatrix}
    \bx^{(1)} & \cdots & \bx^{(n)} & \bx(t)  \\
    x^{(1)}_j & \cdots & x^{(n)}_j & 0 \\
    \bx(t)-\alpha(t)\bx^{(1)} & \cdots & \bx(t)-\alpha(t)\bx^{(n)} & 0\\
    0 & \cdots & \cdots & 0 \\
    0 & \cdots & \cdots & 0 \\
    \mathcal{I}_1 & \cdots & \cdots & \mathcal{I}_{n+1} \\
    1 & \cdots & \cdots & 1
\end{bmatrix}.
\end{equation*}
In total, this process needs ${\rm B}_3\in\mathcal{B}(n,1,2D+6)$. The upper bound of the weights parameter in ${\rm B}_3$ is $\|\theta_{{\rm B}_3}\|_{\infty}\leq O(D^4\ell^2u^2b^2)=O(D^4 n^2u^2b^2)$.

\vspace{5mm}
\noindent
$\bullet$ Implementation of $-\frac{\|\bx(t)-\alpha(t)\bx^{(i)}\|^2}{2\sigma^2(t)}+M$ for $i=1,\cdots,n$.
\\
\noindent
Let us define each attention head $A_{i}$ with $V_i=\mathbf{e}_{2D+2}\mathbf{e}_{2D+6}^{\top}$, and data kernels in the form 

For \(c\in\mathbb R\), define
\[
\mathbf C(c):=
\begin{bmatrix}
0&0&0\\
0&0&c
\end{bmatrix}
\in\mathbb R^{2\times3}.
\]
For $s\in\{-1,1\}$, define
\begin{equation*}
\mathbf W(s):=
\begin{bmatrix}
\mathbf 0_{(D+1)\times(D+1)}
&
\mathbf 0_{(D+1)\times D}
&
\mathbf 0_{(D+1)\times2}
&
\mathbf 0_{(D+1)\times3}
\\
\mathbf 0_{D\times(D+1)}
&
\dfrac{s}{\sqrt{2}\sigma(t)}I_D
&
\mathbf 0_{D\times2}
&
\mathbf 0_{D\times3}
\\
\mathbf 0_{2\times(D+1)}
&
\mathbf 0_{2\times D}
&
\mathbf 0_{2\times2}
&
\mathbf C(\sqrt{M})
\end{bmatrix}\in\mathbb{R}^{(2D+3)\times (2D+6)}.
\end{equation*}
We then set
\[
Q_i^{\mathrm{data}}=\mathbf W(1),
\qquad
K_i^{\mathrm{data}}=\mathbf W(-1).
\]

By Interaction Lemma \ref{lemmaIA}, we can construct $A_{i}$ such that $h_{i}$ interacts with itself only, i.e.,
\begin{align*}
    [A_{i}(H_3)]_{i} &= \mathrm{ReLU}(\langle Q^{data}_{i}h_{i}, K^{data}_{i}h_{i}\rangle) V_ih_{i} \\ 
    &= \mathrm{ReLU}\left(-\frac{1}{2\sigma(t)^2}\langle (\bx(t)-\alpha(t)\bx^{(i)}), (\bx(t)-\alpha(t)\bx^{(i)})\rangle+M\right)\mathbf{e}_{2D+2} \\
    &= \left(-\frac{\|(\bx(t)-\alpha(t)\bx^{(i)})\|^2}{2\sigma(t)^2}+M\right)\mathbf{e}_{2D+2}
\end{align*}
and $[{\rm A}_{i}(H_3)]_q=0$ for $q\neq i$. Here the choice of $M$ in~\eqref{eq:M} ensures that the argument of ReLU is nonnegative on $\mathcal X_t^C$, since $\|\bx(t)-\alpha(t)\bx^{(i)}\|_2/\sigma(t)\le 2b\sqrt D/\sigma(\delta)+C$.

Therefore, the residual multi-head attention yields 

\begin{equation*}
    {\rm MHA}(H_3)+H_3 = \begin{bmatrix}
    \bx^{(1)} & \cdots & \bx^{(n)} & \bx(t)  \\
    x^{(1)}_j & \cdots & x^{(n)}_j & 0 \\
    \bx(t)-\alpha(t)\bx^{(1)} & \cdots & \bx(t)-\alpha(t)\bx^{(n)} & 0\\
    -\frac{\|(\bx(t)-\alpha(t)\bx^{(1)})\|^2}{2\sigma(t)^2}+M & \cdots & -\frac{\|(\bx(t)-\alpha(t)\bx^{(n)})\|^2}{2\sigma(t)^2}+M & 0 \\
    0 & \cdots & \cdots & 0 \\
    \mathcal{I}_1 & \cdots & \cdots & \mathcal{I}_{n+1} \\
    1 & \cdots & \cdots & 1
\end{bmatrix}.
\end{equation*}
Let the feed-forward component have all weights and biases equal to zero, retaining the common constant $M$. Then
\begin{equation*}
    H_4:={\rm B}_4(H_3) =\begin{bmatrix}
    \bx^{(1)} & \cdots & \bx^{(n)} & \bx(t)  \\
    x^{(1)}_j & \cdots & x^{(n)}_j & 0 \\
    \bx(t)-\alpha(t)\bx^{(1)} & \cdots & \bx(t)-\alpha(t)\bx^{(n)} & 0\\
    -\frac{\|(\bx(t)-\alpha(t)\bx^{(1)})\|^2}{2\sigma(t)^2}+M & \cdots & -\frac{\|(\bx(t)-\alpha(t)\bx^{(n)})\|^2}{2\sigma(t)^2}+M & 0 \\
    0 & \cdots & \cdots & 0 \\
    \mathcal{I}_1 & \cdots & \cdots & \mathcal{I}_{n+1} \\
    1 & \cdots & \cdots & 1
\end{bmatrix}.
\end{equation*}
In total, this process needs ${\rm B}_4\in\mathcal{B}(n,6,2D+6)$. 
The upper bound of the weights parameter in ${\rm B}_4$ is $\|\theta_{{\rm B}_4}\|_{\infty}\leq O(D^4\ell^2M^2\cdot M^2)=O(D^4\ell^2M^4)$.

\vspace{5mm}
\noindent
$\bullet$ Implementation of $\mathcal{K}_j(\bs)$ in the $(D+1)$-th row and $(n+1)$-th column.

For the last layer of transformer, we will use softmax activation (with a mask of size $n$). We can construct $V=\mathbf{e}_{D+1}\mathbf{e}^{\top}_{D+1}$, and data kernels  
\begin{equation*}
Q^{data} = 
\left[\begin{array}{cccc|ccc} 
	0 & &  & &  0 & 0 & 0 \\ 
     & \ddots &  & & \vdots & \vdots & \vdots \\
     & & 0 &  & 0 & 0 & 1 \\ 
	 & & & 0 &  0 & 0 & 0 \\ 
\end{array}\right], \quad 
K^{data} = 
\left[\begin{array}{cccc|ccc} 
	0 & & &  &  0 & 0 & 0 \\ 
     & \ddots &  & & \vdots & \vdots & \vdots \\
     & & 1 &   & 0 & 0 & 0 \\ 
	 & & & 0  & 0 & 0 & 0 \\ 
\end{array}\right], 
\end{equation*}
where $Q^{data}\in\mathbb{R}^{(2D+3)\times(2D+6)}$ has $(2D+2,2D+6)$-th position equals to $1$ and all the other entries are $0$, and $K^{data}\in\mathbb{R}^{(2D+3)\times(2D+6)}$ has $(2D+2,2D+2)$-th position equals to $1$ and all the other entries are $0$,
such that 

\begin{align*}
    [{\rm A}(H_4)]_{n+1}
    &= \sum_{i=1}^{n}{\rm softmax}\bigl((K^{data}H_4)^\top Q^{data}h_{n+1}\bigr)_i Vh_i \\
    &=\frac{
    \sum_{i=1}^n x_j^{(i)}
    \exp\left(M-\frac{\|\bx(t)-\alpha(t)\bx^{(i)}\|^2}{2\sigma^2(t)}\right)
    }{
    \sum_{i=1}^n
    \exp\left(M-\frac{\|\bx(t)-\alpha(t)\bx^{(i)}\|^2}{2\sigma^2(t)}\right)
    }\mathbf e_{D+1}
    =\mathcal K_j(\bs)\mathbf e_{D+1}.
\end{align*}
The common factor $e^M$ cancels in the ratio. Since $Q^{data}h_q$ is the same for every token $q$, the attention output is also $\mathcal K_j(\bs)\mathbf e_{D+1}$ at the prompt tokens.

Therefore, the residual attention gives
\begin{equation*}
    {\rm A}(H_4)+H_4 =\begin{bmatrix}
    \bx^{(1)} & \cdots & \bx^{(n)} & \bx(t)  \\
    x^{(1)}_j+\mathcal K_j(\bs) & \cdots & x^{(n)}_j+\mathcal K_j(\bs) & \mathcal{K}_j(\bs) \\
    \bx(t)-\alpha(t)\bx^{(1)} & \cdots & \bx(t)-\alpha(t)\bx^{(n)} & 0\\
    -\frac{\|(\bx(t)-\alpha(t)\bx^{(1)})\|^2}{2\sigma(t)^2}+M & \cdots & -\frac{\|(\bx(t)-\alpha(t)\bx^{(n)})\|^2}{2\sigma(t)^2}+M & 0 \\
    0 & \cdots & \cdots & 0 \\
    \mathcal{I}_1 & \cdots & \cdots & \mathcal{I}_{n+1} \\
    1 & \cdots & \cdots & 1
\end{bmatrix}.
\end{equation*}

For the final residual FFN, let $z$ denote the $(D+1)$-th input coordinate and set
\[
f(z)=-b+\mathrm{ReLU}(z+b)-\mathrm{ReLU}(z-b)-\mathrm{ReLU}(z)+\mathrm{ReLU}(-z)
\]
in row $D+1$, with all other output rows zero. Then $z+f(z)=\max\{-b,\min\{z,b\}\}$. Four ReLU units suffice; identity maps on their nonnegative activations pad the FFN to depth six. Since $|\mathcal K_j(\bs)|\le b$, the query output is unchanged on $\mathcal X_t^C$ and bounded by $b$ for every noisy query.

Therefore, there exists ${\rm B}_5\in\mathcal B(1,6,2D+6)$ such that, on $\mathcal X_t^C$,
\[
    [H_5]_{n+1}:=[{\rm B}_5(H_4)]_{n+1}
    =\begin{bmatrix}
    \bx(t)\\
    \mathcal K_j(\bs)\\
    \mathbf 0_{D+2}\\
    \mathcal I_{n+1}\\
    1
    \end{bmatrix}.
\]

The upper bound of the weights parameter in ${\rm B}_5$ is $\|\theta_{{\rm B}_5}\|_{\infty}\leq O(D^4\ell^2M^2\cdot 1)=O\left(D^4\ell^2M^2\right)$. Finally, applying a decoding layer ${\rm DE}$ to output the element $\mathcal{K}_j(\bs)$, so we get $F_{\theta}(\bs) =\mathcal{K}_{j}(\bs)$ on $\mathcal X_t^C$, with $|F_\theta(\bs)|\le b$ for all $\bx(t)\in\mathbb R^D$, as desired. 

The uniform upper bound for the weight parameters in ${\rm B}_5\circ{\rm B}_4\circ {\rm B}_3\circ {\rm B}_2\circ{\rm B}_1$ is $\kappa\leq O\left(D^4 \ell^2 M^4\right)$. All blocks fit $L_{\mathrm{FFN}}=6$ and $w_{\mathrm{FFN}}=2D+6$. By stacking $D$ such networks componentwise, we can represent the full vector-valued kernel regression estimator in \eqref{eqn:cond_exp_discretized} on $\mathcal X_t^C$, with $\|G_\theta(\bs)\|_\infty\le b$ for all $\bx(t)\in\mathbb R^D$.
\end{proof}

We now fix
\begin{equation}\label{eq:Cdelta}
C_\delta^2=2D\log2+24\log(1/\sigma(\delta)),
\qquad q_\delta:=\mathbb P(\|\eps\|_2>C_\delta).
\end{equation}
Since $\E\exp(\|\eps\|_2^2/4)=2^{D/2}$, Markov's inequality gives
\begin{equation}\label{eq:tail-prob}
q_\delta\le2^{D/2}e^{-C_\delta^2/4}=\sigma(\delta)^6.
\end{equation}
The radius~\eqref{eq:Cdelta} is compatible with the original weight-budget order. Indeed,
\begin{align}
\frac12\left(\frac{2b\sqrt D}{\sigma(\delta)}+C_\delta\right)^2
&\le\frac{4b^2D}{\sigma(\delta)^2}+2D\log2+24\log(1/\sigma(\delta)),\nonumber\\
M&\le cDb^2R\sigma(\delta)^{-2},\nonumber\\
\kappa=c_0D^4(n+1)^2M^4
&\le c_1\frac{D^8n^2b^8R^4}{\sigma(\delta)^8}.\label{eq:kappa-delta}
\end{align}
Here we used $D\ge1$, $R\ge b\ge1$, and $\log(1/s)\le s^{-2}$ for $0<s\le1$.

\paragraph{Token budget.}\label{rem:tokens}
The realization uses $n+1$ active tokens: the $n$ prompt tokens followed by one query token. This is the token budget used in Lemma~\ref{lemma:transformer_kernel_regression}, Proposition~\ref{prop:approximation_error}, and Theorem~\ref{thm:generalization_error}.

\subsection{Statistic Error I}\label{sec:stat_error_I}
Define the pointwise loss
\[
f_\theta\left(\hat\rho,\bx(t),t,\eps\right)
:=
\left\|
T_\theta(\hat\rho)(\bx(t),t)+\frac{\eps}{\sigma(t)}
\right\|_2^2.
\]
Then define the episode-level loss
\[
h_\theta(\hat\rho,\bx)
:=
\E_{t\sim \mathrm{U}[\delta,1],\ \eps\sim \mathcal N(0,I)}
\left[
f_\theta \left(
\hat\rho,\,
\bx(t),\,
t,\,
\eps
\right)
\right],
\]
and the task-level loss
\[
g_\theta(\rho)
:=
\E_{\hat\rho \sim \rho^{\otimes n},\ \bx\sim \rho}
\left[
h_\theta(\hat\rho,\bx)
\right].
\]
With these abbreviations, the four losses are as follows:
\begin{itemize}
\item The semi-discretized population loss is
\[
\mL_n(T_\theta)
:=
\E_{\rho \sim \mu}\big[g_\theta(\rho)\big].
\]

\item The task-discretized loss is
\[
\mL_{n,K}(T_\theta)
:=
\frac{1}{K}\sum_{k=1}^K g_\theta(\rho_k).
\]

\item The task- and episode-discretized loss is
\[
\mL_{n,K,J}(T_\theta)
:=
\frac{1}{KJ}
\sum_{k=1}^K \sum_{j=1}^J
h_\theta\left(\hat\rho_{k,j},\bx_{k,j}\right).
\]

\item The fully discretized empirical loss is
\[
\mathcal L_{\mathfrak S}(T_\theta)
:=
\frac{1}{KJN_T}
\sum_{k=1}^K \sum_{j=1}^J \sum_{l=1}^{N_T}
f_\theta\left(
\hat\rho_{k,j},
\bx_{k,j}(t_{k,j,l}),
t_{k,j,l},
\eps_{k,j,l}
\right).
\]
\end{itemize}

For the statistical analysis only, set
\begin{equation}\label{eq:restricted-loss}
f_\theta^C=f_\theta\mathbf 1_{\{\|\eps\|_2\le C_\delta\}},
\qquad U_\delta=\max\{1,b+C_\delta\}.
\end{equation}
Use the same averages to define $h_\theta^C$, $g_\theta^C$, and the corresponding population and empirical losses. On this event, $\|H\|_{\infty,\infty}\le U_\delta$. This restriction is used only in the proof; the noise law and the training objective remain unchanged.

For the statistical bounds below, use the concrete width choice already available in Lemma~\ref{lemma:transformer_kernel_regression}:
$w=d_{\mathrm{embed}}=w_{\mathrm{FFN}}$, with $w=2D+6$, $L_T=5$, $L_{\mathrm{FFN}}=6$, and $m_T=n$ for the main theorem. Each scalar network has at most
\begin{equation}\label{eqn:PT}
P_T:=w(D+1)+L_T\bigl(3w^2m_T+L_{\mathrm{FFN}}(w^2+w)\bigr)
\end{equation}
free coefficients. This counts the embedding, decoder, one trainable base matrix for each Q/K/V map, and all affine FFN matrices and biases. The prescribed time templates in~\eqref{eq:fixed-schedule} are not free coefficients. Tied weights and unused heads only reduce this count.

We give an explicit covering constant by a layer recurrence. Set $s_0=\sigma(\delta)$, $\tau=2+s_0^{-1}$, and initialize
\[
u_0=1+D\kappa U_\delta,\qquad v_0=DU_\delta.
\]
For $k=0,\ldots,L_T-1$, define
\begin{align*}
q_k&=w\tau\kappa u_k,& e_k&=w(\tau u_k+\tau\kappa v_k),\\
a_k&=m_T(q_k+\ell wq_k^3),&
c_k&=m_T(1+4wq_k^2+3\ell wq_k^2)e_k,\\
z_k&=u_k+a_k,& d_k&=v_k+c_k.
\end{align*}
For the FFN in that block, set $r_{k,0}=z_k$, $t_{k,0}=d_k$, and, for $j=0,\ldots,L_{\mathrm{FFN}}-1$, let
\[
r_{k,j+1}=\kappa+w\kappa r_{k,j},\qquad
t_{k,j+1}=wr_{k,j}+1+w\kappa t_{k,j}.
\]
Complete the block recursion with
\[
u_{k+1}=z_k+r_{k,L_{\mathrm{FFN}}},\qquad
v_{k+1}=d_k+t_{k,L_{\mathrm{FFN}}}.
\]
The resulting mean-network sensitivity is
\begin{equation}\label{eqn:A0}
K_{\mathrm{net}}:=w(u_{L_T}+\kappa v_{L_T}).
\end{equation}
In particular, a maximum coefficient perturbation of $\varepsilon$ changes every output coordinate by at most $K_{\mathrm{net}}\varepsilon$, uniformly over the restricted inputs and all $t\in[\delta,1]$.

\begin{lemma}[Covering number estimation]\label{lem:cover}
For the transformer class with the width and fixed time-operation convention above, the restricted loss class satisfies, for $0<\eta\le R^{\mL}$,
\begin{equation}\label{eq:entropy}
\log\mN\bigl(\eta,\{f_\theta^C:T_\theta\in\mathcal T\},\|\cdot\|_{L^\infty}\bigr)
\le DP_T\log\frac{A_{0,\delta}}{\eta},
\end{equation}
where a sufficient loss-covering constant is
\begin{equation}\label{eq:A-delta}
A_{0,\delta}:=4\kappa(K_{\mathrm{loss}}+1)+2R^{\mL},
\qquad K_{\mathrm{loss}}:=\frac{2D(R+b)}{s_0^4}K_{\mathrm{net}}.
\end{equation}
The same bound holds for $h_\theta^C$ and $g_\theta^C$.
\end{lemma}
\begin{proof}
We prove the network estimate directly. Each scheduled matrix has entry bound $\tau\kappa$; its parameter perturbation is at most $\tau\varepsilon$. Fixed time templates have zero parameter difference. A matrix-vector product therefore has magnitude at most $q_k$ and perturbation at most $e_k\varepsilon$. The ReLU attention head has magnitude at most $\ell wq_k^3$ and perturbation at most $3\ell wq_k^2e_k\varepsilon$.

For the normalized softmax readout, if the values have magnitude at most $B$, then
\[
|\operatorname{Readout}(u,v)-\operatorname{Readout}(u',v')|
\le2B\|u-u'\|_\infty+\|v-v'\|_\infty.
\]
Indeed, the sum of the absolute score derivatives is at most $2B$, and the value derivatives are the softmax weights, whose sum is one. Interpolating scores and values proves the inequality. Since the attention-score perturbation is at most $2wq_ke_k\varepsilon$, the softmax head has perturbation at most $(1+4wq_k^2)e_k\varepsilon$. Summing these sufficient ReLU and softmax bounds over the heads gives $a_k,c_k$. Residual addition and the affine/ReLU FFN estimates give the stated recursions. The embedding and decoder give the initial values and~\eqref{eqn:A0}. These estimates use the original blocks and are uniform in time.

By~\eqref{eq:residual-cancel}, the loss conversion is
\begin{equation}\label{eq:loss-cover-conversion}
|f_\theta-f_{\theta'}|
\le\frac{2D(R+b)}{s_0^4}\|G_\theta-G_{\theta'}\|_\infty.
\end{equation}
Thus the restricted loss is $K_{\mathrm{loss}}$-Lipschitz in the coefficient maximum norm. Both restricted losses vanish outside the same noise event. Partition the $DP_T$-dimensional coefficient cube into cells of side
$r=\eta/[2(K_{\mathrm{loss}}+1)]$ and select one admissible representative from each occupied cell. Allowing endpoint cells, at most
\[
(2\kappa/r+2)^{DP_T}
=\left(\frac{4\kappa(K_{\mathrm{loss}}+1)+2\eta}{\eta}\right)^{DP_T}
\le\left(\frac{A_{0,\delta}}{\eta}\right)^{DP_T}
\]
representatives suffice. Their coefficient distance from any member of the same cell is at most $2r$, so their loss distance is strictly below $\eta$. The centers belong to the legal bounded-output class. This proves~\eqref{eq:entropy}. All network operations and the restriction event are measurable, and the common loss envelope gives integrability. Probability averaging cannot increase uniform distances, so the same parameter grid covers $h^C$ and $g^C$.
\end{proof}

We are now ready to prove Lemma~\ref{lemma:stat_error_I}.
\begin{proof}
Note that
\[
\E_{\mathfrak S}\left[\mL_n(T_{\theta^*;\mathfrak S})-\mL_{\mathfrak S}(T_{\theta^*;\mathfrak S})\right]
\le\E_{\mathfrak S}\sup_{T_\theta\in\mathcal T}
(\mL_n(T_\theta)-\mL_{\mathfrak S}(T_\theta)).
\]

Since the tail contribution to the empirical loss is nonnegative and the population tail is at most $R^{\mL} q_\delta$, for every training set,
\begin{equation}\label{eq:stat-tail}
\sup_{T_\theta\in\mathcal T}(\mL_n-\mL_{\mathfrak S})(T_\theta)
\le\sup_{T_\theta\in\mathcal T}(\mL_n^C-\mL_{\mathfrak S}^C)(T_\theta)
+R^{\mL} q_\delta.
\end{equation}
It remains to apply the hierarchical symmetrization argument to these auxiliary losses.

For each fixed $T_\theta$, decompose
\begin{align*}
\mL_n^{C}(T_\theta)-\mL_{\mathfrak S}^{C}(T_\theta)
={}&[\mL_n^{C}(T_\theta)-\mL_{n,K}^{C}(T_\theta)]\\
&+[\mL_{n,K}^{C}(T_\theta)-\mL_{n,K,J}^{C}(T_\theta)]\\
&+[\mL_{n,K,J}^{C}(T_\theta)-\mL_{\mathfrak S}^{C}(T_\theta)].
\end{align*}
The first bracket is the task-level statistical error, the second is the episode-level statistical error, and the third is the time/noise discretization error.
For the first term, standard symmetrization gives
\[
\E_{\rho_{1:K}}\sup_{T_\theta\in\mathcal T}
(\mL_n^{C}-\mL_{n,K}^{C})(T_\theta)
\le2\E_{\rho_{1:K}}\widehat{\mathrm{Rad}}_K
(\{g_\theta^{C}:T_\theta\in\mathcal T\};\rho_{1:K}).
\]

For the second term, condition on all tasks; for the third, condition on all tasks and episodes. Under these conditional laws the samples are independent, though they need not be identically distributed. Independent ghost copies and Jensen's inequality give the same symmetrization bound with $KJ$ and $KJN_T$ signs, respectively.

For a function class $\mathcal H$ evaluated on $N$ observations, write
\[
\widehat{\mathrm{Rad}}_N(\mathcal H;Z_{1:N})
:=\E_\zeta\left[\sup_{h\in\mathcal H}\frac1N\sum_{i=1}^N\zeta_i h(Z_i)\right],
\]
where the $\zeta_i$ are independent Rademacher variables. By Dudley's entropy integral,
\begin{align*}
\E\widehat{\mathrm{Rad}}_N(\mathcal H)
&\le\inf_{0<\xi\le R^{\mL}}
\left(2\xi+\frac{12}{\sqrt N}\int_\xi^{R^{\mL}}
\sqrt{\log\mN(\eta,\mathcal H,\|\cdot\|_{L^\infty})}\,d\eta\right)\\
&\le\inf_{0<\xi\le R^{\mL}}
\left(2\xi+\frac{12R^{\mL}}{\sqrt N}
\sqrt{DP_T\log(A_{0,\delta}/\xi)}\right)\\
&\le\frac{R^{\mL}}{\sqrt N}
\left(12\sqrt{DP_T\log\frac{A_{0,\delta}\sqrt N}{R^{\mL}}}+2\right).
\end{align*}
The last inequality is obtained by letting $\xi=R^{\mL}/\sqrt N$.
Apply this estimate at $N=K,KJ,KJN_T$ and sum the three bounds.
Adding~\eqref{eq:stat-tail} proves~\eqref{eq:gen_bound}.

For the stated network parameters, $DP_T=O(nD^3)$. Equations~\eqref{eq:Cdelta},~\eqref{eqn:A0}, and~\eqref{eq:A-delta} imply
\[
\log\frac{A_{0,\delta}\sqrt{KJN_T}}{R^{\mL}}
\le c\log\frac{nDKJN_T}{\sigma(\delta)},
\]
where $c$ depends only on $b,R$ and the fixed depths. Indeed, $C_\delta^2=2D\log2+24\log(1/\sigma(\delta))$. For fixed $L_T=5$ and $L_{\mathrm{FFN}}=6$, the recursions defining $K_{\mathrm{net}}$ are finite polynomials with nonnegative coefficients in $D,w,m_T,\ell,\kappa,U_\delta$ and $\tau=2+\sigma(\delta)^{-1}$. Hence~\eqref{eq:A-delta} has polynomial growth in $n,D,\sigma(\delta)^{-1}$ under the stated weight budget, and its logarithm has the displayed order. In particular, the bias-inclusive count satisfies $DP_T\le3136nD^3$ when $n,D\ge1$; the statistical rate is unchanged. Substituting these bounds, $R^{\mL}=D(R+b)^2\sigma(\delta)^{-4}$, and $q_\delta\le\sigma(\delta)^6$ into~\eqref{eq:gen_bound} proves~\eqref{eq:stat_error1_final_bound}.

\end{proof}
\subsection{Statistical Error II}\label{sec:stat_error_II}
\begin{proof}
Note that
\begin{align*}
\E_{\mathfrak S}[\mL_{\mathfrak S}(T_{\hat\theta})-\mL_n(T_{\hat\theta})]
={}&\E_{\mathfrak S}[\mL_{n,K}(T_{\hat\theta})-\mL_n(T_{\hat\theta})]\\
&+\E_{\mathfrak S}[\mL_{n,K,J}(T_{\hat\theta})-\mL_{n,K}(T_{\hat\theta})]\\
&+\E_{\mathfrak S}[\mL_{\mathfrak S}(T_{\hat\theta})-\mL_{n,K,J}(T_{\hat\theta})].
\end{align*}
Since the variables $g_{\hat\theta}(\rho_k)$ with $\rho_k\stackrel{\mathrm{iid}}\sim\mu$ are bounded by $R^{\mL}$, Hoeffding's inequality gives, for any $u>0$,
\[
\mathbb P(\mL_{n,K}(T_{\hat\theta})-\mL_n(T_{\hat\theta})\ge u)
\le\exp\left(-\frac{2Ku^2}{(R^{\mL})^2}\right).
\]
Thus with probability at least $1-\epsilon$,
\[
\mL_{n,K}(T_{\hat\theta})-\mL_n(T_{\hat\theta})
\le R^{\mL}\sqrt{\frac{\log(1/\epsilon)}{2K}},
\]
and therefore
\[
\E_{\mathfrak S}[\mL_{n,K}(T_{\hat\theta})-\mL_n(T_{\hat\theta})]
\le R^{\mL}\sqrt{\frac{\log(1/\epsilon)}{2K}}+2\epsilon R^{\mL}.
\]
Conditional Hoeffding bounds at the episode and time--noise levels give the same expressions with $KJ$ and $KJN_T$. Combining these three bounds yields~\eqref{eq:stat-II-Hoeffding}.

For every $0<\epsilon<1$, the same argument gives
\begin{equation}\label{eq:stat-II-Hoeffding}
\E_{\mathfrak S}[\mL_{\mathfrak S}(T_{\hat\theta})-\mL_n(T_{\hat\theta})]
\le R^{\mL}\sqrt{\frac{\log(1/\epsilon)}2}
\left(\frac1{\sqrt K}+\frac1{\sqrt{KJ}}+\frac1{\sqrt{KJN_T}}\right)+6\epsilon R^{\mL}.
\end{equation}

For the cutoff used in Theorem~\ref{thm:generalization_error}, set
$s_0=\sigma(\delta)\in(0,1)$ and choose $\epsilon=s_0^6$. The explicit inequality~\eqref{eq:stat-II-Hoeffding} becomes
\begin{equation}\label{eq:stat-II-cutoff}
\begin{aligned}
&\E_{\mathfrak S}[\mL_{\mathfrak S}(T_{\hat\theta})-\mL_n(T_{\hat\theta})]\\
&\quad\le \frac{D(R+b)^2}{s_0^4}\sqrt{3\log(1/s_0)}
\left(\frac1{\sqrt K}+\frac1{\sqrt{KJ}}+\frac1{\sqrt{KJN_T}}\right)
+6D(R+b)^2s_0^2.
\end{aligned}
\end{equation}
Here $\log(1/s_0^6)=6\log(1/s_0)$ and $R^{\mL}=D(R+b)^2s_0^{-4}$.
For each fixed cutoff, this gives the order in~\eqref{eqn:statistic_error_II}; its implied constant can depend on the fixed loss envelope. If $s_0=1$, taking $\epsilon=1/2$ and using $K,J,N_T\ge1$ gives the same coarse order. When the cutoff varies with $n$, the theorem proof retains the explicit dependence in~\eqref{eq:stat-II-cutoff}.

\end{proof}
\subsection{Approximation Error III}\label{sec:approx_error_III}

Fix a task distribution $\rho$ and $t\in[\delta,1]$. Define
\begin{align}
K_{\sigma(t)}(z)&=\exp\left(-\frac{\|z\|_2^2}{2\sigma(t)^2}\right),\nonumber\\
\hat N_n(\bx(t))&=\frac1n\sum_{i=1}^nK_{\sigma(t)}\!\left(\bx(t)-\alpha(t)\bx^{(i)}\right)\bx^{(i)},
&\hat D_n(\bx(t))&=\frac1n\sum_{i=1}^nK_{\sigma(t)}\!\left(\bx(t)-\alpha(t)\bx^{(i)}\right),\nonumber\\
N(\bx(t))&=\E_{\bx\sim\rho}[K_{\sigma(t)}\!\left(\bx(t)-\alpha(t)\bx\right)\bx],
&D(\bx(t))&=\E_{\bx\sim\rho}[K_{\sigma(t)}\!\left(\bx(t)-\alpha(t)\bx\right)].\label{eq:kernel-defs}
\end{align}
Here $D$ denotes the ambient dimension, whereas $D(\bx(t))$ denotes the population kernel denominator. Since the Gaussian kernel is strictly positive, Bayes' formula gives
\begin{equation}\label{eq:ratios}
\mathcal K(\bs)=\frac{\hat N_n(\bx(t))}{\hat D_n(\bx(t))},
\qquad
\E(\bx\mid\bx(t))=\frac{N(\bx(t))}{D(\bx(t))}.
\end{equation}
The noisy-query density and its score satisfy
\begin{align}
p_{\rho,t}(\bx(t))&=\frac1{(2\pi\sigma(t)^2)^{D/2}} D(\bx(t)),\label{eq:query-density}\\
T^*(\rho)(\bx(t),t)&=-\frac{\bx(t)}{\sigma(t)^2}
+\frac{\alpha(t)}{\sigma(t)^2}\E(\bx\mid\bx(t)).\label{eq:true-score}
\end{align}
Differentiation under the integral in~\eqref{eq:query-density} is justified by the compact support of $\rho$ and the smooth Gaussian kernel.

\begin{lemma}[Finite prompt approximation error]\label{lem:approx}

Under Assumptions~\ref{assump: assump_task_regularity} and~\ref{ass:noise}, let $T_{\hat\theta}$ be the candidate from Lemma~\ref{lemma:transformer_kernel_regression} with $C=C_\delta$ and weight budget~\eqref{eq:kappa-delta}. For every task distribution $\rho$ and every $t\in[\delta,1]$,
\begin{equation}\label{eq:lemma5-final}
\E_{\hat\rho,\bx,\eps}
\|T_{\hat\theta}(\hat\rho)(\bx(t),t)-T^*(\rho)(\bx(t),t)\|_2^2
\le\frac{C_{b,d}A_{\mM}D}{n\sigma(\delta)^{d+4}}
+4Db^2\sigma(\delta)^2,
\end{equation}
where $C_{b,d}=32e\,2^db^2$, the prompt is sampled from $\rho^{\otimes n}$, and the clean query and noise are independent of the prompt.

\end{lemma}

\begin{proof}
Assume first that $\alpha(t)>0$.

\paragraph{Control at a fixed noisy query.}
Fix $\bx(t)$, and suppress the arguments of $N,D,\hat N_n,\hat D_n$ within this paragraph. The centered numerator is
\[
\hat N_n-\frac ND\hat D_n
=\frac1n\sum_{i=1}^n
K_{\sigma(t)}\!\left(\bx(t)-\alpha(t)\bx^{(i)}\right)\left(\bx^{(i)}-\frac ND\right).
\]
Each summand has mean zero, since
$\E_{\bx\sim\rho}[K_{\sigma(t)}\!\left(\bx(t)-\alpha(t)\bx\right)(\bx-N/D)]=0$.
Moreover, $N/D$ and every clean sample lie in $[-b,b]^D$. Independence therefore yields
\begin{align}
\E_{\hat\rho}\left\|\hat N_n-\frac ND\hat D_n\right\|_2^2
&=\frac1n\E_{\bx\sim\rho}
\left[K_{\sigma(t)}\!\left(\bx(t)-\alpha(t)\bx\right)^2\left\|\bx-\frac ND\right\|_2^2\right]\nonumber\\
&\le\frac{4Db^2}{n}\E_{\bx\sim\rho}[K_{\sigma(t)}\!\left(\bx(t)-\alpha(t)\bx\right)^2].\label{eq:centered-variance}
\end{align}
On the event $\hat D_n\ge D/2$, divide~\eqref{eq:centered-variance} by $D^2/4$ to obtain
\[
\E_{\hat\rho}\left[\|\mathcal K(\bs)-\E(\bx\mid\bx(t))\|_2^2
\mathbf 1_{\{\hat D_n\ge D/2\}}\right]
\le\frac{16Db^2}{n}
\frac{\E_{\bx\sim\rho}[K_{\sigma(t)}\!\left(\bx(t)-\alpha(t)\bx\right)^2]}{D(\bx(t))^2}.
\]
On the complementary event, the squared error is at most $4Db^2$. Chebyshev's inequality gives
\[
\mathbb P_{\hat\rho}(\hat D_n<D/2)
\le\frac{4\operatorname{Var}(\hat D_n)}{D^2}
=\frac4n\left(\frac{\E_{\bx\sim\rho}[K_{\sigma(t)}\!\left(\bx(t)-\alpha(t)\bx\right)^2]}{D(\bx(t))^2}-1\right).
\]
Combining both events proves
\begin{equation}\label{eq:fixed-query}
\E_{\hat\rho}\|\mathcal K(\bs)-\E(\bx\mid\bx(t))\|_2^2
\le\frac{32Db^2}{n}
\frac{\E_{\bx\sim\rho}[K_{\sigma(t)}\!\left(\bx(t)-\alpha(t)\bx\right)^2]}{D(\bx(t))^2}.
\end{equation}

\paragraph{Integration against the noisy-query law.}
Using~\eqref{eq:query-density}, the expectation of the ratio in~\eqref{eq:fixed-query} is
\begin{align}
&\E_{\bx(t)}\left[
\frac{\E_{\bx\sim\rho}[K_{\sigma(t)}\!\left(\bx(t)-\alpha(t)\bx\right)^2]}{D(\bx(t))^2}\right]\nonumber\\
&\quad=\frac1{(2\pi\sigma(t)^2)^{D/2}}\int_{\R^D}
\frac{\int_{\mM}K_{\sigma(t)}\!\left(\bx(t)-\alpha(t)\bx\right)^2\,d\rho(\bx)}{D(\bx(t))}\,d\bx(t).
\label{eq:one-denominator}
\end{align}
The factor $D(\bx(t))$ in the query density cancels one denominator. We bound the resulting integral directly, without a uniform lower bound on $D(\bx(t))$.

\paragraph{Localization on the clean support.}
Let $A_1,\ldots,A_{J'}$ be a finite measurable partition of $\mM$; the symbol $J'$ here is only a partition size and is unrelated to the number of training episodes. Discard cells of zero $\rho$-mass. For nonnegative $a_j$ and positive $b_j$,
\[
\frac{\sum_j a_j}{\sum_j b_j}
=\sum_j\frac{a_j}{\sum_k b_k}\le\sum_j\frac{a_j}{b_j}.
\]
Consequently,
\begin{equation}\label{eq:split}
\frac{\int_{\mM}K_{\sigma(t)}\!\left(\bx(t)-\alpha(t)\bx\right)^2\,d\rho(\bx)}{D(\bx(t))}
\le\sum_j
\frac{\int_{A_j}K_{\sigma(t)}\!\left(\bx(t)-\alpha(t)\bx\right)^2\,d\rho(\bx)}
{\int_{A_j}K_{\sigma(t)}\!\left(\bx(t)-\alpha(t)\bx\right)\,d\rho(\bx)}.
\end{equation}
For every positive-mass cell, define its conditional probability law by
$\rho(B\mid A_j)=\rho(B\cap A_j)/\rho(A_j)$ for measurable $B$.
The common mass $\rho(A_j)$ cancels from the numerator and denominator of the corresponding term. Thus~\eqref{eq:one-denominator} is at most $\sum_j I_j(t)$, where
\begin{equation}\label{eq:Ij}
I_j(t):=\frac1{(2\pi\sigma(t)^2)^{D/2}}\int_{\R^D}
\frac{\E_{\bx\sim\rho(\cdot\mid A_j)}[K_{\sigma(t)}\!\left(\bx(t)-\alpha(t)\bx\right)^2]}
{\E_{\bx\sim\rho(\cdot\mid A_j)}[K_{\sigma(t)}\!\left(\bx(t)-\alpha(t)\bx\right)]}
\,d\bx(t).
\end{equation}
No lower bound on the mass of a partition cell is needed.

\paragraph{Gaussian integral on each cell.}
Draw $\bx,\bx'$ independently from $\rho(\cdot\mid A_j)$. Convexity of $u\mapsto1/u$ implies, at each fixed query,
\begin{align*}
\frac{\E_{\bx}[K_{\sigma(t)}\!\left(\bx(t)-\alpha(t)\bx\right)^2]}{\E_{\bx'}[K_{\sigma(t)}\!\left(\bx(t)-\alpha(t)\bx'\right)]}
&\le\E_{\bx}[K_{\sigma(t)}\!\left(\bx(t)-\alpha(t)\bx\right)^2]\,\E_{\bx'}[1/K_{\sigma(t)}\!\left(\bx(t)-\alpha(t)\bx'\right)]\\
&=\E_{\bx,\bx'}\left[\frac{K_{\sigma(t)}\!\left(\bx(t)-\alpha(t)\bx\right)^2}{K_{\sigma(t)}\!\left(\bx(t)-\alpha(t)\bx'\right)}\right].
\end{align*}
All integrands are nonnegative. Tonelli's theorem therefore yields
\begin{equation}\label{eq:jensen}
I_j(t)\le\E_{\bx,\bx'}\left[
\frac1{(2\pi\sigma(t)^2)^{D/2}}\int_{\R^D}\frac{K_{\sigma(t)}\!\left(\bx(t)-\alpha(t)\bx\right)^2}{K_{\sigma(t)}\!\left(\bx(t)-\alpha(t)\bx'\right)}\,d\bx(t)\right].
\end{equation}
Completing the square gives
\begin{align*}
2\|\bx(t)-\alpha(t)\bx\|_2^2
-\|\bx(t)-\alpha(t)\bx'\|_2^2
={}&\|\bx(t)-\alpha(t)(2\bx-\bx')\|_2^2\\
&-2\alpha(t)^2\|\bx-\bx'\|_2^2.
\end{align*}
Hence
\begin{align}
&\frac1{(2\pi\sigma(t)^2)^{D/2}}\int_{\R^D}\frac{K_{\sigma(t)}\!\left(\bx(t)-\alpha(t)\bx\right)^2}{K_{\sigma(t)}\!\left(\bx(t)-\alpha(t)\bx'\right)}\,d\bx(t)\nonumber\\
&\quad=
\exp\left(\frac{\alpha(t)^2\|\bx-\bx'\|_2^2}{\sigma(t)^2}\right)
\frac1{(2\pi\sigma(t)^2)^{D/2}}\int_{\R^D}
\exp\left(-\frac{\|\bx(t)-\alpha(t)(2\bx-\bx')\|_2^2}{2\sigma(t)^2}\right)
\,d\bx(t)\nonumber\\
&\quad=\exp\left(\frac{\alpha(t)^2\|\bx-\bx'\|_2^2}{\sigma(t)^2}\right).
\label{eq:Gaussian-integral}
\end{align}
The normalized Gaussian integral is exactly one in ambient dimension $D$.

\paragraph{Choice of the partition.}
Take a cover of $\mM$ by balls $\mathcal B_j$ of radius
$r_t=\sigma(t)/(2\alpha(t))$ with centers on $\mM$ and cardinality $\mN(\mM,r_t)$. Turn the cover into a measurable partition by setting
\[
A_1=\mM\cap\mathcal B_1,
\qquad A_j=(\mM\cap\mathcal B_j)\setminus\bigcup_{i<j}\mathcal B_i.
\]
Every cell has diameter at most $\sigma(t)/\alpha(t)$. Thus~\eqref{eq:jensen}--\eqref{eq:Gaussian-integral} imply $I_j(t)\le e$. Applying~\eqref{eq:fixed-query}--\eqref{eq:one-denominator} and~\eqref{eq:cover-geometry},
\begin{align}
\E_{\bx(t),\hat\rho}\|\mathcal K(\bs)-\E(\bx\mid\bx(t))\|_2^2
&\le\frac{32eDb^2}{n}
\mN\left(\mM,\frac{\sigma(t)}{2\alpha(t)}\right)\nonumber\\
&\le\frac{32eDb^2}{n}\mN(\mM,\sigma(t)/2)
\le\frac{C_{b,d}A_{\mM}D}{n\sigma(t)^d}.
\label{eq:kernel-cover}
\end{align}
The second inequality uses $\alpha(t)\le1$ and monotonicity of covering numbers; the last evaluates~\eqref{eq:cover-geometry} at $\sigma(t)/2\le1$.

\paragraph{Transfer to the transformer candidate.}
On $E_\delta:=\{\|\eps\|_2\le C_\delta\}$, Lemma~\ref{lemma:transformer_kernel_regression} gives $G_{\hat\theta}=\mathcal K$. On the complement, both $G_{\hat\theta}$ and $\E(\bx\mid\bx(t))$ have coordinates in $[-b,b]$. Therefore
\begin{align}
\E\|T_{\hat\theta}-T^*\|_2^2
&=\frac{\alpha(t)^2}{\sigma(t)^4}\E\|G_{\hat\theta}-\E(\bx\mid\bx(t))\|_2^2\nonumber\\
&\le\frac{\alpha(t)^2}{\sigma(t)^4}
\E\|\mathcal K-\E(\bx\mid\bx(t))\|_2^2
+\frac{4Db^2\alpha(t)^2}{\sigma(t)^4}q_\delta\nonumber\\
&\le\frac{C_{b,d}A_{\mM}D}{n\sigma(\delta)^{d+4}}
+4Db^2\sigma(\delta)^2.
\label{eq:network-tail}
\end{align}
This proves~\eqref{eq:lemma5-final}. If $\alpha(t)=0$, both score functions equal $-\bx(t)/\sigma(t)^2$, so the score error is zero and the same bound holds.
\end{proof}

Under the stated assumptions, we are now ready to prove Proposition~\ref{prop:approximation_error}.
\begin{proof}

The candidate satisfies the budgets in~\eqref{eq:prop-parameters} by~\eqref{eq:kappa-delta} and the token-budget discussion. The prompt is conditionally independent of the clean query and its perturbation, given the task. Therefore~\eqref{eq:true-score} gives
\[
\E\left[T^*(\rho)(\bx(t),t)+\frac{\eps}{\sigma(t)}
\,\middle|\,\rho,t,\bx(t),\hat\rho\right]=0.
\]
Expanding the difference of the two squared losses, the cross term has conditional expectation zero. Consequently,
\begin{equation}\label{eq:DSM}
\mL_n(T_\theta)-\mL(T^*)
=\E_{\rho,\hat\rho,\bx,t,\eps}
\|T_\theta(\hat\rho)(\bx(t),t)-T^*(\rho)(\bx(t),t)\|_2^2.
\end{equation}
All quantities are integrable: the two mean terms are bounded and $\sigma(t)\ge\sigma(\delta)>0$. Averaging~\eqref{eq:lemma5-final} over the task and time distributions proves~\eqref{eqn:final_approx_error}.

\end{proof}
\subsection{Proof of Theorem~\ref{thm:generalization_error}}\label{sec:proof_generalization_error}
\begin{proof}
Set $s_0=\sigma(\delta)$ and
\[
\Phi_{K,J,N_T}:=\frac1{\sqrt K}+\frac1{\sqrt{KJ}}+\frac1{\sqrt{KJN_T}}.
\]
We use the inequality~\eqref{eq:stat-II-cutoff} for Statistical Error~$\mathrm{II}$. Since $n\ge2$ and $D,K,J,N_T\ge1$,
\[
\sqrt{3\log(1/s_0)}
\le\sqrt3\sqrt{nD^3\log\frac{nDKJN_T}{s_0}}.
\]
Thus its statistical contribution is absorbed into Statistical Error~$\mathrm{I}$, while its tail contribution is $6D(R+b)^2s_0^2$.
By combining~\eqref{eq:stat_error1_final_bound},~\eqref{eq:stat-II-cutoff}, and~\eqref{eqn:final_approx_error} in~\eqref{eqn:gen_error_decomp}, we obtain
\begin{equation}\label{eqn:main_pf1}
\begin{aligned}
&\E_{\mathfrak S}\left[\mL_n(T_{\theta^*;\mathfrak S})-\mL(T^*)\right]\\
&\quad\le (C'_1+\sqrt3)\frac{D(R+b)^2}{\sigma(\delta)^4}
\sqrt{nD^3\log\frac{nDKJN_T}{\sigma(\delta)}}
\left(\frac1{\sqrt K}+\frac1{\sqrt{KJ}}+\frac1{\sqrt{KJN_T}}\right)\\
&\qquad+\frac{C_{b,d}A_{\mM}D}{n\sigma(\delta)^{d+4}}
+D[4b^2+7(R+b)^2]\sigma(\delta)^2.
\end{aligned}
\end{equation}
The coefficient $7(R+b)^2$ combines the localization remainder $(R+b)^2$ from Statistical Error~$\mathrm{I}$ and the tail contribution $6(R+b)^2$ from Statistical Error~$\mathrm{II}$.

Choose $\delta$ such that $\sigma(\delta)=n^{-1/(d+6)}$, as in the statement of the theorem. Then
\[
\log\frac{nDKJN_T}{\sigma(\delta)}
=\log\left(n^{1+\frac1{d+6}}DKJN_T\right)
\le\log(n^2DKJN_T),
\]
and
\[
\frac1{n\sigma(\delta)^{d+4}}=n^{-\frac2{d+6}},
\qquad\sigma(\delta)^2=n^{-\frac2{d+6}}.
\]

The loss bound becomes
\[
R^{\mL}=\frac{D(R+b)^2}{\sigma(\delta)^4}
=D(R+b)^2n^{4/(d+6)},
\]
so its dependence on $n$ must be included in the statistical term. The candidate from Proposition~\ref{prop:approximation_error} uses the $n+1$ active-token realization of Lemma~\ref{lemma:transformer_kernel_regression}. Its weight budget satisfies
\[
\kappa=O\left(\frac{D^8n^2b^8R^4}{\sigma(\delta)^8}\right)
=O\left(D^8 n^{\frac{2d+20}{d+6}}b^8R^4\right).
\]

Therefore
\begin{equation}\label{eqn:main_pf3}
\begin{aligned}
&\E_{\mathfrak S}\left[\mL_n(T_{\theta^*;\mathfrak S})-\mL(T^*)\right]\\
&\quad\le (C'_1+\sqrt3)(R+b)^2D^{5/2}n^{\frac12+\frac4{d+6}}
\sqrt{\log(n^2DKJN_T)}
\left(\frac1{\sqrt K}+\frac1{\sqrt{KJ}}+\frac1{\sqrt{KJN_T}}\right)\\
&\qquad+D[C_{b,d}A_{\mM}+4b^2+7(R+b)^2]n^{-\frac2{d+6}}.
\end{aligned}
\end{equation}
With $C_1=(C'_1+\sqrt3)(R+b)^2$ and $C_2=C_{b,d}A_{\mM}+4b^2+7(R+b)^2$, absorbing the fixed factors proves~\eqref{eqn:generalization_error_bound}.
\end{proof}

\section{Auxiliary Lemmas}
For completeness, we state the following lemmas; Their proof can be found in \cite{shen2026understanding}.
\begin{lemma} [Interaction Lemma] \label{lemmaIA}
   Let $H=[h_t]_{1\leq t\leq\ell}\in\mathbb{R}^{d_{embed}\times\ell}$ be an embedding matrix such that $h_t^{(d_{embed}-2):(d_{embed}-1)}=\mathcal{I}_t$ and $h^{d_{embed}}_t=1$. Fix $1\leq t_1,t_2\leq \ell, 1\leq i\leq d_{embed}$, and $\ell\in\mathbb{N}$. Suppose $d_{embed}\geq 5$ and $\|H\|_{\infty,\infty}<U$ for some $U>0$, and the data kernels $Q^{data} \in \mathbb{R}^{(d_{embed}-3) \times d_{embed}}$ (the first $(d_{embed}-3)$ rows in the query matrix $Q$) and $K^{data}\in \mathbb{R}^{(d_{embed}-3) \times d_{embed}}$ (the first $(d_{embed}-3)$ rows in the key matrix $K$) satisfy $\max\{\|Q^{data}\|_{\infty,\infty},\|K^{data}\|_{\infty,\infty}\}\leq\kappa$. Then one can construct an attention head $A$ with ReLU activation such that 
   \[ [A(H)]_t = 
   \begin{cases} 
    \mathrm{ReLU}(\langle Q^{data}h_t, K^{data}h_{t_2}\rangle)\mathbf{e}_i & \text{if}\text{  } t=t_1, \\
      0 & \text{otherwise}. \\
   \end{cases}
\]
The weight parameters of this attention head satisfies $\|\theta_A\|_{\infty}=O(d_{embed}^4\kappa^2\ell^2U^2)$.
\end{lemma}

The same positional selection applies to any linear value map $V$, replacing $\mathbf e_i$ by $Vh_{t_2}$. The weight bound becomes the maximum of the original bound and $\|V\|_{\infty,\infty}$.

\begin{lemma} [Decrementing Lemma] \label{lemmadecrement}
Let $d_{embed}\geq 5$ and $H=[h_t]_{1\leq t\leq\ell}\in\mathbb{R}^{d_{embed}\times\ell}$, be an embedding matrix such that $h_t^{(d_{embed}-2):(d_{embed}-1)}=(\mathcal{I}_t^1,\mathcal{I}_t^2)=\mathcal{I}_t$ and $h^{d_{embed}}_t=1$. Then for any $r_1,r_2$ with  $1\leq r_1\leq r_2\leq d_{embed}-3$ and any $k_1,k_2$ with $1\leq k_1,k_2\leq \ell$ and any $M>0$, there exists a six-layer residual feed-forward network $({\rm FFN})$ such that
\[ \text{\rm FFN}(h_t) + h_t = 
   \begin{cases} 
    h_t & \text{if}\text{  } t\in\{1,\cdots,k_1\}\cup\{k_2,\cdots,\ell\} \\
      \begin{bmatrix}
      (h_t)_{1} \\
          \vdots \\
          (h_t)_{r_1-1} \\
          (h_t)_{r_1}-M \\
          \vdots \\
          (h_t)_{r_2}-M \\
          (h_t)_{r_2+1} \\
          \vdots \\
          (h_t)_{d_{embed}-3} \\
          \mathcal{I}_t \\
          1
      \end{bmatrix} & \text{otherwise}\text{ }  \\
   \end{cases}
\]
Additionally, we have $\|\theta_{\rm FFN}\|_{\infty}\leq O(\ell M)$.
\end{lemma}

\section{Additional Experimental Details}\label{sec:network_details}
\subsection{Network Architecture} We provide additional implementation details for the architectures used in our experiments. While the models employed in practice are more expressive than the simplified versions presented in the main text for theoretical analysis, they retain the same core design principles. All three methods use a time-conditioned, multi-resolution U-Net as the denoising backbone, and differ primarily in how the prompt set is encoded and injected. In ICL DM, the prompt set is represented by multi-scale features that interact with the noisy query through cross-attention. In Cond DM, the prompt set is compressed into a single learned embedding that modulates the U-Net through FiLM-style conditioning. In SNO, the learned prompt embedding is replaced by a projected kernel mean embedding (KME), while the denoising U-Net remains otherwise the same as in Cond DM.

\paragraph{ICL DM.} Our implementation of ICL DM uses a time-conditioned, multi-resolution U-Net with prompt-dependent cross-attention. Given an unordered prompt set $\hat\rho=\{\bx^{(i)}\}_{i=1}^{n}$, a noisy query $\bx(t)$ and diffusion time $t$, we first encode the diffusion time using a sinusoidal embedding followed by a multilayer perceptron. The prompt set is then processed by a prompt encoder composed of an input convolution, an initial downsampling layer, and a hierarchy of residual blocks with time modulation. At each resolution, a set-transformer block applies self-attention across the prompt elements, allowing information to be exchanged before the features are aggregated by permutation-invariant mean pooling. This produces multi-scale key-value feature maps that summarize the prompt set. In parallel, the noisy query $\bx(t)$ is passed through a U-Net encoder-decoder with residual blocks at corresponding resolutions, while cross-attention modules at every scale in both the encoder and decoder inject the pooled prompt features into the query pathway. After a final upsampling stage to the original spatial resolution, the network outputs the predicted noise field, which serves as the model prediction for denoising score matching. Specifically, we use a U-Net architecture with 4 resolution levels, where the number of channels at each level is [64, 128, 256, 512]. The timestep embedding is produced by a two-layer MLP with hidden dimension $256$, each residual block contains two convolutional layers, cross-attention uses $4$ heads at every scale, and the prompt set transformer uses $2$ attention heads.

\paragraph{Cond DM.} Our conditional diffusion model uses a time-conditioned, multi-resolution U-Net in which the prompt set is compressed into a single global conditioning vector rather than injected through cross-attention. Given an unordered prompt set \(\hat\rho=\{\bx^{(i)}\}_{i=1}^{n}\), a noisy query \(\bx(t)\), and diffusion time t, we encode t using a sinusoidal embedding followed by a multilayer perceptron. The prompt set is passed through a prompt encoder consisting of an input convolution, an initial downsampling layer, and a hierarchy of residual blocks, and only the deepest prompt feature is retained. This feature is averaged over the prompt dimension to ensure permutation invariance, aggregated over spatial locations by a learned attention-based pooling module, and mapped to a prompt embedding vector. The resulting prompt embedding is added to the time embedding and used to modulate all residual blocks of the denoising U-Net through FiLM-style conditioning. The denoising network follows an encoder-decoder architecture with skip connections and includes a self-attention block at the bottleneck to enhance global spatial interactions. A final upsampling stage restores the original spatial resolution, after which the network predicts the noise field used for denoising score matching. Specifically, the U-Net has 4 resolution levels with channel dimensions [64,128,256,512]; the timestep embedding and prompt embedding both have dimension 256; the spatial attention pooling module uses a hidden dimension of 128; and both the deepest prompt self-attention block and the bottleneck self-attention block use 2 attention heads.

\paragraph{SNO.} Our implementation of SNO uses a time-conditioned, multi-resolution U-Net in which the conditioning information is provided by a kernel mean embedding (KME) of the prompt set rather than by cross-attention over prompt features. Given a KME vector associated with the prompt set, a noisy query $\bx(t)$, and diffusion time $t$, we first encode the diffusion time using a sinusoidal embedding followed by a multilayer perceptron. The KME vector is then mapped by a two-layer projection network to the same latent dimension as the timestep embedding, and the resulting conditioning vector is added to the time embedding. This combined embedding is used to modulate all residual blocks in the denoising U-Net through FiLM-style conditioning. Other than replacing the prompt encoder and attention-based prompt pooling in Cond DM by the KME projection module, the rest of the denoising architecture is identical to the conditional diffusion model baseline. Specifically, we use a U-Net architecture with 4 resolution levels, where the number of channels at each level is $[64,128,256,512]$. The timestep embedding and projected KME embedding both have dimension $256$, and the bottleneck self-attention block uses $2$ attention heads.

\subsection{Training Details} All models are trained using the Adam optimizer with a learning rate of 1e-4 and a batch size of 32. 

\begin{table}[h]
\centering
\begin{tabular}{c c c c c c c}
\hline
Dataset & Resolution & $K$ & $B$ & $n$ & $m$ & Epochs \\
\hline
NS & 128x128 & 100 & 1000 & 100 & 100 & 400 \\
ERA5 & 128x128 & 100 & 1000 & 100 & 100 & 400 \\
QuickDraw & 64x64 & 5, 10, 100 & 2000 & 200 & 200 & 400 \\
Food-101 & 128x128 & 5, 10, 100 & 1000 & 200 & 200 & 2000 \\
\hline
\end{tabular}
\caption{Dataset specifications and sampling parameters used in the experiments. Here, $K$ denotes the number of meta-training tasks, $B$ the number of samples used for training per task, $n$ the number of prompt samples provided to the model, $m$ the number of generated samples used for evaluation, and Epochs the number of training epochs used for ICL DM, Cond DM, and SNO.}
\label{tab:dataset_stats}
\end{table}

For sampling, we set $\delta=10^{-4}$ and discretize the reverse-time SDE using $N_{\mathrm{rev}}=1000$ Euler--Maruyama steps for all experiments.

\subsection{Evaluation Metrics}\label{sec:metrics} 
The metrics used to evaluate the performance of the models are as follows:
\begin{itemize}
    \item \textbf{Scalar-Valued Empirical Probability Density Functions (PDFs).} For each ensemble of
scalar field values
\[
\{\mathbf{u}_i\}_{i=1}^N
:=
\{u_{i,p,q}\mid i=1,\dots,N;\; p,q=1,\dots,N_x\},
\]
we estimate the empirical probability density function (PDF) of the scalar
variable \(u\) using Kernel Density Estimation (KDE):
\begin{equation}
\mathrm{PDF}(u)
:=
\frac{1}{N_s h}
\sum_{k=1}^{N_s}
K\left(\frac{u-u_k}{h}\right),
\qquad
N_s = N \times N_x^2,
\end{equation}
where \(K(\cdot)\) is a smooth kernel function, typically chosen as the Gaussian kernel,
\begin{equation}
K(z)=\frac{1}{\sqrt{2\pi}}\exp\left(-\frac{z^2}{2}\right),
\end{equation}
and \(h>0\) denotes the bandwidth parameter controlling the smoothness of the resulting density estimate.

\item \textbf{Jensen--Shannon divergence (JSD) of PDFs.}
Given two empirical datasets $\{\bx_i\}_{i=1}^{n}$ and $\{\by_j\}_{j=1}^{m}$, we first
estimate their corresponding probability density functions \(p\) and \(q\).
The Jensen--Shannon divergence between \(p\) and \(q\) is defined by
\begin{equation}
\mathrm{JSD}(p\|q)
:=
\frac{1}{2}\,\mathrm{KL}(p\|m)
+
\frac{1}{2}\,\mathrm{KL}(q\|m),
\qquad
m := \frac{1}{2}(p+q),
\end{equation}
where \(\mathrm{KL}\) denotes the Kullback--Leibler divergence.

\item \textbf{Energy spectra.} For a scalar field \(f(\bx)\) defined on a uniform grid \(\{\bx_n\}\), the energy
spectrum describes how the energy of \(f\) is distributed across Fourier modes.
It is defined by
\begin{equation}
E_f(k)
=
\sum_{|\bk|=k}
\left|
\sum_{n}
f(\bx_n)\exp\left(-\bj\,2\pi\,\bk\cdot \bx_n / L\right)
\right|^2,
\end{equation}
where \(\bx_n\) denotes the grid points, \(j\) is the imaginary unit, \(\bk\) is the
wave vector, and \(k=|\bk|\) is its magnitude. Thus, \(E_f(k)\) represents the total
energy contained in all Fourier modes with wave number magnitude \(k\).
\item \textbf{Mean Energy Log Ratio (MELR).}
To measure the spectral alignment between two empirical datasets
\(\{\bx_i\}_{i=1}^{n}\) and \(\{\by_j\}_{j=1}^{m}\), we consider the Mean Energy
Log Ratio (MELR), defined as
\begin{equation}
\mathrm{MELR}(\bx,\by)
=
\sum_{k} w_k
\left|
\log\left(\frac{E_{\bx}(k)}{E_{\by}(k)}\right)
\right|,
\end{equation}
where \(E_{\bx}(k)\) and \(E_{\by}(k)\) denote the average energy spectra of the
two datasets at wave number \(k\), and \(w_k\) is a user-specified weight
function. If \(w_k=1\) for all \(k\), we obtain the unweighted mean energy log
ratio. If \(w_k = E_{\by}(k)/\sum_j E_{\by}(j)\), we obtain the weighted mean
energy log ratio.

\item \textbf{Fr\'echet Inception Distance (FID)}. The Fr\'echet Inception Distance (FID) quantifies the discrepancy between the distributions of real and generated images by comparing the means and covariance matrices of features extracted from a pretrained Inception-v3 network; smaller values indicate better agreement.

\item \textbf{Problem-specific FID}. We use standard Fr\'echet-distance formula, but evaluate it in a problem-specific feature space rather than the default feature space induced by an ImageNet-pretrained Inception-v3 network. Specifically, we replace the standard Inception encoder with a feature extractor trained for the dataset of interest. For both the Quick, Draw! dataset \cite{quickdraw} and the Food-101 dataset \cite{bossard14}, we use a ResNet-18 model pretrained on the corresponding dataset to extract features from real and generated samples, and then compute the FID score from the resulting empirical feature distributions. The corresponding metrics are denoted by FID-QD and FID-Food, respectively.
\end{itemize}

\end{document}